%% file: main_icml.tex
\documentclass{article}

\usepackage{microtype}
\usepackage{graphicx}
\usepackage{subcaption}
\usepackage{booktabs} %

\usepackage{natbib}

\usepackage{hyperref}

\usepackage[preprint]{icml2026}

\usepackage{amsmath}
\usepackage{amssymb}
\usepackage{mathtools}
\usepackage{amsthm}

\usepackage[capitalize,noabbrev]{cleveref}

\theoremstyle{plain}
\newtheorem{theorem}{Theorem}[section]
\newtheorem{proposition}[theorem]{Proposition}
\newtheorem{lemma}[theorem]{Lemma}
\newtheorem{corollary}[theorem]{Corollary}
\theoremstyle{definition}

\theoremstyle{remark}
\newtheorem{remark}[theorem]{Remark}

\input{header/header}
\input{header/letters}

\input{header/def}
\input{header/def_sarsa}

\icmltitlerunning{Convergence Guarantees for Federated SARSA with Local Training and Heterogeneous Agents}

\begin{document}

\twocolumn[
  \icmltitle{Convergence Guarantees for Federated SARSA\texorpdfstring{\\}{ } with Local Training and Heterogeneous Agents}

  \icmlsetsymbol{equal}{*}

  \begin{icmlauthorlist}
    \icmlauthor{Paul Mangold}{polytechnique}
    \icmlauthor{Eloïse Berthier}{ensta}
    \icmlauthor{Eric Moulines}{polytechnique,mbzuai}
  \end{icmlauthorlist}

  \icmlaffiliation{polytechnique}{CMAP, CNRS, École polytechnique, Institut Polytechnique de Paris, 91120 Palaiseau, France}
  \icmlaffiliation{ensta}{Unité d’Informatique et d’Ingénierie des Systèmes,
ENSTA, Institut Polytechnique de Paris, 91120 Palaiseau, France }
  \icmlaffiliation{mbzuai}{Mohamed bin Zayed University of Artificial Intelligence (MBZUAI), UAE}

  \icmlcorrespondingauthor{Paul Mangold}{paul.mangold@polytechnique.edu}

  \icmlkeywords{Federated reinforcement learning, Markovian noise, Federated SARSA, Theoretical reinforcement learning, Theoretical analysis}

  \vskip 0.3in
]

\printAffiliationsAndNotice{}  %

\begin{abstract}
We present a novel theoretical analysis of Federated \SARSA (\FedSARSA) with linear function approximation and local training.
We establish convergence guarantees for \FedSARSA in the presence of heterogeneity, both in local transitions and rewards, providing the first sample and communication complexity bounds in this setting. At the core of our analysis is a new, exact multi-step error expansion for single-agent \SARSA, which is of independent interest. 
Our analysis precisely quantifies the impact of heterogeneity, demonstrating the convergence of \FedSARSA with multiple local updates. Crucially, we show that \FedSARSA achieves linear speed-up with respect to the number of agents, up to higher-order terms due to Markovian sampling. 
Numerical experiments support our theoretical findings.
\end{abstract}

\section{Introduction}
\label{sec:intro}
\input{src/intro}

\section{Related Work}
\label{sec:related-work}
\input{src/related-work}

\section{Background on Federated \textsf{SARSA}}
\label{sec:definitions}
\input{src/definitions}

\section{Single-Agent \textsf{SARSA}}
\label{sec:convergence}

\input{src/convergence}

\section{Convergence of \textsf{FedSARSA}}
\label{sec:fed-sarsa}
\input{src/federated-sarsa}

\section{Numerical Experiments}
\label{sec:numerical-expe}

\input{src/expe.tex}

\section{Conclusion and Discussion}
\label{sec:conclu}
\input{src/conclu.tex}

\section*{Impact Statement}
This paper presents work whose goal is to advance the field of machine learning. There are many potential societal consequences of our work, none of which we feel must be specifically highlighted here.

\bibliography{references}
\bibliographystyle{icml2026}

\newpage
\appendix
\onecolumn

\renewcommand{\cstthm}[2][1=]{#2}
\renewcommand{\lessthm}{\le}

\section{Technical Results on Markov Chains}
\label{sec:markov}
\input{src/sup-markov}

\section{Technical Lemmas}%
\label{sec:proof-technical}
\input{src/sup-proof-single-sarsa.tex}

\section{Bound on terms of the decomposition}
\label{sec:app-bound-term-decomposition}
\input{src/sup-terms-decomp}

\section{Proof for single-agent \SARSA --- Proofs of Lemma~\ref{lem:one-step-improvement} and Theorem~\ref{thm:convergence-rate-single-agent}}
\label{sec:proof-single-sarsa}
\input{src/sup-proof-descent-single-agent}

\section{Proofs for federated \SARSA}
\label{sec:proof-fed-sarsa}
\input{src/sup-proof-fed-sarsa.tex}

\end{document}

%% file: header/header.tex
\usepackage{color}
\usepackage{graphicx}
\usepackage{subcaption}
\usepackage{booktabs}

\usepackage{amsfonts,amsmath,amssymb,amsthm}

\usepackage{thmtools}
\usepackage{regexpatch}
\makeatletter
\xpatchcmd\thmt@restatable{%
\csname #2\@xa\endcsname\ifx\@nx#1\@nx\else[{#1}]\fi
}{%
\ifthmt@thisistheone
\csname #2\@xa\endcsname\ifx\@nx#1\@nx\else[{#1}]\fi
\else
\csname #2\@xa\endcsname[{Restated}]
\fi}{}{}
\makeatother

\usepackage{xcolor}
\usepackage{srcltx}
\usepackage{etoolbox}
\usepackage{nameref}
\usepackage{comment}
\usepackage[shortlabels]{enumitem}
\setlist[enumerate]{noitemsep, topsep=0pt, leftmargin=*}
\usepackage{amsfonts} %
\usepackage{nicefrac} %
\usepackage{mathtools}
\usepackage[capitalize,noabbrev]{cleveref}
\usepackage{xargs}
\usepackage{multirow}
\usepackage[textwidth=2cm,textsize=scriptsize]{todonotes}
\usepackage{xspace}

\usepackage{xfrac}

\usepackage{amssymb}%
\usepackage{float}
\usepackage{bbm} 

\usepackage{tabularx}

%% file: header/letters.tex
\ExplSyntaxOn

\NewDocumentCommand{\definealphabet}{mmmm}
 {%
  \int_step_inline:nnn { `#3 } { `#4 }
   {
    \cs_new_protected:cpx { #1 \char_generate:nn { ##1 }{ 11 } }
     {
      \exp_not:N #2 { \char_generate:nn { ##1 } { 11 } }
     }
   }
 }

\ExplSyntaxOff

\definealphabet{bb}{\mathbb}{A}{Z}
\definealphabet{mc}{\mathcal}{A}{Z}
\definealphabet{mc}{\mathcal}{a}{z}
\definealphabet{mf}{\mathfrak}{A}{Z}
\definealphabet{mf}{\mathfrak}{a}{z}
\definealphabet{ms}{\mathsf}{A}{Z}
\definealphabet{ms}{\mathsf}{a}{z}

\newcommand*{\eg}{e.g.\@\xspace}
\newcommand*{\ie}{i.e.\@\xspace}

\makeatletter
\newcommand*{\etc}{%
    \@ifnextchar{.}%
        {etc}%
        {etc.\@\xspace}%
}
\makeatother

%% file: header/def.tex
\def\Zset{\mathsf{Z}}
\def\Zsigma{\mathcal{Z}}
\def\rset{\mathbb{R}}

\def\nset{\ensuremath{\mathbb{N}}}
\def\nsets{\ensuremath{\mathbb{N}^*}}

\def\rmd{\mathrm{d}}

\DeclareMathOperator{\proj}{proj}

\newcommandx\sequence[4][2=,3=,4=]
{\ifthenelse{\equal{#3}{}}{\ensuremath{\{ #1_{#2 #4}\}}}{\ensuremath{\{ #1_{#2 #4} \}_{#2 \in #3}}}}
\newcommandx\sequenceD[2][2=]
{\ifthenelse{\equal{#2}{}}{\ensuremath{\{ #1\}}}{\ensuremath{\{ #1\!~:\!~#2  \}}}}
\newcommandx\sequenceDouble[4][3=,4=]
{\ifthenelse{\equal{#3}{}}{\ensuremath{\{ (#1_{#3},#2_{#3})\}}}{\ensuremath{\{ (#1_{#3},#2_{#3})\}_{#3 \in #4}}}}
\newcommandx{\sequencen}[2][2=n\in\nset]{\ensuremath{\{ #1, \eqsp #2 \}}}
\newcommandx\sequencens[2][2=n]
{\ensuremath{\{ #1_{#2} \!~:\!~#2\in\nsets\}}}
\newcommandx\sequencet[4]
{\ensuremath{\{ #1{#2_{#3}} \, : \, \eqsp #3 \in #4 \}}}

\def\PE{\mathbb{E}}
\def\P{\mathbb{P}}
\newcommandx{\CPE}[3][1=]
{\ifthenelse{\equal{#1}{}}{{\mathbb E}[ #2 \, | #3 ]}{{\mathbb E}_{#1}[ #2 \, | #3 ]}}
\newcommandx{\bCPE}[3][1=]
{\ifthenelse{\equal{#1}{}}{{\mathbb E}\Big[ #2 \, \Big| #3 \Big]}{{\mathbb E}_{#1}\Big[ #2 \, \Big | #3 \Big]}}
\newcommandx{\CPP}[3][1=]
{\ifthenelse{\equal{#1}{}}{{\mathbb P}\left(\left. #2 \, \right| #3 \right)}{{\mathbb P}_{#1}\left(\left. #2 \, \right | #3 \right)}}

\def\PP{\P}
\def\PPcoupling{\bar{\P}}
\def\PEcoupling{\bar{\PE}}
\newcommand{\proba}[1]{\mathbb{P}\left( #1 \right)}

\newcommand{\abs}[1]{\vert #1\vert}
\newcommandx{\norm}[2][2=]{\Vert#1 \Vert_{{#2}}}
\newcommandx{\bnorm}[2][2=]{\Big\Vert#1 \Big\Vert_{{#2}}}
\newcommandx{\normop}[1]{\norm{#1}[\textrm{op}]}
\newcommandx{\normir}[1]{\norm{#1}[\infty,1]}

\newcommand{\vertiii}[1]{{\left\vert\kern-0.25ex\left\vert\kern-0.25ex\left\vert #1 
    \right\vert\kern-0.25ex\right\vert\kern-0.25ex\right\vert}}
\newcommandx{\triplenorm}[2][2=]{\vertiii{#1}_{#2}}

\newcommand{\pscal}[2]{\langle#1\,,\,#2\rangle}
\newcommand{\bpscal}[2]{\Big\langle#1\,,\,#2\Big\rangle}

\newcommandx{\as}[1][1=\PP]{\ensuremath{#1\, -\mathrm{a.s.}}}

\newcommand{\eqsp}{\enspace}

\newcommand{\Id}{\mathrm{I}}

\def\policy{\pi}

\newcommand{\algoname}[1]{\textsf{\small #1}\xspace}

\def\A{\mathcal{A}}
\def\S{\mathcal{S}}

\def\nagent{N}

\def\eqdef{\overset{\Delta}{=}}

\newcommand{\Ee}{\mathbb{E}}
\newcommand{\Rr}{\mathbb{R}}

\newcommand{\cC}{\mathcal{C}}

\newcommand{\cA}{\mathcal{A}}

\allowdisplaybreaks %

\makeatletter
\newtheoremstyle{boldclaim} %
  {3pt}                     %
  {3pt}                     %
  {\itshape}                %
  {}                        %
  {\bfseries}               %
  {.}                       %
  { }                       %
  {}                        %
\makeatother

\theoremstyle{boldclaim}
\newtheorem{claim}{Claim}

\def\sigfield{\mathcal{F}}
\def\sigfieldcoupling{\bar{\mathcal{F}}}
\newtheorem{assum}{\textbf{A}\hspace{-2pt}}
\crefname{assum}{A\hspace{-2pt}}{A\hspace{-2pt}}

\newcommand{\intlist}[2]{\{#1, \dots, #2\}}

\newcommandx{\cstthm}[2][1=]{#1}
\newcommand{\lessthm}{\lesssim} 
\def\TV{\operatorname{TV}} 

%% file: header/def_sarsa.tex
\def\SARSA{\algoname{SARSA}}

\def\FedSARSA{\algoname{FedSARSA}}
\def\SCAFFOLD{\algoname{Scaffold}}
\def\SCAFFLSA{\algoname{Scafflsa}}

\def\temperature{\beta}
\newcommandx{\polimprove}[1][1=\temperature]{\operatorname{Imp}_{#1}}

\def\param{\theta}
\newcommandx{\globparam}[1]{\param_{#1}}
\newcommandx{\avgparam}[1]{\bar{\param}_{#1}}
\newcommandx{\locparam}[2]{\param^{(#1)}_{#2}}

\newcommandx{\transient}[1]{\theta^{\text{(tr)}}_{#1}}
\newcommandx{\avgupdatebias}[1]{\rho_{#1}}
\newcommandx{\avgfluct}[1]{\tilde\theta^{\text{(fl)}}_{#1}}
\newcommandx{\avgsamplingerror}[1]{\varphi_{#1}}

\newcommand{\loccontract}[2]{\Gamma_{#2}^{(#1)}}
\newcommand{\loccontractbar}[2]{\bar{\Gamma}_{#2}^{(#1)}}

\def\noisew{\varepsilon}
\newcommand{\locnoise}[2]{\noisew^{(#1)}(#2)}
\newcommand{\locnoisetheta}[3]{\noisew_{#2}^{(#1)}(#3)}

\def\errorw{\varphi}

\newcommand{\locerrordet}[2]{\errorw_{#2}^{(#1)}}

\newcommand{\biasterm}{\Delta_{1:\nlupdates}}

\def\Aw{\mathbf{A}}
\def\bw{\mathbf{b}}
\newcommandx{\nbarA}[1]{\bar{\Aw}\!^{(#1)}}
\newcommandx{\barA}{\bar{\Aw}}
\newcommandx{\barb}{\bar{\bw}}
\newcommandx{\nbarb}[1]{\bar{\bw}^{(#1)}}
\newcommandx{\nA}[2]{\Aw\!^{(#1)}(#2)}
\newcommandx{\nb}[2]{\bw^{(#1)}(#2)}
\newcommandx{\nAs}[2]{\Aw\!^{(#1)}(#2)}
\newcommandx{\nbs}[2]{\bw^{(#1)}(#2)}
\newcommandx{\nAsw}[1]{\Aw\!^{(#1)}}
\newcommandx{\nbsw}[1]{\bw^{(#1)}}
\newcommandx{\nAt}[2]{\widetilde{\Aw}\!^{(#1)}(#2)}
\newcommandx{\nAo}[2]{\mathring{\Aw}\!^{(#1)}(#2)}
\newcommandx{\nbt}[2]{\widetilde{\bw}^{(#1)}(#2)}

\newcommandx{\avgA}[1]{\bar{\Aw}_{#1}}

\def\lminAbar{a} %

\def\randStatew{Z}
\def\randStateopt{Y}
\newcommand{\locRandState}[2]{\randStatew_{#2}^{(#1)}}
\newcommand{\locRandStateOpt}[2]{\randStateopt_{#2}^{(#1)}}
\newcommand{\globRandState}[1]{\randStatew_{#1}}
\newcommand{\globRandStateOpt}[1]{\randStateopt_{#1}}
\newcommand{\bglobRandState}[1]{\bar{\randStatew}_{#1}}
\newcommand{\bglobRandStateOpt}[1]{\bar{\randStateopt}_{#1}}

\def\statespace{\mathcal{S}}
\def\actionspace{\mathcal{A}}
\def\policy{\pi}
\def\qfunc{Q}
\def\discount{\gamma}
\def\feature{\phi}
\def\reward{r}
\newcommandx{\nreward}[1]{\reward^{(#1)}}
\newcommandx{\nkerMDP}[1]{P^{(#1)}}
\newcommandx{\npolkerMDP}[2]{P_{#2}^{(#1)}}
\newcommandx{\nstab}[2]{{\mu_{#2}^{(#1)}}}

\def\nepisode{T}

\def\step{\eta}

\def\nagent{N}
\def\nlupdates{H}

\def\param{\theta}
\def\varparam{\vartheta}
\def\paramlim{\param_\star}
\newcommand{\locparamlim}[1]{\varparam_\star^{(#1)}}
\newcommandx\paramlimc[1][1=]{\paramlim^{(#1)}}
\newcommand{\locparamstar}[1]{\param_\star^{(#1)}}

\def\implip{\mathrm{C}_{\textnormal{lip}}}
\def\invdistlip{\mathrm{C}_{\mu}}

\def\projradius{\mathrm{C}_{\textnormal{proj}}}
\def\locprojradius{\mathrm{\widetilde{C}}_{\textnormal{proj}}}
\def\projset{\mathcal{W}}

\def\hgtyA{\zeta_{\barA}}
\def\hgtyAtheta{\zeta_{\paramlim}}

\def\kerhgty{\epsilon_p}
\def\rewardhgty{\epsilon_r}

\newcommandx{\statdist}[1]{\nu_{#1}}
\newcommandx{\locstatdist}[2]{\nu^{(#1)}_{#2}}
\newcommandx{\locstatestatdist}[2]{\mu^{(#1)}_{#2}}
\def\taucoupling{\tau_{\mathrm{couple}}}
\def\taumix{\tau_{\mathrm{mix}}}
\newcommand{\boundmixing}[1]{(1/4)^{\lfloor #1 / \taumix \rfloor}}

\def\boundA{\mathrm{C}_A}
\def\boundb{\mathrm{C}_b}

\def\boundGrad{\mathrm{G}}

\def\nactions{|\cA|}

\newcommand{\globfiltr}[1]{\mcF_{#1}}

\newcommand{\startedfromstationary}{\delta}

\newcommandx{\dterm}[1]{\boldsymbol{\mathrm{T}_{#1}}}
\newcommandx{\dfterm}[1]{\boldsymbol{\mathrm{U}_{#1}}} 
\newcommandx{\intermterm}[3][3=]{\boldsymbol{\mathrm{#1}^{#3}_{#2}}}

%% file: src/intro.tex
Federated reinforcement learning (FRL) \citep{zhuo2019federated} allows multiple agents to collaboratively learn a policy without exchanging raw data. By sharing information, agents can accelerate training by leveraging one another’s experience. This paradigm is particularly valuable when communication, storage, or privacy constraints preclude direct data sharing. Yet, effective learning in such settings remains difficult. Communication is costly, and while federated methods try to mitigate this by relying on local updates, these updates can cause client drift when environment dynamics are different from one agent to another.

Despite rapid progress in FRL \citep{qi2021federated,jin2022federated,khodadadian2022federated}, little attention has been given to the federated counterpart of the classical on-policy method \SARSA. This method, which simultaneously updates and evaluates a policy, is fundamental in RL. 
Most existing works have studied either federated policy evaluation with TD learning \citep{wang2024federated,mitra2024temporal,mangold2024scafflsa,beikmohammadi2025collaborative}, methods like Q-learning \citep{khodadadian2022federated,woo2025blessing,zheng2024federated}, or policy gradient \citep{yang2024federated,lan2025asynchronous,labbi2025global}. 
A notable exception is \citet{zhang2024finite}, who analyzed the \FedSARSA algorithm and showed that federated training can reduce variance. However, when agents are heterogeneous, their convergence rate is affected by a persistent bias, which remains even with a single local update.

In this paper, we present a novel analysis of the \FedSARSA\ algorithm \emph{with multiple local training steps}, establishing explicit finite-time convergence bounds. We also introduce a federated variant of fitted \SARSA\ \citep{zou2019finite}, where the policy is updated only at communication rounds, guaranteeing that all agents are using the same policy at all times. Our analysis begins with a new framework for single-agent \SARSA, based on a novel expansion of the error along the trajectory that separates transient and fluctuation components, and tightly characterizes the impact of Markovian noise.
We then extend this novel analysis to the federated setting, establishing sharp convergence guarantees where we quantify explicitly  the impact of heterogeneity in local updates. Unlike prior work \citep{zhang2024finite}, our results do not rely on an averaged environment assumption, and directly characterize the local drift arising from multiple heterogeneous local steps, allowing for a tighter characterization of the impact of agent heterogeneity on \FedSARSA.

An important feature of our analysis is the characterization of the limiting point, to which \FedSARSA converges. This allows for a rigorous analysis of the algorithm when both transition kernels and reward functions differ across agents.
The main contributions of our work are the following:
\begin{enumerate}[itemsep=0.3em]
    \item We develop a novel analysis of federated \SARSA, precisely characterizing its convergence point and deriving explicit finite-time bounds that quantify the effect of environmental heterogeneity. We provide the first sample complexity bounds for \FedSARSA, showing that it achieves \emph{linear speed-up} in the number of agents. This follows from a careful analysis of the impact of Markovian noise, showing that higher-order variance terms scale proportionally to the chain's mixing time $\taumix$.
    \item At the core of our analysis lies a new analysis of single-agent \SARSA\ based on a refined error decomposition and a sharp characterization of the impact of Markovian noise, which is of independent interest.
    \item We provide numerical illustrations validating our theory and empirically demonstrating the linear speed-up of \FedSARSA. 
    We also provide a JAX implementation of a deep \FedSARSA variant, demonstrating the applicability of our method to federated deep RL.
\end{enumerate}
The paper is organized as follows: we introduce FRL in \Cref{sec:definitions}, and discuss related work in \Cref{sec:related-work}.
We present our analysis of single-agent \SARSA in \Cref{sec:convergence}, and our results on \FedSARSA in \Cref{sec:fed-sarsa}. 
We then illustrate our theory numerically in \Cref{sec:numerical-expe}.

\paragraph{Notations.}
We denote by $\pscal{\cdot}{\cdot}$ the Euclidean inner product and by $\norm{\cdot}$ its associated norm. 
All vectors are column vectors. We let $\Id$ be the $d \times d$ identity matrix and $e_j$ the $j$-th vector of the canonical basis of $\rset^d$. 
For a matrix $A \in \rset^{d,d}$, we denote $A_{i,j}$ its $i,j$-th coordinate, and for a vector $b \in \rset^d$, we denote $b_i$ its $i$-th coordinate.
For two sequences $(u_n)_{n \ge 0}$ and $(v_n)_{n \ge 0}$, we write $u_n \lesssim v_n$ if there exists $c>0$ such that $u_n \le c v_n$ for all $n \ge 0$, and $u_n \approx v_n$ if both $u_n \lesssim v_n$ and $v_n \lesssim u_n$. 
For a closed convex set $\projset \subseteq \rset^d$, $\Pi_{\projset}$ denotes the projection onto $\projset$. 
Finally, for a set $X$, $\mathcal{P}(X)$ denotes the set of probability measures on the measurable space $(X,\mathcal{B}(X))$, where $\mathcal{B}(X)$ is the Borel $\sigma$-field of $X$.

%% file: src/related-work.tex
\paragraph{\SARSA and RL.} 
\textit{(Tabular.)} The \SARSA algorithm (short for State–\!Action–\!Reward–\!State–\!Action), introduced by \citet{rummery1994line}, is a classical \emph{on-policy} reinforcement learning method. 
\citet{singh2000convergence} proved its convergence in the tabular case, showing that if the policy becomes greedy while maintaining exploration (the greedy-in-the-limit with infinite exploration condition), the state-action values converge to the unique optimal solution. 
This result extended earlier convergence proofs for \textit{off-policy} Q-learning~\citep{watkins1992q}.

\textit{(Linear function approximation.)}
In large or continuous state spaces, \SARSA is  combined with linear function approximation (LFA) for the state-action function using an embedding in $\rset^d$. 
\citet{tsitsiklis1997analysis} established the convergence of TD(0) to a fixed point of a linear equation; subsequent works analyzed it in both asymptotic and finite-time regimes~\citep{tsitsiklis1997analysis,bhandari2018finite,samsonov2024improved}. 
\SARSA generalizes TD learning~\citep{vamvoudakis2021handbook,meyn2022control}, estimating the state-action value of the current policy via TD updates while improving the policy. 
\citet{de2000existence} proved the existence of a solution for \SARSA\ using fixed-point arguments. 
Related convergence results have also been established for Q-learning~\citep{melo2008analysis} and actor-critic~\citep{wu2020finite,barakat2022analysis}.

\textit{(Convergence Rates for \SARSA.)} 
Extending TD to \SARSA\ requires a policy improvement operator $\polimprove$ (formally defined in \eqref{eq:polimp}), which updates the policy from the parameters learned via TD. 
In analyses of single-agent \SARSA, it is commonly assumed that $\polimprove$ satisfies a Lipschitz condition, with $\implip \ge 0$,
\begin{align*}
    \abs{ \polimprove(\param_1)(a | s) - \polimprove(\param_2)(a | s)}
    & \le 
      \implip \norm{ \param_1 - \param_2 },
\end{align*}
for all $s, a \in \statespace \times \actionspace$ and $\param_1, \param_2 \in \rset^d$. 
Setting $\implip=0$ recovers TD(0), without policy improvement. 
\citet{perkins2002convergent} proved convergence of a \SARSA\ variant with LFA, using multiple TD updates between improvements, small $\implip$, and $\varepsilon$-greedy policies. 
\citet{melo2008analysis} later established the first asymptotic convergence proof for \SARSA\ with LFA.
Under a similar Lipschitz assumption on~$\polimprove$, \citet{zou2019finite} gave a non-asymptotic convergence analysis of \SARSA with LFA, using projection onto a bounded ball at each step, akin to TD(0) with Markovian noise~\citep{bhandari2018finite}. 
More recently, \citet{zhang2023convergence} studied \SARSA with larger $\implip$, identifying a \textit{chattering} phenomenon~\citep{gordon1996chattering,gordon2000reinforcement}.

\paragraph{Federated Reinforcement Learning.}
Federated reinforcement learning (FRL) \citep{zhuo2019federated,qi2021federated} generalizes federated learning \citep{mcmahan2017communication,kairouz2021advances,li2020federated,ogier2022flamby} to sequential decision making.
Early theoretical work concentrated on policy evaluation, particularly federated TD–type methods under generative and online-interaction setting, establishing finite-time convergence guarantees and often linear-in-agents speedups \citep{khodadadian2022federated,dal2023federated,tian2024one}.
A parallel line of research addresses environment heterogeneity (agents facing distinct MDPs), characterizing how aggregation must adapt to model mismatch and under what conditions collaboration remains beneficial \citep{jin2022federated,woo2025blessing}. Beyond evaluation, tabular setting has received complexity analyses. For Q-learning, recent results establish linear speedups together with sharp—often optimal or near-optimal—sample and communication bounds \citep{khodadadian2022federated,woo2025blessing,zheng2024federated,salgia2024sample}. For value-iteration–style methods, federated regret bounds with linear speedups and explicit heterogeneity terms are now available \citep{labbi2025federated}. On the policy-gradient side (generative/simulator-oracle setting), theory now encompasses both natural policy gradient and actor–critic methods \citep{lan2023improved,yang2024federated,jordan2024decentralized}. In contrast, the online-interaction setting in FRL remains comparatively underexplored. A notable exception is \citet{zhang2024finite}, who analyze federated SARSA under agent heterogeneity and prove convergence to a heterogeneity-dependent neighborhood of the optimum: convergence cannot, in general, be made arbitrarily precise.

%% file: src/definitions.tex
\begin{algorithm*}[t]
\caption{\FedSARSA: Federated State-Action-Reward-State-Action}
\label{algo:fed-sarsa}
\begin{algorithmic}[1]
\STATE \textbf{Input:} step sizes $\step_t > 0$, initial parameters $\param_0$, projection set $\projset$, number of local steps $H > 0$, number of communications $\nepisode > 0$, initial distribution $\varrho$ over states
\STATE Initialize first state $s^{(c)}_{-1,\nlupdates} \sim \varrho$ for $c \in \intlist{1}{\nagent}$ and initial policy $\policy_{\globparam{0}} = \polimprove(\qfunc_{\globparam{0}}$
\FOR{step $t=0$ to $\nepisode-1$}
\FOR{agent $c=1$ to $\nagent$}
\STATE Initialize $\locparam{c}{t,0} = \globparam{t}$, take first action $a^{(c)}_{t,0} \sim \policy_{\globparam{t}}(\cdot|s^{(c)}_{t-1,\nlupdates})$
\FOR{step $h=0$ to $\nlupdates-1$}
\STATE Take action $a^{(c)}_{t, h+1} \sim \policy_{\globparam{t}}(\cdot|s^{(c)}_{t, h})$, observe reward $\nreward{c}(s^{(c)}_{t,h}, a^{(c)}_{t,h})$, next state $s^{(c)}_{t, h+1}$
\STATE Compute $\delta^{(c)}_{t,h} = \nreward{c}(s^{(c)}_{t, h+1}, a^{(c)}_{t,h+1}) + \discount \feature(s^{(c)}_{t, h+1}, a^{(c)}_{t, h+1})^\top \locparam{c}{t, h} - \feature(s^{(c)}_{t,h}, a^{(c)}_{t,h})^\top \locparam{c}{t, h}$
\STATE Update $\locparam{c}{t, h+1} =  \locparam{c}{t,h} + \step_{t} \delta^{(c)}_{t,h} \feature(s^{(c)}_{t,h}, a^{(c)}_{t,h}) $
\ENDFOR
\ENDFOR
\STATE Compute $\avgparam{t+1} = \frac{1}{\nagent} \sum_{c=1}^\nagent \locparam{c}{t,\nlupdates}$ and update of global parameter $\globparam{t+1} = \Pi\left( \avgparam{t+1} \right)$
\STATE Policy improvement $\policy_{\param_{t+1}} = \polimprove(\qfunc_{\param_{t+1}})$
\ENDFOR
\STATE \textbf{Return: } $\globparam{\nepisode}$
\end{algorithmic}
\end{algorithm*}

\paragraph{Federated Reinforcement Learning.}
In FRL, $\nagent > 0$ agents collaborate to learn a single shared policy. 
Formally, each agent's environment is modeled as a Markov Decision Process (MDP), yielding $\nagent$ MDPs
\(
\{(\statespace, \actionspace, \nkerMDP{c}, \nreward{c}, \discount)\}_{c \in \intlist{1}{\nagent}},
\)
with a common state space $\S$, action space $\A$, and discount factor $\gamma \in (0,1)$. 
Each agent $c \in \intlist{1}{\nagent}$ has its own transition kernel $\smash{\nkerMDP{c}}$, where $\smash{\nkerMDP{c}(\cdot|s,a)}$ denotes the probability of transitioning from state $s \in \S$ after action $a \in \A$, 
and a deterministic reward function $\smash{\nreward{c}: \S \times \A \to [0,1]}$. 
State-action pairs are embedded in $\rset^d$ via a feature map $\feature:(s,a)\mapsto \feature(s,a)$. 
For a policy $\smash{\policy_\param}$ parameterized by $\smash{\param \in \rset^d}$, we denote by $\smash{\npolkerMDP{c}{\param}}$ the induced state transition kernel and by $\smash{\nstab{c}{\param}}$ the stationary distribution satisfying $\smash{\nstab{c}{\param}\npolkerMDP{c}{\param} = \nstab{c}{\param}}$.
In this context, heterogeneity lies both in the transition kernel and the rewards.
We measure it using~$\kerhgty$ and~$\rewardhgty$, defined as
\begin{equation}
\begin{aligned}
\kerhgty 
& \eqdef
\sup_{\varrho \in \mathcal{P}(\S \times \A)}
\sup_{c, c' \in \intlist{1}{\nagent}}
\norm{ \varrho \nkerMDP{c}  -\varrho \nkerMDP{c'} }[\TV]
\eqsp,
\\[-0.2em]
\label{eq:def-heterogeneity-reward}
\rewardhgty
& \eqdef
\sup_{s, a \in \S \times \A}
\sup_{c, c' \in \intlist{1}{\nagent}} 
\abs{ \nreward{c}(s, a)  - \nreward{c'}(s,a) }
\eqsp,
\end{aligned}
\end{equation}
where $\kerhgty$ measures heterogeneity in the transition dynamics and $\rewardhgty$ heterogeneity in the rewards.
This measure of heterogeneity is classical in federated RL, and has been used in many prior works \citep{wang2024federated,zhang2024finite}.

\paragraph{Federated \SARSA.}
\SARSA\ combines TD learning \citep{sutton1988learning} with policy improvement: TD updates estimate the state-action value function, which is then used to update the policy. 
In its federated version (\FedSARSA), agents collaboratively learn a shared policy. 
Each agent performs several local TD updates, after which a central server aggregates the local estimators to update the global policy.

We approximate the state-action value function by a linear model $\qfunc_\param:(s,a)\mapsto \feature(s,a)^\top \param$ for $(s,a)\in \S\times \A$, where $\feature$ is a fixed embedding of state-action pairs. 
At global iteration $t \ge 0$ and local iteration $h \ge 0$, the parameter $\smash{\locparam{c}{t,h}\in\rset^d}$ of agent $c \in \intlist{1}{\nagent}$ is updated via the local TD rule
\begin{align*}
\locparam{c}{t,h+1}
& =
\locparam{c}{t,h}
+ \step_t \big( \nAs{c}{\locRandState{c}{t,h+1}} \locparam{c}{t,h} + \nbs{c}{\locRandState{c}{t,h+1}} \big)
\eqsp,
\end{align*}
where, for $h \ge 0$, $\smash{\locRandState{c}{t,h} = (S^{(c)}_{t,h}, A^{(c)}_{t,h}, S^{(c)}_{t,h+1}, A^{(c)}_{t,h+1})}$ takes values in $\msZ \eqdef (\S \times \A)^{\times 2}$, and represents the current and the next observed state-action pairs. The  TD error is defined at iteration $t,h$ by $\smash{\nAs{c}{\locRandState{c}{t,h+1}} \locparam{c}{t,h} + \nbs{c}{\locRandState{c}{t,h+1}}}$ where for $z= (s,a,s',a') \in \msZ$, define $\nAsw{c}$ and $\nbsw{c}$ as
\begin{align*}
 \nonumber   \nAs{c}{z}
    & =
        \feature(s, a) \left(
            \gamma \feature(s', a')^\top - \feature(s, a)^\top
        \right)
       \\
    \nbs{c}{z}
   & = \feature(s, a) \reward^{(c)}(s, a)
      \eqsp.
\end{align*}
After $\nlupdates > 0$ local TD updates are performed, the local parameters are sent to the server, which averages them and projects the result on a compact convex set $\projset \subseteq \rset^d$,
\begin{align*}
\globparam{t+1}
& =
\proj_{\projset}(\avgparam{t+1}) ,
\text{ where } 
{ \avgparam{t+1} =
\tfrac{1}{\nagent} \textstyle{\sum_{c=1}^\nagent} \locparam{c}{t,\nlupdates}}
\eqsp.
\end{align*}
The global, shared policy is then updated using the softmax of the approximated state-action value 
\begin{align}
\label{eq:polimp}
    \polimprove \colon \param \mapsto \policy_\param,
    \text{where }
    \policy_\param(a | s) \propto \exp( \temperature \feature(s, a)^\top \param )
\end{align}
for all $s, a \in \S \times \A$, and $\temperature > 0$ is referred to as  the \emph{sharpness} of the policy.
Note that this policy improvement operator is $\implip$-Lipschitz, with $\implip = \temperature$, that is that for $\param, \param'$, we have, for $(s, a) \in \S \times \A$,
\begin{align*}
\abs{ \policy_\param( a | s ) - \policy_{\param'}( a | s ) }
\le 
\implip \norm{ \param - \param' }
\eqsp.
\end{align*}
We give the pseudo-code for \FedSARSA in \Cref{algo:fed-sarsa}.

\paragraph{Assumptions.}
In the following, we assume the feature map to be bounded, which gives uniform bounds on $\nAs{c}{\cdot}$ and $\nbs{c}{\cdot}$ and restricts the diameter of $\projset$.
\begin{assum}
\label{assum:bounded-A-b} 
The state-action feature map $\feature$ is such that $\sup_{s, a \in \S \times \A} \norm{ \feature(s, a) } \le 1$.
This gives the almost sure bounds, %
for any $c \in \intlist{1}{\nagent}$,
\begin{align*}
\norm{ \nAsw{c} } &  \le \boundA \eqdef (1 + \discount)
\eqsp,
\quad
\norm{ \nbsw{c} } \le \boundb \eqdef 1
\eqsp,
\end{align*}
where for $z \in \msZ$, $\nAs{c}{z} \in \rset^{d \times d}$ and $\nbs{c}{z} \in \rset^d$, we defined the norms $\norm{\nAsw{c}}^2 = \sum_{1 \le i,j \le d} \sup_{z \in \msZ} \nAsw{c}_{i,j}(z)^2$ and $\norm{\nbsw{c}}^2 = \sum_{1 \le i \le d} \sup_{z \in \msZ} \nbsw{c}_i(z)^2$.
\end{assum}
Given a policy $\smash{\policy_\param}$ with stationary distribution $\smash{\nstab{c}{\param}}$ over the state space, we define for each agent $\smash{c \in \intlist{1}{\nagent}}$ a stationary distribution $\smash{\locstatdist{c}{\param}}$ on $\msZ$. 
A tuple $(s,a,s',a')$ is drawn by sampling $s \sim \nstab{c}{\param}$, $a \sim \policy_\param(\cdot|s)$, $s' \sim \nkerMDP{c}(\cdot|s,a)$, and $a' \sim \policy_\param(\cdot|s')$. 
We then define the expectations of $\nAs{c}{\cdot}$ and $\nbs{c}{\cdot}$ with respect to this distribution.
\begin{align}
 \nonumber   \nbarA{c}(\param)
    & =
      \PE_{z \sim \statdist{\param}} \!\big[
        \feature(s, a) \big(
            \gamma \feature(s', a')^\top \!\!-\! \feature(s, a)^\top
        \big)
      \big]
      \eqsp\! \\
      \eqsp
    \nbarb{c}(\param)
    & =
      \PE_{z \sim \statdist{\param}}\! \big[
        \feature(s, a) \reward^{(c)}(s, a)
      \big]
      \eqsp,
\label{eq:def-bar-Ab}
\end{align}
where the key difference with TD learning \citep{samsonov2024improved,mangold2024scafflsa} is that $z \sim \statdist{\param}$ depends on the parameter currently being optimized $\param$.
From these definitions, we see that a limit point $\paramlim \in \rset^d$ of the \FedSARSA\ algorithm must satisfy the equation
\begin{align}
\label{eq:def-paramlim-fed}
\tfrac{1}{\nagent} \textstyle{\sum_{c=1}^\nagent} \nbarA{c}(\paramlim)\paramlim 
+ \tfrac{1}{\nagent} \textstyle{\sum_{c=1}^\nagent} \nbarb{c}(\paramlim) = 0 \, .
\end{align}
We discuss the existence of such points in \Cref{sec:target-federated-point}.
We assume that the solutions of \eqref{eq:def-paramlim-fed} and $\projset$ satisfy the following.
\begin{assum}
\label{assum:contraction}
There exists $a > 0$ such that, for any $\paramlim$ satisfying \eqref{eq:def-paramlim-fed} and $c \in \intlist{1}{\nagent}$, $\nbarA{c}(\paramlim)$ is negative definite and the largest eigenvalue of $1/2(\nbarA{c}(\paramlim) + \nbarA{c}(\paramlim)^\top)$ is $-\lminAbar$.
\end{assum}
\begin{assum}
\label{assum:projset}
The set $\projset \subseteq \rset^d$ is large enough so that there exists a $\paramlim \in \projset$ satisfying \eqref{eq:def-paramlim-fed}.
Moreover, there exists $\projradius > 0$ such that $\sup_{\param \in \projset} \norm{ \param} \le \projradius$. 
\end{assum}
Under \Cref{assum:projset}, we also define the following two quantities, which bound intermediate \FedSARSA's iterates and updates,   
\begin{align}
\label{eq:def-locproj-boundgrad}
\locprojradius = 4 \projradius + 1
\text{ and }
\boundGrad \eqdef \boundA \locprojradius + \boundb
\eqsp.
\end{align}
The next two assumptions define the Markovian property of the noise and the variance of the Markov chains at stationarity, and are classical in the analysis of RL methods.
\begin{assum}
\label{assum:markov-chain}
The kernels $\nkerMDP{c}_\param$ have invariant distributions $\smash{\nstab{c}{\param}}$ and are uniformly geometrically ergodic. 
There exists $\taumix \ge 0$ such that, for all distributions $\varrho, \varrho'$ over $\msZ$, and $h \ge 0$,
\begin{align*}
\norm{ \varrho (\nkerMDP{c}_\theta)^h - \varrho' (\nkerMDP{c}_\theta)^h }[\TV] 
\le (1/4)^{\lfloor h / \taumix \rfloor}
\eqsp.
\end{align*} 
\end{assum}
Following previous work on \SARSA \citep{zou2019finite,zhang2024finite}, we assume the policy improvement operator's Lipschitz constant $\implip$ is not too large.
\begin{assum}
\label{assum:lipschitz-improvement}
The constant $\implip$ is such that
$80 \boundGrad \implip \nactions \taumix \le \lminAbar$.
\end{assum}
This assumption is classical in finite-sample analysis of \SARSA and \FedSARSA \citep{zhang2024finite}. To our knowledge, the only theoretical work that relaxes this assumption is \citet{zhang2023convergence}. However, without this assumption, \SARSA does not necessarily converge, and one can only show that \SARSA's iterates remain in a bounded domain. In this work, we aim to study the \emph{convergence} of \FedSARSA, requiring this assumption: extending our results to larger Lipschitz constants is a promising direction for future research on federated \SARSA.

%% file: src/convergence.tex
First, we present our novel analytical framework for the single-agent \SARSA algorithm, that is \Cref{algo:fed-sarsa} with $\nagent=1$, or \Cref{algo:single-sarsa} in \Cref{sec:proof-single-sarsa}.
In the single-agent case, we have $\nbarA{1}(\param) = \barA(\param)$, and we can define the solution $\paramlim$ as the vector that satisfies
\begin{align}
\label{eq:def-local-paramlimc}
\nbarA{1}(\paramlim) \paramlim + \nbarb{1}(\paramlim) = 0
\eqsp.
\end{align}
Existing analysis \citep{zou2019finite} guarantees that such a $\paramlim$ exists, and that it is unique under assumptions \Cref{assum:bounded-A-b}-\ref{assum:lipschitz-improvement}.
We propose a novel decomposition of the error of \SARSA as
\begin{align}
\nonumber
\locparam{1}{t,h+1} - \paramlim
= \,  &
\big( \Id + \step_t \nbarA{1}(\paramlim) \big)
(\locparam{1}{t,h} - \paramlim)
\\
& \quad + \step_t \locerrordet{1}{t,h} + \step_t \locnoisetheta{1}{t,h}{\locRandState{1}{t,h+1}}
\eqsp, \label{eq:local-rec-one-step}
\end{align}
where we introduced the notations
\begin{align}
\nonumber%
&\locerrordet{1}{t,h}
 =
( \nbarA{1}(\globparam{t}) - \nbarA{1}(\paramlim) ) \locparam{1}{t, h} 
+ \nbarb{1}(\globparam{t}) - \nbarb{1}(\paramlim) 
\eqsp 
\label{eq:def-single-locnoisetheta}
\\
& \locnoisetheta{1}{t,h}{\locRandState{1}{t,h+1}}
 =
( \nAs{1}{\locRandState{1}{t,h+1}} - \nbarA{1}(\globparam{t}) ) \locparam{1}{t, h} 
\\
&\qquad~\quad\qquad\qquad\qquad + \nbs{1}{\locRandState{1}{t,h+1}} - \nbarb{1}(\globparam{t})
\eqsp. \nonumber
\end{align}
The term $\locerrordet{1}{t,h}$ accounts for the discrepancy between  the current policy parameterized by $\globparam{t}$ and the optimal one, parameterized by $\paramlim$.
The second term $\smash{\locnoisetheta{1}{t,h}{\locRandState{1}{t,h+1}}}$ is a fluctuation term.
Now, we define the following matrices, which will appear in all subsequent error decompositions,
\begin{align}
\label{eq:def-loccontract-1}
\loccontract{1}{t,k:h} &=
\big( \Id + \step_{t} \nbarA{1}(\paramlim) \big)^{k-h+1}
\eqsp,
\text{ for } k, h \ge 0
\eqsp,
\end{align}
with the convention that $\loccontract{1}{t,k:h} = \Id$ when $h < k$.
Unrolling \eqref{eq:local-rec-one-step} gives the following decomposition.
\begin{restatable}{claim}{decompositionoflocalupdate}
\label{claim:decomposition-error}
Let $t \ge 0$.
The updates of \SARSA at block $t$ can be written as 
\begin{align}
\nonumber
& \locparam{1}{t,\nlupdates} - \paramlim
=
\textstyle
\loccontract{1}{t,1:\nlupdates}  \left( \globparam{t} - \paramlim \right)
 + \sum_{h=1}^{H} \step_{t} \loccontract{1}{t,h+1:H}    \locerrordet{1}{t,h-1}
\\  & ~~~~~~~~~~~~~~~~~~~~~~ + \textstyle \sum_{h=1}^{H} \step_{t} \loccontract{1}{t,h+1:H} \locnoisetheta{1}{t,h-1}{\locRandState{1}{t,h}}
\eqsp.
\label{eq:decomposition-H-updates-appendix-claim}
\end{align}
\end{restatable}
We prove this claim in \Cref{sec:app-error-decomposition}.
We are now ready to prove the following lemma, which shows that \SARSA reduces the error between consecutive blocks of updates. %
\begin{restatable}{lemma}{singleagentonestepimprovement}
\label{lem:one-step-improvement}
Assume \Cref{assum:bounded-A-b}--\ref{assum:lipschitz-improvement}.
Let $t \ge 0$, assume that the step size satisfies $\step_t \nlupdates \boundA \le 1/6$. %
Then,  it holds that
\begin{align*}
\Ee[ \norm{ \locparam{1}{t,\nlupdates} - \paramlim }^2 ]
& \le
(1 \!-\! \tfrac{\step_t \lminAbar\nlupdates}{4} ) \norm{ \globparam{t} \!-\! \paramlim }^2
 \!+\! \cstthm[c_1]{136} \step_t^2 \nlupdates \taumix \boundGrad^2 %
\\
& \quad 
+ \startedfromstationary \tfrac{ \cstthm[c_1]{58} \step_t \taumix^2 \boundGrad^2 }{\nlupdates \lminAbar} 
+ \tfrac{\cstthm[c_1]{976} \step_t^3 \nlupdates \taumix^2 \boundGrad^2 \boundA^2 }{\lminAbar} 
\eqsp,
\end{align*}%
where \cstthm[$c_1 > 0$ is an absolute constant]{} and $\startedfromstationary=0$ if episodes start in the stationary distribution and $\startedfromstationary=1$ otherwise.
\end{restatable}
We prove this lemma in \Cref{sec:proof-single-sarsa}.
The proof is based on the error expansion from \Cref{claim:decomposition-error}. The first term is a transient term, which decreases linearly towards zero, the second term is a fluctuation term and third term is an error term due to sampling from the ``wrong'' policy.
The first term can be bounded using \Cref{assum:contraction}, and the third term's bound follows from \Cref{assum:lipschitz-improvement}.
Due to the Markovian nature of the noise, the second term requires a very careful examination, in order to handle all the correlation between pairs of iterates. In \Cref{sec:markov}, we provide an analysis of this Markovian error, with tight bounds depending on $\nlupdates$ and $\taumix$.

Note that, when $\startedfromstationary=1$, one of the error term scales in $\step_t / \nlupdates$: controlling it requires setting $\nlupdates \ge \taumix$.
This is unavoidable, since when $\startedfromstationary = 1$, it is necessary to do $\taumix$ updates to get close to the stationary distribution of the Markov chain.
One can eliminate this term by skipping about $\taumix$ samples before updating $\globparam{t}$, ensuring that the updates start in the stationary distribution with high probability.
We can now state our main theorem in the single-agent setting, which gives a convergence rate for \SARSA.
\begin{restatable}{theorem}{singleagentconvergencerate}
\label{thm:convergence-rate-single-agent}
Assume \Cref{assum:bounded-A-b}--\ref{assum:lipschitz-improvement}.
Assume that the step size $\step_t = \step$ is constant and satisfies $\step \nlupdates \boundA \le 1/5$ and that $\nlupdates \ge \taumix$.
Then it holds that
\begin{align*}
 \Ee[ \norm{ \globparam{\nepisode} - \paramlim }^2 ]
& \lessthm
(1 - \tfrac{\step \lminAbar\nlupdates}{4} )^\nepisode \norm{ \globparam{0} - \paramlim }^2
 + \tfrac{\cstthm[c_1]{544} \step \taumix \boundGrad^2}{\lminAbar} %
\\
& \quad 
+ \startedfromstationary \tfrac{ \cstthm[c_1]{232}  \taumix^2 \boundGrad^2 }{\nlupdates^2 \lminAbar^2} 
+ \tfrac{\cstthm[c_1]{3904} \step^2 \taumix^2 \boundGrad^2 \boundA^2 }{\lminAbar^2} 
\eqsp,
\end{align*}
where $\startedfromstationary$ is defined in \Cref{lem:one-step-improvement}.
\end{restatable}
We state this theorem with explicit constants and prove it in \Cref{sec:proof-single-sarsa}.
It shows that \SARSA converges linearly to a neighborhood of $\paramlim$, and that the size of this neighborhood is determined by the step size, the variance of the updates, and the mixing time~$\taumix$.
\begin{restatable}{corollary}{singleagentcomplexity}
\label{cor:complexity-single-agent}
Assume \Cref{assum:bounded-A-b}--\ref{assum:lipschitz-improvement}.
Let ${\epsilon > 0}$, set $\step 
\approx
\min\big( \frac{1}{\boundA}, \frac{\lminAbar \epsilon^2}{\boundGrad^2 \taumix}, \frac{\lminAbar \epsilon}{\boundGrad \boundA \taumix} \big)$ and $\nlupdates
\approx
\max\big( 1 , \frac{\boundGrad \taumix}{\lminAbar \epsilon} \big)$, then \SARSA reaches $\Ee[\norm{\globparam{\nepisode}\!-\!\paramlim}^2] \lesssim \epsilon^2$ %
with 
\begin{align*}
    {\nepisode \nlupdates \approx
\max\big( \tfrac{\boundA}{\lminAbar} , \tfrac{\boundGrad^2 \taumix}{\lminAbar^2 \epsilon^2}, \tfrac{\boundA \boundGrad \taumix}{\lminAbar^2 \epsilon} \big)
\log\big(
\tfrac{\norm{ \globparam{0} \!-\! \paramlim }^2}{\epsilon}
\big)}
\end{align*}
samples and $\nepisode \gtrsim \frac{\boundA}{\lminAbar} \log\big( \frac{\norm{ \globparam{0} \!-\! \paramlim }^2}{\epsilon} \big)$ policy~updates.
\end{restatable}
The proof is in \Cref{sec:proof-single-sarsa}.
This shows that, with proper hyperparameter settings, single-agent \SARSA\ reaches a solution with mean squared error $\epsilon^2$ using $O(\log(1/\epsilon))$ policy improvements and $O(1/\epsilon^2 \log(1/\epsilon))$ samples.
It highlights the relevance of keeping the policy constant during blocks, reducing the need for policy improvement steps, while keeping the same overall sample complexity.
\begin{remark}
Our analysis can be extended to the setting where the policy is updated after each sample, by bounding the difference between the samples obtained with the fixed policy $\smash{\policy_{\globparam{t}}}$ and the updated policy $\smash{\policy{\raisebox{-0.3ex}{$\scriptstyle \locparam{1}{t,h}$}}}$, as proposed by \citet{zou2019finite}.
We refrain from extending our analysis to this setting, since, in federated settings, one may desire the policy to remain identical for all agents at all times.
\end{remark}

%% file: src/federated-sarsa.tex
We now present our main result, establishing the global convergence of the \FedSARSA algorithm to a point $\paramlim$.
To this end, we first establish existence and uniqueness of $\paramlim$, in \Cref{sec:target-federated-point}.
We then extend the methodology that we introduced in \Cref{sec:convergence} to \FedSARSA in \Cref{sec:rate-for-fedsarsa}, establishing the first convergence result for \FedSARSA and the corresponding sample and communication complexity.

\paragraph{Limit Point of \textsf{FedSARSA}.}
\label{sec:target-federated-point}

To identify the limit of the \FedSARSA algorithm, we consider the idealized, deterministic \FedSARSA algorithm, where the local updates are replaced by their expected value, and only a single local step is performed.
This gives the global parameter update
\begin{align}
\label{eq:det-fed-sarsa-update}
& \globparam{t+1}
 = \textstyle
\proj_{\projset} \left( 
\globparam{t} 
+ \step_t \kappa_t \right) \, , \\
& \text{where } \kappa_t = 
\textstyle
\tfrac{1}{\nagent} \sum_{c=1}^\nagent \nbarA{c}(\globparam{t}) \globparam{t}
+ \tfrac{1}{\nagent} \sum_{c=1}^\nagent \nbarb{c}(\globparam{t})
\eqsp. \nonumber
\end{align}
This algorithm must converge to a point $\paramlim$ that is a fixed point of \eqref{eq:det-fed-sarsa-update}.
However, the existence of such a point $\paramlim$ is not straightforward. 
The main difficulty lies in the fact that the matrices $\nbarA{c}(\cdot)$ and $\nbarb{c}(\cdot)$ depend on the current policy.
Indeed, for a fixed policy parameter~$\omega$, finding~$\omega_\star$ such that $\smash{\frac{1}{\nagent}\sum_{c=1}^\nagent \nbarA{c}(\omega) \omega_\star + \frac{1}{\nagent}\sum_{c=1}^\nagent \nbarb{c}(\omega) = 0}$ boils down to the federated TD learning algorithm with linear approximation, which is known to converge \citep{mangold2024scafflsa}.
In the next proposition, we extend this result to the fixed points of the \FedSARSA update, establishing the existence and unicity of a solution of \eqref{eq:def-paramlim-fed}. %
\begin{restatable}{proposition}{propconvergencefedsarsatothetastar}
\label{prop:existence-theta-star-fed}
Assume \Cref{assum:bounded-A-b}--\ref{assum:lipschitz-improvement}. %
There exists a unique parameter $\paramlim \in \projset$ such that \begin{equation*}
    \textstyle 
    \frac{1}{\nagent} \sum_{c=1}^\nagent \nbarA{c}(\paramlim) \paramlim +  \frac{1}{\nagent} \sum_{c=1}^\nagent \nbarb{c} (\paramlim) = 0
    \eqsp.
\end{equation*}
\end{restatable}
We postpone the proof to \Cref{sec:app-convergence-fedsarsa}.
The definition of $\paramlim$ as a solution of this equation is crucial.
In other FRL works \citep{wang2024federated,zhang2024finite}, heterogeneity is handled by introducing a \emph{virtual environment} using averaged transitions and rewards from all environments.
Unfortunately, studying convergence to the optimal parameters that correspond to this average environment leads to non-vanishing bias.
In contrast, our approach will allow to show convergence to arbitrary precision towards to the unique fixed point~$\paramlim$ defined in \Cref{prop:existence-theta-star-fed}. 
The global fixed point~$\paramlim$ can be related to the local optimum. %
\begin{restatable}{proposition}{propdistanceloctoglobal}
\label{prop:propdistanceloctoglobal}
Assume \Cref{assum:bounded-A-b}--\ref{assum:lipschitz-improvement}.
For any $c \in \intlist{1}{\nagent}$, assume that $\paramlimc[c] \in \projset$, then the local optimum $\smash{\paramlimc[c]}$ (defined analogously to \eqref{eq:def-local-paramlimc}) satisfies
\begin{align*}
\norm{ \paramlimc[c] - \paramlim } 
\lessthm
\cstthm{\frac{480}{79}}
(1 + \taumix) (\kerhgty \norm{ \paramlim } + \rewardhgty)%
\eqsp,
\end{align*}
where $\kerhgty$ and $\rewardhgty$ are defined in~\eqref{eq:def-heterogeneity-reward}.%
\end{restatable}
This shows that when agents are increasingly homogeneous (\ie, $\kerhgty$ and $\rewardhgty$ get closer to zero), the shared parameter $\paramlim$ and the local optimums $\smash{\paramlimc[c]}$ become closer.
The proof is postponed to \Cref{sec:app-convergence-fedsarsa}.

\paragraph{Convergence Rate of \textsf{FedSARSA}.}
\label{sec:rate-for-fedsarsa}

Now that we identified the limit point of \FedSARSA, we can decompose its error similarly to \SARSA in \Cref{claim:decomposition-error}.
To this end, we follow \eqref{eq:def-loccontract-1} and define the matrices $\smash{\loccontract{c}{t,k:h} \eqdef ( \Id + \step_{t} \nbarA{c}(\paramlim) )^{k-h+1}}$ for $k \ge h \ge 0$ and $c \in \intlist{1}{\nagent}$, with the convention $\smash{\loccontract{c}{t,k:h} = \Id}$ for $k < h$, and $\paramlim$ as defined in \Cref{prop:existence-theta-star-fed}.
We define, for $c \in \intlist{1}{\nagent}$,
\begin{align*}
&\locerrordet{c}{t,h}
 =
( \nbarA{c}(\globparam{t}) - \nbarA{c}(\paramlim) ) \locparam{c}{t, h} 
+ \nbarb{c}(\globparam{t}) - \nbarb{c}(\paramlim) 
\\
&\locnoisetheta{c}{t,h}{\locRandState{c}{t,h+1}}
 =
( \nAs{c}{\locRandState{c}{t,h+1}} - \nbarA{c}(\globparam{t}) ) \locparam{c}{t, h} 
\\&\qquad\qquad\qquad\quad+ \nbs{c}{\locRandState{c}{t,h+1}} - \nbarb{c}(\globparam{t})
\eqsp,
\end{align*}
which are the federated counterparts of~\eqref{eq:def-single-locnoisetheta}. %
We also define the local limit parameter $\smash{\locparamlim{c}}$ as the solution of the local equation, when samples are collected from the global optimal policy $\paramlim$,
\begin{equation*}
\nbarA{c}(\paramlim) \locparamlim{c} 
+  \nbarb{c} (\paramlim) = 0
\eqsp.
\end{equation*}
The point $\locparamlim{c}$ is used solely for analysis purposes, and allows to measure heterogeneity through the two following quantities, which we relate to $\kerhgty$ and $\rewardhgty$ from~\eqref{eq:def-heterogeneity-reward}.
\begin{restatable}{proposition}{boundonheterogeneityconstants}
\label{prop:heterogeneity-constants}
Assume \Cref{assum:bounded-A-b}, \Cref{assum:contraction}, \Cref{assum:markov-chain}, and \Cref{assum:lipschitz-improvement}.
For $c \in \intlist{1}{\nagent}$, there exists $\hgtyA, \hgtyAtheta \ge 0$ such that
\begin{gather*}
\textstyle 
\norm{ \nbarA{c}(\paramlim) \!-\! \barA(\paramlim) }^2 \!\le\! \hgtyA^2
~,~~
\norm{ \nbarA{c}(\paramlim) (\locparamlim{c} \!-\! \paramlim) }^2 \!\le\! \hgtyAtheta^2
\,,
\end{gather*}
where we introduced the constants $\smash{\hgtyA \eqdef 4 \boundA (1 + \taumix) \kerhgty}$ and $\smash{\hgtyAtheta \eqdef 6 (1 + \taumix) (\kerhgty \norm{ \paramlim } + \rewardhgty)}$. %
\end{restatable}
The proof is postponed to \Cref{sec:app-heterogeneity-drift-bound}.
We measure two types of heterogeneity: $\smash{\hgtyA}$ measures heterogeneity of the matrices $\smash{\nbarA{c}(\paramlim)}$ themselves, while $\smash{\hgtyAtheta}$ relates the discrepancy between the global and local solutions of TD learning when following the global optimal policy.
Similarly to \Cref{claim:decomposition-error}, we obtain a decomposition of the federated update.
\begin{claim}
\label{claim:decomposition-error-federated}
For $t \ge 0$, the global updates of \Cref{algo:fed-sarsa} satisfy, before projection,
\begin{align*}
\avgparam{t+1} 
& - \paramlim
=
\textstyle
\frac{1}{\nagent} \sum_{c=1}^\nagent \loccontract{c}{1:\nlupdates}  \left( \globparam{t} - \paramlim \right)
\!+\! \biasterm
\\[-0.3em]
&
\textstyle
~ +\frac{\step_t}{\nagent} \sum_{c=1}^\nagent \sum_{h=1}^{H}\loccontract{c}{h+1:H} \Big( \locnoisetheta{c}{t,h-1}{\locRandState{c}{t,h}} \!+\! \locerrordet{c}{t,h-1} \Big) ,
\end{align*}
where $\biasterm \eqdef \frac{1}{\nagent} \sum_{c=1}^\nagent ( \Id - \loccontract{c}{1:\nlupdates}  ) ( \locparamlim{c} - \paramlim )$ accounts for bias due to heterogeneity.
\end{claim}
Analogously to \Cref{lem:one-step-improvement} in the single-agent case, we bound the progress in-between policy updates.%
\begin{restatable}{lemma}{federatedonestepimprovement}
\label{lem:federated-one-step-improvement}
Assume \Cref{assum:bounded-A-b}, \Cref{assum:contraction}, \Cref{assum:markov-chain}, and \Cref{assum:lipschitz-improvement}.
Let $t \ge 0$, assume that the step size satisfies $\step_t \nlupdates \boundA \le 1/6$.
Then, it holds that, \cstthm[for some universal constant $c_2 > 0$,]{}
\begin{align*}
&\Ee[ \norm{ \avgparam{t+1} \!-\! \paramlim }^2 ]
 \!\le\!
(1\! -\! \tfrac{\step_t \lminAbar\nlupdates}{8} ) \norm{ \globparam{t} \!-\! \paramlim }^2\!\!+\! \tfrac{\cstthm[c_2]{9} \step_t^3 \nlupdates (\nlupdates\!-\!1)^2}{\lminAbar} \hgtyA^2 \hgtyAtheta^2
\\
& \quad\quad\quad %
+ \tfrac{\cstthm[c_2]{136} \step_t^2 \nlupdates \taumix \boundGrad^2}{\nagent} %
+ \tfrac{\cstthm[c_2]{58}   \startedfromstationary  \step_t \boundGrad^2 \taumix^2 }{\nlupdates \lminAbar} 
+ \tfrac{\cstthm[c_2]{976} \step_t^3 \boundGrad^2 \boundA^2 \nlupdates \taumix^2}{\lminAbar} 
\eqsp,
\end{align*}
where $\smash{\startedfromstationary=0}$ if the $\smash{\locRandState{c}{t,0}}$ are sampled from the stationary distribution $\smash{\locstatdist{c}{\globparam{t}}}$ and $\smash{\startedfromstationary=1}$ otherwise.
\end{restatable}
We prove this lemma in \Cref{sec:app-convergence-rate-federated}.
The proof essentially follows the same structure as the proof of \Cref{lem:one-step-improvement}, using the error decomposition from \Cref{claim:decomposition-error-federated}.
First, we note that the transient terms and error due to sampling from the ``wrong'' policy are handled in the same way.
The analysis differs with the single-agent case in two crucial ways.
First, environment heterogeneity induces an additional error term, which increases with $\epsilon_p$ and $\epsilon_r$, scaling with the constants defined in \Cref{prop:heterogeneity-constants}.
Second, the leading variance terms \emph{decrease with the number of agents}: this allows \FedSARSA to achieve reduced sample complexity per agent, which is essential in federated RL. 
Moreover, our sharp analysis technique allows us to show that higher-order terms (due to the Markovian nature of the noise), only increase with $\taumix$, and not with $\nlupdates$.
This is in stark contrast with existing analyses (\eg, \citet{zhang2024finite}), and allows to derive improved sample complexity.
We now state our main theorem, assessing the convergence of \FedSARSA.

\begin{figure*}[t]
  \centering
  \vspace{-0.5em}
  
  \includegraphics[width=0.8\textwidth]{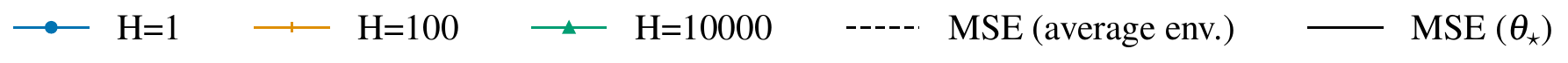}

  \vspace{-0.3em}

  \rotatebox{90}{\qquad\quad\quad\quad ~~ \small MSE}
  \begin{subfigure}[b]{0.24\linewidth}
         \centering
         \includegraphics[width=\textwidth]{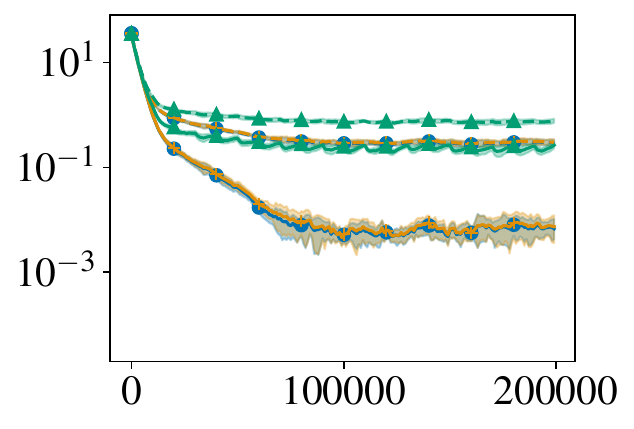}\\[-0.5em]
         {\small \quad Samples per Agent}

         \caption{$\nagent=2$}
         \label{fig:convergence-multiple-h-garnet}
     \end{subfigure}
     \begin{subfigure}[b]{0.24\linewidth}
         \centering
         \includegraphics[width=\textwidth]{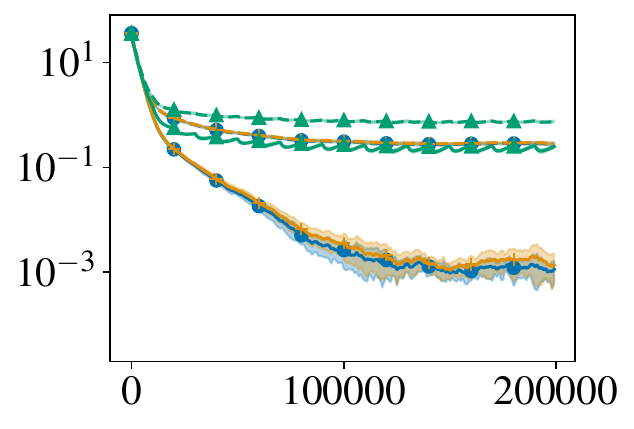}\\[-0.5em]

         {\small \quad Samples per Agent}

         \caption{$\nagent=10$}
         \label{fig:convergence-multiple-h-gridworld}
     \end{subfigure}
     \begin{subfigure}[b]{0.24\linewidth}
         \centering
         \includegraphics[width=\textwidth]{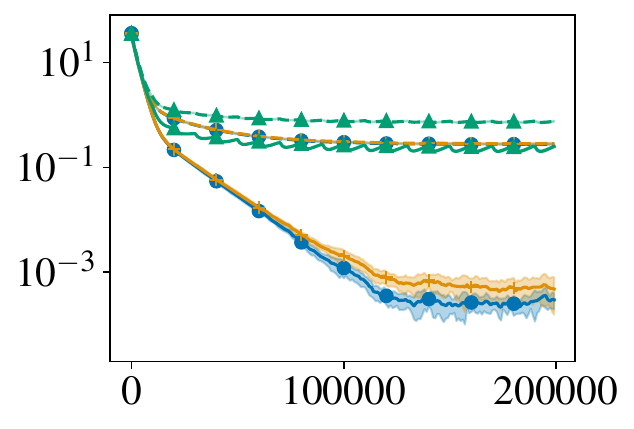}\\[-0.5em]
         {\small \quad Samples per Agent}

         \caption{$\nagent=50$}
         \label{fig:convergence-linear-speed-up-garnet}
     \end{subfigure}
     \begin{subfigure}[b]{0.24\linewidth}
         \centering
         \includegraphics[width=\textwidth]{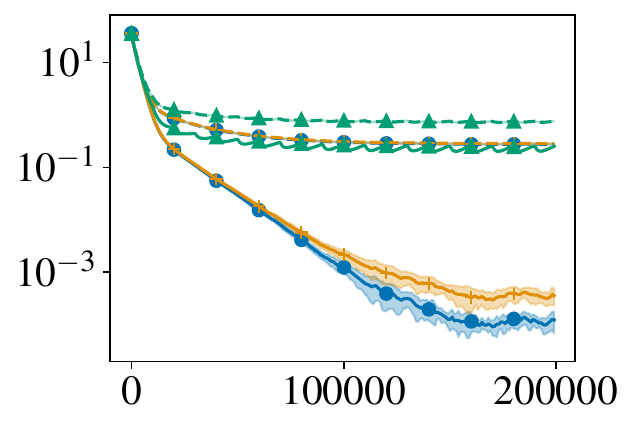}\\[-0.5em]
         {\small \quad Samples per Agent}
         
         \caption{$\nagent=100$}
         \label{fig:convergence-linear-speed-up-gridworld}
     \end{subfigure}
    \caption{MSE as a function of the number of communications. 
    For each run, we report two errors: (i) in solid lines, the error $\PE[ \norm{ \globparam{t} - \paramlim }^2 ]$ in estimating $\paramlim$ as defined in \Cref{prop:existence-theta-star-fed}, and (ii) in dashed lines, the error $\PE[ \norm{ \globparam{t} - \chi_\star }^2 ]$ in estimating $\chi_\star$, the limit of \SARSA on the averaged environment. 
    For each plot, we report the average over $10$ runs and the corresponding standard deviation.}
    \label{fig:results}
    \vspace{-0.5em}
\end{figure*}

\begin{restatable}{theorem}{federatedconvergencerate}
\label{thm:convergence-rate-federated}
Assume \Cref{assum:bounded-A-b}, \Cref{assum:contraction},  \Cref{assum:markov-chain}, and \Cref{assum:lipschitz-improvement}, that the step size $\step_t = \step$ is constant and satisfies $\step \nlupdates \boundA \le 1/5$ and that $\nlupdates \ge \taumix$.
Then it holds that
\begin{align*}
& \Ee[ \norm{ \globparam{\nepisode} \!- \!\paramlim }^2 ]
 \lessthm
(1 \!-\! \tfrac{\step \lminAbar\nlupdates}{8} )^\nepisode \norm{ \globparam{0} \!-\! \paramlim }^2+ \tfrac{\cstthm[c_2]{72} \step^2 (\nlupdates-1)^2}{\lminAbar^2} \hgtyA^2 \hgtyAtheta^2
\\
& \quad \quad \quad \quad \quad
+ \tfrac{\cstthm[c_2]{1088} \step \taumix \boundGrad^2}{\nagent \lminAbar} %
+ \tfrac{\cstthm[c_2]{464}   \startedfromstationary \boundGrad^2 \taumix^2 }{\nlupdates^2 \lminAbar^2} 
+ \tfrac{\cstthm[c_2]{7808} \step^2 \boundGrad^2 \boundA^2 \taumix^2}{\lminAbar^2} 
\eqsp,
\end{align*}
where $\smash{\startedfromstationary=0}$ if the $\smash{\locRandState{c}{t,0}}$ are sampled from the stationary distribution $\smash{\locstatdist{c}{\globparam{t}}}$ and $\smash{\startedfromstationary=1}$ otherwise.
\end{restatable}
The two main differences with the single-agent cases are that (i) leading variance terms decrease in $1/\nagent$, and (ii) heterogeneity induces additional error, smaller than
\begin{align*}
\tfrac{\step^2 (\nlupdates\!-\!1)^2}{\lminAbar^2} \hgtyA^2 \hgtyAtheta^2
& \!\lesssim\!
\tfrac{\step^2 (\nlupdates\!-\!1)^2}{\lminAbar^2} \boundA^2 (1\! + \!\taumix)^4 \kerhgty^2 (\kerhgty \norm{ \paramlim } \!+\! \rewardhgty)^2.
\end{align*}
This term decreases to zero with the step size $\step$ and disappears when $\nlupdates=1$, which is in stark contrast with existing work \citep{zhang2024finite}, which exhibit a non-vanishing bias when $\kerhgty \neq 0$ or $\rewardhgty \neq 0$.
Interestingly, this term decreases to zero as kernel heterogeneity $\kerhgty \rightarrow 0$, \emph{even if the rewards are heterogeneous}, which is in line with previous observations in the literature \citep{zhu2024towards,yang2024federated,labbi2025federated}.
We now state the sample and communication complexity of \FedSARSA.
\begin{restatable}{corollary}{corollaryfedsarsasamplecommcomplexity}
\label{cor:complexity-federated}
Assume \Cref{assum:bounded-A-b}--\ref{assum:lipschitz-improvement}.
Let $\epsilon > 0$.
Set $\step \approx \min\big( \frac{1}{\boundA},
\frac{\nagent \lminAbar \epsilon^2}{\boundGrad^2 \taumix},
\frac{\lminAbar \epsilon}{ \boundGrad \boundA  \taumix}
\big)$, and $\nlupdates$ such that $\nlupdates \lesssim \frac{\taumix \boundGrad}{\hgtyA \hgtyAtheta} \max\big(
\tfrac{\boundGrad}{\nagent \epsilon}, \boundA\big)$ and $\nlupdates \gtrsim \frac{\startedfromstationary \taumix \boundGrad}{\lminAbar \epsilon}$,
then \FedSARSA reaches $\Ee[\norm{\globparam{\nepisode}-\paramlim}^2] \lesssim \epsilon^2$ with $\nepisode \gtrsim 
\max\big( \frac{\boundA}{\lminAbar}, \frac{\hgtyA \hgtyAtheta }{ \lminAbar^2 \epsilon} \big) \log\big(
\frac{\norm{ \globparam{0} - \paramlim }^2}{\epsilon}
\big)$ communications, and 
\begin{align}
\label{eq:bound-th}
\nepisode \nlupdates
\approx 
\max\big( \tfrac{\boundA}{\lminAbar},
\tfrac{\boundGrad^2 \taumix}{\nagent \lminAbar^2 \epsilon^2},
\tfrac{ \boundGrad \boundA  \taumix}{\lminAbar^2 \epsilon}
\big)
\log\big(
\tfrac{\norm{ \globparam{0} - \paramlim }^2}{\epsilon}
\big)
\end{align}
samples per agent.
\end{restatable}
We prove this corollary in \Cref{sec:app-convergence-rate-federated}.
This corollary shows that \FedSARSA can exploit the experience of multiple agents to accelerate training, a property in FRL, typically known as \emph{linear speed-up}. This effect is most pronounced when high precision is required and heterogeneity is moderate, in which case each agent takes
$\smash{\nepisode\nlupdates \approx 
\tfrac{\boundGrad^2 \taumix}{\nagent \lminAbar^2 \epsilon^2} 
\log\!\left(\tfrac{\norm{ \globparam{0} - \paramlim }^2}{\epsilon}\right)}$
samples. In other regimes, other terms in the maximum may dominate. This occurs when (i) low-precision results suffice (\ie, large $\epsilon$), or (ii) heterogeneity is high (\ie, large $\hgtyA \hgtyAtheta$).  

Note that the linear speed-up is not unconditional and is limited by two phenomena that cannot be avoided.
First, the algorithm cannot be faster than its deterministic counterpart, which gives the first term of the max in \eqref{eq:bound-th}.
Second, due to the Markovian nature of the noise, higher-order terms scaling in $\step^2 \taumix^2$ remain in the rate of \Cref{thm:convergence-rate-federated}), which gives the third term of the max in \eqref{eq:bound-th}. 
Nonetheless, we stress that previous analyses exhibited a $\step^2 \taumix \nlupdates$ terms, which we reduce to $\step^2 \taumix^2$: this is a direct consequence of our tighter analysis of Markovian noise.
Finally, we note that in heterogeneous environments, the number of communications scales polynomially in $O(1/\epsilon)$. Reducing this to $O(\log(1/\epsilon)$ is a promising open question, which could be tackled by communicating in between policy updates, or using heterogeneity-correction methods such as \SCAFFOLD or \SCAFFLSA \citep{karimireddy2020scaffold,mangold2024scafflsa}.

%% file: src/expe.tex
\begin{figure*}[t]
  \centering
  \vspace{-0.5em}

  \rotatebox{90}{\qquad\qquad\quad ~~ \small Rewards}
  \begin{subfigure}[b]{0.24\linewidth}
         \centering
         \includegraphics[width=\textwidth]{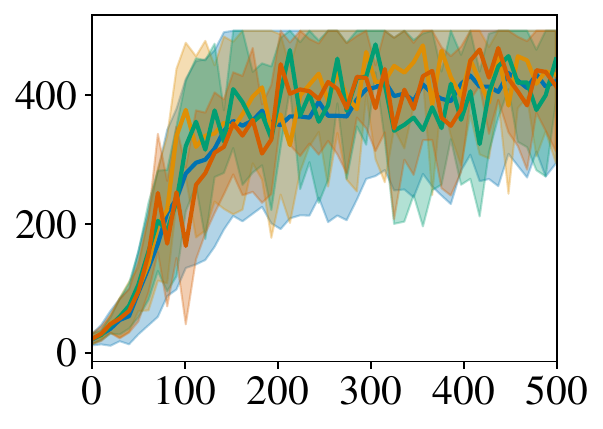}\\[-0.5em]
         {\quad \small  Episodes}

         \caption{$\nagent=2$}
         \label{fig:cartpole-N2}
     \end{subfigure}
     \begin{subfigure}[b]{0.24\linewidth}
         \centering
         \includegraphics[width=\textwidth]{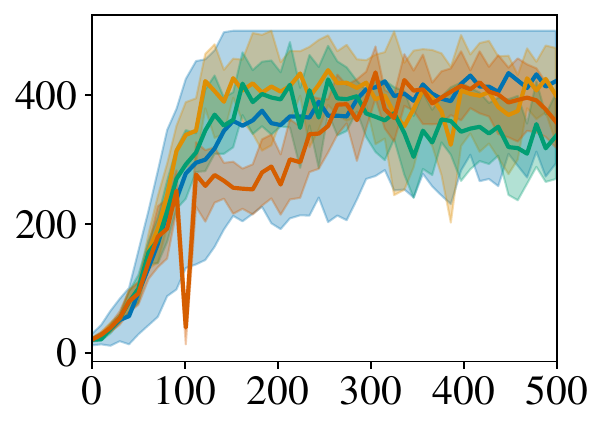}\\[-0.5em]
         
        {\quad \small  Episodes}
        
         \caption{$\nagent=10$}
         \label{fig:cartpole-N10}
     \end{subfigure}
     \begin{subfigure}[b]{0.24\linewidth}
         \centering
         \includegraphics[width=\textwidth]{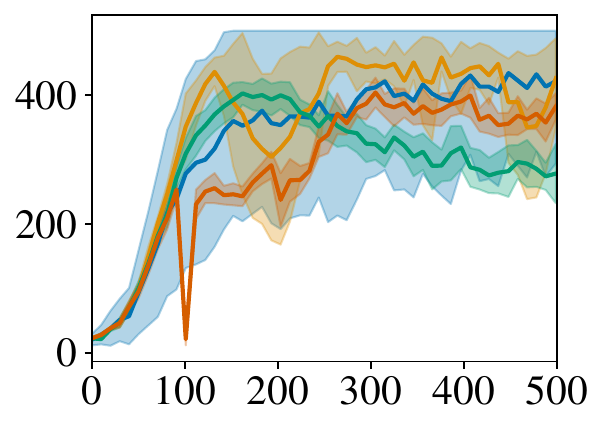}\\[-0.5em]
         
         {\quad \small  Episodes}

         \caption{$\nagent=50$}
         \label{fig:cartpole-N50}
     \end{subfigure}
     \begin{subfigure}[b]{0.24\linewidth}
         \centering
         \includegraphics[width=\textwidth]{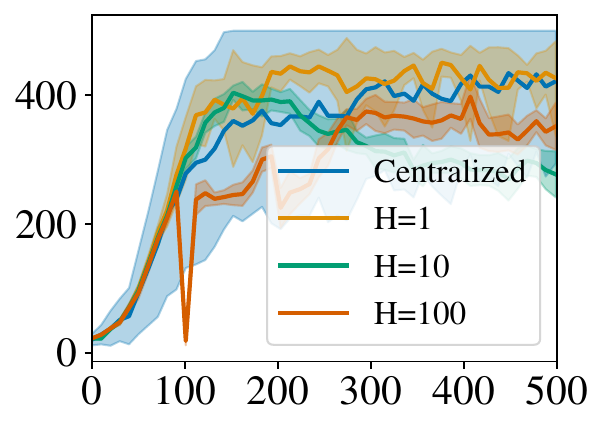}\\[-0.5em]

         {\quad \small  Episodes}
         
         \caption{$\nagent=100$}
         \label{fig:cartpole-N200}
     \end{subfigure}
    \caption{Rewards over the course of the learning for \SARSA and \FedSARSA on the CartPole environment for different numbers of agents $\nagent$ and numbers of local updates $\nlupdates$.
    For each plot, we report the average over $10$ runs and the corresponding standard deviation.}
    \label{fig:deep-results}
  \vspace{-0.5em}
\end{figure*}

We now study tabular \FedSARSA algorithm on synthetic problems, and propose an extension to deep RL.
First, we generate $2$ different instances of the Garnet environment \citep{archibald1995generation,geist2014off} with $|\S|=10$ states and $|\A| = 3$ actions, where we connect each state with $2$ other states with random transitions.
For tabular experiments, we choose $\nagent \in \{2, 10, 50, 100\}$ and equip half of the agents with the first environment, and the other half with the second one.
For deep experiments, we use $\nagent \in \{2, 10, 50, 100\}$ copies of CartPole, and use a deep network with two hidden layers to approximate the state-action value function.
All experiments are run on a single laptop with an RTX 2000 Ada Generation Laptop GPU, and the code is available %
online at \url{https://github.com/pmangold/fed-sarsa}.%

\paragraph{\FedSARSA has linear speed-up.}
In \Cref{fig:results}, we report the convergence of \FedSARSA for different number of agents, with the same shared problem.
In all experiments, we keep the same step size and the same number of local updates.
As predicted by our theory, increasing the number of agents reduces the variance, allowing to reach solutions with higher precision.
However, as the number of local steps gets larger, the bias of \FedSARSA due to heterogeneity increases, and eventually prevents the algorithm from converging to satisfying precision.
Remarkably, this phenomenon appears only when the number of local updates is very large, highlighting the relevance of \FedSARSA. %

\paragraph{\FedSARSA converges even with multiple local steps.}
In \Cref{fig:results}, we run the algorithm for $\nlupdates \in \{1, 100, 10000\}$.
For small values of $\nlupdates$, \FedSARSA has small heterogeneity bias, while for $\nlupdates = 10000$, \FedSARSA suffers from a large heterogeneity bias.
This is in line with our theory, which shows that \FedSARSA's bias increases with the number of local iterations, but remains small as long as $\nlupdates$ is small.
We stress that this is the case until $\nlupdates = 10000$ local updates, allowing for significant communication reduction.

\paragraph{\FedSARSA converges to $\paramlim$.}
We study the norm of two errors for the \FedSARSA: (i) the error $\smash{\bar{\param}_{t,h} - \paramlim}$, where $\smash{\bar{\param}_{t,h}}$ is the average of the local parameters $\smash{\locparam{c}{t,h}}$ for $c \in \intlist{1}{\nagent}$, in estimating $\smash{\paramlim}$, as defined in \Cref{prop:existence-theta-star-fed}, and (ii) the error $\smash{\globparam{t} - \chi_\star}$, where $\chi_\star$ is the point to which local \SARSA converges when run on the \emph{averaged environment}.
In \Cref{fig:results}, we plot the error as a function of the number of updates, for $\nlupdates \in \{1, 100, 10000\}$.
The error relative to $\paramlim$ is reported in solid lines, and quickly becomes small, while the error relative to $\smash{\chi_\star}$ remains large, demonstrating that \FedSARSA does not converge to this point.
This underlines the soundness of our analysis in showing convergence to the solution of the fixed-point equation defined in \Cref{prop:existence-theta-star-fed} rather than to the solution of a virtual environment.

\paragraph{Deep \FedSARSA.}
To evaluate the soundness of the learned policy in real environments, we introduce a deep variance of \FedSARSA for episodic environments. 
To this end, we propose a full-featured JAX \citep{jax2018github} implementation of \SARSA and \FedSARSA, supporting parallel environment executions using {gymnax} \citep{gymnax2022github}.
Following common deep RL practice, we stabilize the learning with a small replay buffer, as well as a target Q network, which remains fixed during episodes, and is updated to the current value of the Q network at the end of episodes.
We run experiments on the CartPole environment, which we report in \Cref{fig:deep-results}.
Our results show that \FedSARSA reaches large rewards faster and in a more stable way than \SARSA, provided the number of local steps is not too large.
When the number of local steps increases, the algorithm becomes biased due to noisy updates and heterogeneity, making rewards drop just after aggregation.

%% file: src/conclu.tex
We provide the first sample and communication complexity result for the \FedSARSA algorithm in heterogeneous environments, assuming that the policy improvement operator is Lipschitz.
Our results highlight that \FedSARSA converges with arbitrary precision even with multiple local steps, and that it has linear speed-up. %
To conduct this analysis, we develop a novel analytical framework for single-agent \SARSA, based on a novel, exact expansion of the algorithm's error over multiple updates, paired with a careful analysis of the impact of Markovian noise.
We then characterize the point to which \FedSARSA converges, highlighting that, contrary to common belief, the algorithm does not converge to the optimal parameter of a virtual averaged environment, but rather to the solution of an equation that depends on all environments.
Together with numerical validation and extensions to deep RL, our results demonstrate that federated \SARSA is both theoretically sound and practically applicable.
Expanding our theoretical analysis to more general policy improvement operators constitutes a promising next step towards bridging the gap between the theoretical foundations of \FedSARSA and its practical performance.
Finally, when the number of local iterations increases, \FedSARSA's iterations become increasingly biased due to noise and heterogeneity. This opens novel perspectives for debiasing \FedSARSA in communication-constrained settings.

%% file: src/sup-markov.tex
\subsection{Markov kernels and canonical Markov chains}
\label{subsec:canonical}
Throughout this section, $P$ denotes a Markov transition kernel on a Polish space $(\Zset,\Zsigma)$. We denote by $(\Zset^\nset, \Zsigma^{\otimes \nset})$ the set of $\Zset$-valued sequence endowed with the product $\sigma$-field and  by $(\globRandState{h})_{h \in \nset}$ the canonical (or coordinate) process (see Chapter~3 \cite{douc2018markov}). For any probability distribution $\varrho$ on $(\Zset,\Zsigma)$, there exists a unique distribution $\PP_\varrho$ on the canonical space such that the coordinate process $(\globRandState{h})_{h \in \nset}$ is a Markov chain with initial distribution $\varrho$ and Markov kernel $P$.  
We denote by $\PE_{\varrho}$ the corresponding expectation. For any bounded measurable function $f$ on $(\Zset,\Zsigma)$,
\[
\PE_{\varrho}[f(Z_0)] = \varrho(f), \quad \text{and} \quad
\CPE[\varrho]{f(Z_{h+1})}{\sigfield_h} = P f(Z_h), \quad \PP_\varrho\text{-a.s.},
\]
where $(\sigfield_h = \sigma(Z_0,\dots,Z_h))_{h \in \nset}$ is the \emph{canonical filtration} and for any bounded measurable function, $P f(z)= \int_{\Zset} P(z,\rmd z') f(z')$. For any two markov kernels $P$ and $Q$, we denote by $PQ$ the Markov kernel defined by $PQ f(z)= \int_{\Zset} P(z,\rmd z') Qf(z')$. Finally, we denote $k \in \nset$,  $P^k$ the $k$-th power of $P$.

\subsection{Exact coupling}
\label{subsec:coupling}
Let $(Z_h)_{h \in \nset}$ and $(Z'_h)_{h \in \nset}$ be (discrete time) processes. By a \emph{coupling} of $(Z_h)$ and $(Z'_h)$ we mean a simultaneous realizations of these processes on the same probability space $(\Omega,\sigfieldcoupling,\PPcoupling)$. We say that a \emph{coupling is successful} if the two processes agree $\PPcoupling( Z_h = Z'_h, \text{for all $h$ large enough})=1$.  
Let $\varrho$ and $\varrho'$ be two probability distributions on $(\Zset,\Zsigma)$. 
We say that $(\Omega,\sigfieldcoupling,\PPcoupling,(\bar{Z}_h)_{h \in \nset},(\bar{Z}'_h)_{h \in \nset},\taucoupling)$ is an \emph{exact coupling} of $(\PP_\varrho,\PP_{\varrho'})$ if:
\begin{enumerate}
  \item For all $A \in \Zsigma^{\otimes \nset}$, \quad $\PPcoupling((\bar{Z}_h)_{h \in \nset} \in A) = \PP_\varrho(A)$ \quad and \quad $\PPcoupling((\bar{Z'}_h)_{h \in \nset} \in A) = \PP_{\varrho'}(A)$.
  \item $\PP(\taucoupling < \infty) = 1$.
  \item For all $h \in \nset$, \quad $\bar{Z}_{h + \taucoupling} = \bar{Z}'_{h + \taucoupling}$.
\end{enumerate}
Note that we allow ourselves a slight abuse of notation here, as the distribution $\PPcoupling$ depends on both $\varrho$ and $\varrho'$. To avoid cumbersome notation, this dependence is kept implicit.

In particular, at all times $h$ subsequent to the coupling time $\taucoupling$, the two processes coincide $Z_h = Z'_h$.
Theorem 19.3.9 of \citet{douc2018markov} shows that:
\begin{lemma}
\label{lem:coupling-construction}
There exists a maximal exact coupling, \ie\ an exact coupling $(\Omega,\sigfieldcoupling,\PP,(\bar{Z}_h)_{h \in \nset},(\bar{Z}'_h)_{h \in \nset},\taucoupling)$, such that 
\begin{align*}
\proba{\taucoupling > h}
=  (1/2) \norm{ \varrho P^h - \varrho' P^h }[\TV]
\eqsp.
\end{align*}
\end{lemma}

\begin{corollary}
\label{cor:coupling-time-expect}
Assume that $P$ is uniformly geometrically ergodic with mixing time $\taumix$, \ie\
\[
(1/2) \norm{\varrho P^h - \varrho' P^h}[\TV] \leq (1/4)^{\lfloor h/ \taumix \rfloor}. 
\] 
Let $(\Omega,\sigfieldcoupling,\PP,(\bar{Z}_h)_{h \in \nset},(\bar{Z}'_h)_{h \in \nset},\taucoupling)$ be an exact coupling of $\PP_\varrho$ and $\PP_{\varrho'}$. Then,
\begin{align*}
\PE[ \taucoupling ] \le \frac{4 \taumix}{3}
\eqsp, \qquad
\text{ and }
\quad
\PE[ \taucoupling^2 ] \le \frac{20 \taumix^2}{9}
\eqsp.
\end{align*}
\end{corollary}
\begin{proof}
Note first that
\begin{align*}
\PE[ \taucoupling ] 
& =
\sum_{h = 0}^\infty \proba{\taucoupling > h}
=
(1/2) \sum_{h = 0}^\infty \norm{ \varrho P^h - \statdist{} }[\TV] \leq  \sum_{h=0}^\infty (1/4)^{\lfloor h / \taumix \rfloor} = (4/3) \taumix.
\end{align*}
The second inequality follows from 
\begin{align*}
\PE[ \taucoupling^2 ] 
& =
\sum_{h = 0}^\infty h^2 \proba{\taucoupling = h}
=
\sum_{h = 0}^\infty h^2 \big(\proba{\taucoupling > h-1} - \proba{\taucoupling > h}\big)
\eqsp,
\end{align*}
which gives, after reorganizing the terms,
\begin{align*}
\PE[ \taucoupling^2 ] 
& =
\sum_{h = 0}^\infty \big( (h+1)^2 - h^2 \big) \proba{\taucoupling > h}
=
\sum_{h = 0}^\infty (2h+1) \proba{\taucoupling > h}
\eqsp.
\end{align*}
Using $\proba{\taucoupling > h} = \norm{\varrho P^h - \statdist{} }[\TV] \le \boundmixing{h}$ and grouping the terms by blocks gives
\begin{align*}
\PE[ \taucoupling^2 ] 
& \le
\sum_{h = 0}^\infty (2h+1) \boundmixing{h}
=
\taumix^2 \sum_{k = 0}^\infty (2k+1) (1/4)^k 
= \taumix^2 \Big( \frac{8}{9} + \frac{4}{3} \Big)
\eqsp,
\end{align*}
and the second inequality of the corollary follows.
\end{proof}

\begin{lemma}
\label{lem:relation-markov-stationary}
Assume that $P$ is uniformly geometrically ergodic with unique invariant distribution $\statdist{}$, \ie\ $\statdist{} P = \statdist{}$ and mixing time $\taumix$.
Let $f_h : \Zset \rightarrow \rset^d$, for $h \ge 0$, be functions such that $\norm{ f }[2,\infty] = \sup_{h \ge 0} \sup_{z \in \Zset} \abs{ f_h(z) } < +\infty$.
Then, it holds that
\begin{align*}
\textstyle
\bnorm{ 
\PE_{\varrho}\Big[ \sum_{h=0}^{\nlupdates-1} f_h(\globRandState{h}) \Big]
}
& \le 
\textstyle
\bnorm{ 
\PE_{\statdist{}}\Big[ \sum_{h=0}^{\nlupdates-1} f_h(\globRandState{h}) \Big]
}
+ \frac{8 \norm{f}[\infty] \taumix}{3}
\eqsp,
\\
\textstyle
\PE_{\varrho}\Big[ \bnorm{ \sum_{h=0}^{\nlupdates-1} f_h(\globRandState{h}) }^2 \Big]
& \le 
\textstyle
2 \PE_{\statdist{}}\Big[ \bnorm{  \sum_{h=0}^{\nlupdates-1} f_h(\globRandState{h}) }^2 \Big]
+ \frac{160 \norm{f}[\infty]^2 \taumix^2}{9}
\eqsp.
\end{align*}
\end{lemma}
\begin{proof}
Let $(\Omega,\sigfieldcoupling,(\bglobRandState{h})_{h \in \nset},(\bglobRandStateOpt{h})_{h \in \nset},\taucoupling)$ an exact maximal coupling of $\PP_\varrho$ and $\PP_{\statdist{}}$.
Denote $S^Z_\nlupdates = \sum_{h=0}^{\nlupdates-1} f_h(\bglobRandState{h})$ and $S^Y_\nlupdates = \sum_{h=0}^{\nlupdates-1} f_h(\bglobRandStateOpt{h})$, and using the definition of $\taucoupling$, we have
\begin{align}
\nonumber
\textstyle
\PE_{\varrho}\Bigl[ \sum_{h=0}^{\nlupdates-1} f(\globRandState{h}) \Bigr]
&= 
\PEcoupling\bigl[S^Y_{\nlupdates}\bigr] + \PEcoupling\bigl[S^Z_{\nlupdates}\bigr] - \PEcoupling\bigl[S^Y_{\nlupdates}\bigr]\\
\textstyle
&= 
\PE_{\statdist{}}\Bigl[ \sum_{h=0}^{\nlupdates-1} f(\globRandState{h}) \Bigr]
+ \PEcoupling\Bigl[ \sum_{h=0}^{\taucoupling}\bigl(f(\bglobRandState{h}) - f(\bglobRandStateOpt{h})\bigr) \Bigr], \label{eq:decomp-f-Sn}
\end{align}
since for $h \ge \taucoupling$, $f(\bglobRandState{h}) - f(\bglobRandStateOpt{h}) = 0$.
The first inequality follows by taking the norm of \eqref{eq:decomp-f-Sn}, using the triangle inequality, the definition of $\norm{ f }[2,\infty]$, and \Cref{cor:coupling-time-expect}.
The second inequality follows from Young's inequality, which gives
\begin{align*}
\textstyle
\PE_\varrho\Big[ \bnorm{\! \sum_{h=0}^{\nlupdates-1} \!f(\globRandState{h}) }^2 \Big]
\!& \le
\textstyle
2 \PE_{\statdist{}}\Big[ \bnorm{\! \sum_{h=0}^{\nlupdates-1} \!f(\globRandState{h}) }^2 \Big]
\!+\! 
2 \PEcoupling \Big[ \bnorm{ \!\sum_{h=0}^{\taucoupling}\! f(\bglobRandState{h}) - f(\bglobRandStateOpt{h}) }^2 \Big]
\eqsp.
\end{align*}
We then use Jensen's inequality, and the definition of $\norm{ f }[2,\infty]$, to bound the second term as 
\begin{align*}
2 \PE_{\varrho}\Big[ \bnorm{ \!\sum_{h=0}^{\taucoupling}\! f(\globRandState{h}) \!-\! f(\globRandStateOpt{h}) }^2 \Big]
& \le
2 \PE_{\varrho}\Big[ \taucoupling \sum_{h=0}^{\taucoupling}\!\norm{  f(\globRandState{h}) \!-\! f(\globRandStateOpt{h}) }^2 \Big]
\\
& \le
2 \PE_{\varrho}\Big[ 4 \taucoupling^2 \norm{f}[\infty]^2 \Big]
\eqsp,
\end{align*}
and the result follows from \Cref{cor:coupling-time-expect}.
\end{proof}

\subsection{Berbee's Lamma}
\label{subsec:berbee-lemma}
We conclude this section  by giving a simplified statement of the Berbee lemma. Consider the extended measurable space $\tilde{\Zset}_{\nset}=\Zset^{\nset} \times[0,1]$, equipped with the $\sigma$-field $\tilde{\Zset}_{\nset}=\Zset^{\otimes \nset} \otimes \mathcal{B}([0,1])$. For each probability measure $\varrho$ on $(\Zset,\Zsigma)$, we consider the probability measure $\tilde{\PP}_{\varrho}=\PP_{\varrho} \otimes \operatorname{Unif}([0,1])$ and denote by $\tilde{\PE}_{\varrho}$ the corresponding expected value. Finally, we denote by $(\tilde{Z}_k)_{k \in \nset}$ the canonical process defined, for each $k \in \nset$, by $\tilde{Z}_k:((z_i)_{i \in \nset}, u) \in \Zset_{\nset} \mapsto z_k$ and $U:((z_i)_{i \in \nset}, u) \in \Zset_{\nset} \mapsto u$. Under $\tilde{\PP}_{\varrho}$, the process $\{\tilde{Z}_k\}_{k \in \nset}$ is by construction a Markov chain with initial distribution $\varrho$ and Markov kernel $P$, and is independent of $U$. The distribution of $U$ under $\tilde{\PP}_{\varrho}$ is uniform over $[0,1]$.

We first recall \cite{rio2017asymptotic}, Chapter~5. Let  $\mathcal{A}$ and  $\mathcal{B}$ two  $\sigma$-fields of $(\Omega, \mathcal{T}, \mathbb{P})$. The $\beta$-mixing of  $(\mathcal{A}, \mathcal{B})$ is defined by:
$$
\beta(\mathcal{A}, \mathcal{B})=\frac{1}{2} \sup \left\{\sum_{i \in I} \sum_{j \in J}\left|\mathbb{P}\left(A_i \cap B_j\right)-\mathbb{P}\left(A_i\right) \mathbb{P}\left(B_j\right)\right|\right\}
$$
the maximum being taken over all finite partitions $\left(A_i\right)_{i \in I}$ and $\left(B_j\right)_{j \in J}$ of $\Omega$ with the sets $A_i$ in $\mathcal{A}$ and the sets $B_j$ in $\mathcal{B}$.
\begin{lemma} 
\label{lem:rio-chapter-5}
Let $\mathcal{A}$ be a $\sigma$-field in  $(\Omega, \mathcal{T}, \mathbb{P})$ and $X$ be a random variable with values in some Polish space. Let $\delta$ be a random variable with uniform distribution over $[0,1]$, independent of the $\sigma$-field generated by $X$ and $\mathcal{A}$. Then there exists a random variable $X^*$, with the same law as $X$, independent of $X$, such that $\mathbb{P}\left(X \neq X^*\right)=\beta(\mathcal{A}, \sigma(X))$. Furthermore $X^*$ is measurable with respect to the $\sigma$-field generated by $\mathcal{A}$ and  $(X, \delta)$.
\end{lemma}

\begin{lemma}
\label{lem:coupling-indep}
Assume that $P$ is uniformly geometrically ergodic with mixing time $\taumix$. Set $m \in \nset$.
Then,  there exists a random process $(\tilde{Z}_k^{\star})_{k \in \nset}$ defined on $(\tilde{\Zset}_{\nset}, \tilde{\Zsigma}_{\nset}, \tilde{\PP}_{\varrho})$ such that for any $k \geq m$,
\begin{enumerate}
\item $\tilde{Z}_{k}^{\star}$ is independent of $\tilde{\mathcal{F}}_{k-m}=\sigma\{\tilde{Z}_{\ell}: \ell \leq k-m\}$;
\item $\tilde{\PP}_{\varrho}(\tilde{Z}_k^{\star} \neq \tilde{Z}_k) \leq \boundmixing{m}$;
\item the random variables $\tilde{Z}_k^{\star}$ and $\tilde{Z}_k$ have the same distribution under $\tilde{\PP}_{\varrho}$.
\end{enumerate}
\end{lemma}
\begin{proof}
We apply for each $k \in \nset$ \Cref{lem:rio-chapter-5} with $\Omega= \tilde{\Zset}_{\nset}$, $\PP= \tilde{\PP}_\varrho$,
$\mathcal{A}= \sigma\{\tilde{Z}_{\ell}: \ell \geq k-m\}$ and $X = \tilde{Z}_k$. We conclude by using the bound for $\beta$-mixing coefficient given in \cite{douc2018markov}, Theorem~3.3.
\end{proof}

\begin{lemma}
\label{lem:coupling-indep-1}
Assume that $P$ is uniformly geometrically ergodic with mixing time $\taumix$.
Let $0 \leq m \leq k \in \nset$ and $f: \Zset \to \rset$ and $g: \Zset^{k-m} \to \rset$ be two bounded measurable functions. 
Then, for any initial distribution $\varrho$, 
\begin{align*}
|\PE_\varrho[ f(Z_{k}) g(Z_{0:{k-m}}) ] | 
& \leq 
\norm{f}[\infty] | \PE_\varrho[ g(Z_{0:k-m}) ] | 
+ 2 \norm{f}[\infty] \norm{g}[\infty] \boundmixing{m}, 
\end{align*}
where, for any sequence $(u_\ell)_{\ell \in \nset}$ and $0 \leq k \leq \ell$, we set $u_{k:\ell}= [u_k,\dots,u_\ell]$.
\end{lemma}
\begin{proof}
Using \Cref{lem:coupling-indep}, we get that 
\begin{align*}
A: = \PE_\varrho[ f(Z_{k}) g(Z_{0:k-m})]
&= \tilde{\PE}_\varrho[ f(\tilde{Z}_{k}) g(\tilde{Z}_{0:k-m})] \\
&= \tilde{\PE}_\varrho[ f(\tilde{Z}_{k}^\star) g(\tilde{Z}_{0:k-m})] + \tilde{\PE}_{\varrho}[\{f(\tilde{Z}_{k}) -  f(\tilde{Z}_{k}^\star)\} g(\tilde{Z}_{0:k-m})]
\end{align*}
where by construction $\tilde{Z}_{k}^\star$ is independent of $\sigma\{\tilde{Z}_{\ell}: \ell \geq k-m\}$ under $\tilde{\PP}_\varrho$. Hence: 
\[
\tilde{\PE}_\varrho[ f(\tilde{Z}_{k}^\star) g(\tilde{Z}_{0:k-m})] = \PE_\varrho[f(Z_k)] \PE_\varrho[g(Z_{0:k-m})],
\]
where we have used that, under $\tilde{\PP}_\varrho$, the law of $\tilde{Z}_{k}^\star$ and $\tilde{Z}_k$ coincide and 
$\tilde{\PE}_{\varrho}[g(\tilde{Z}_{0:k-m})]= \PE_\varrho[g(Z_{0:k-m})]$. Finally,
\begin{align*}
|\tilde{\PE}_{\varrho}[\{f(\tilde{Z}_{k}) -  f(\tilde{Z}_{k}^\star)\} g(\tilde{Z}_{0:k-m})]| 
&= |\tilde{\PE}_{\varrho}[\{f(\tilde{Z}_{k}) -  f(\tilde{Z}_{k}^\star)\} g(\tilde{Z}_{0:k-m}) \mathbbm{1}_{\{\tilde{Z}_k \ne \tilde{Z}_k^*\}}] |\\
&\leq 2 \norm{f}[\infty] \norm{g}[\infty] \tilde{\PP}_{\varrho}(\tilde{Z}_k \ne \tilde{Z}_k^\star) \eqsp.
\end{align*} 
The result follows.
\end{proof}

\begin{lemma}
\label{lem:coupling-indep-2}
Assume that $P$ is uniformly geometrically ergodic with mixing time $\taumix$.
Let $0\leq h \leq k \leq \ell \leq m$ and let $f_h, f_k, f_\ell, f_m : \msZ \rightarrow \rset$ be bounded measurable functions.  
Then, for any initial distribution $\varrho$, it holds that 
\begin{equation}
\label{eq:product-3-terms}
\begin{aligned}
& \Big |\PE_{\varrho}\big[(f_h(Z_h) - \varrho P^h f_h) (f_k(Z_k) - \varrho P^k f_k) (f_\ell(Z_\ell) - \varrho P^\ell f_\ell) \big] \Big| 
\\
& \qquad\qquad\qquad 
\leq 8 \norm{f_h}[\infty]\norm{f_k}[\infty]\norm{f_\ell}[\infty] 
\boundmixing{(\ell-k)} \boundmixing{(k-h)}
\eqsp.
\end{aligned}
\end{equation}
and
\begin{equation}
\label{eq:product-4-terms}
\begin{aligned}
& \Big |\PE_{\varrho}\big[
(f_h(Z_h) - \varrho P^h f_h)
(f_k(Z_k) - \varrho P^k f_k)
(f_\ell(Z_\ell) - \varrho P^\ell f_\ell) 
(f_m(Z_m) - \varrho P^m f_m) 
\big] \Big| 
\\
& \! \leq \! 
16 
\norm{ {f}_\ell}[\infty]
\norm{ {f}_m }[\infty]
\norm{ {f}_h }[\infty]
\norm{ {f}_k }[\infty]
\boundmixing{(m - \ell)} (\boundmixing{(\ell  - k)} \!\!+\! \boundmixing{(k  -  h)}) 
~.
\end{aligned}
\end{equation}
\end{lemma}
\begin{proof}
In this proof, we denote $\bar{f}_i = f_i - \varrho P^i f_i$, $i \in \{h,k,\ell,m\}$.
We first prove the first part \eqref{eq:product-3-terms} of the lemma. 
Note that
\begin{align*}
\PE_{\varrho}[ \bar{f}_h(Z_h) \bar{f}_k(Z_k) \bar{f}_\ell(Z_\ell)] 
= \PE_{\varrho}[\bar{f}_h(Z_h) \bar{f}_k(Z_k) P^{\ell-k} \bar{f}_\ell(Z_k)] 
\eqsp.
\end{align*}
Using the Chapman-Kolmogorov equation for Markov chain, we have that, for all $z \in \Zset$,
\begin{equation}
\label{eq:chapman-kolmogorov}
P^{\ell-k} \bar{f}_\ell(z)= \delta_z P^{\ell-k} f_\ell - \varrho P^k P^{\ell-k} f_\ell 
\eqsp.
\end{equation}
This implies that
\begin{equation}
\label{eq:bound-P-ell}
\begin{aligned}
\norm{P^{\ell-k} \bar{f}_\ell}[\infty] 
& \leq 2 \boundmixing{(\ell-k)} \norm{f_\ell}[\infty]
\eqsp,
\\
\norm{ \bar{f}_k P^{\ell-k} \bar{f}_\ell }[\infty] 
& \leq 4 \boundmixing{(\ell-k)} \norm{f_k}[\infty]  \norm{f_\ell}[\infty] 
\eqsp.
\end{aligned}
\end{equation}
Noticing that $\PE_\rho[ \bar{f_h}(Z_h) ] = 0$, we use \Cref{lem:coupling-indep-1} to bound
\begin{align*}
\abs{ \PE_{\varrho}[ \bar{f}_h(Z_h) \bar{f}_k(Z_k) \bar{f}_\ell(Z_\ell)] }
& \le 
\norm{ \bar{f}_h }[\infty]
\norm{ \bar{f}_k P^{\ell-k} \bar{f}_\ell }[\infty]
\boundmixing{k-h}
\\
& \le 
4 \norm{ \bar{f}_h }[\infty]
\norm{ {f}_k }[\infty] 
\norm{ {f}_\ell }[\infty]
\boundmixing{(\ell-k)}
\boundmixing{(k-h)}
\eqsp,
\end{align*}
and \eqref{eq:product-3-terms} follows.
To prove the second inequality \eqref{eq:product-4-terms}, we first use the Markov property 
\begin{align*}
\PE_{\varrho}[ \bar{f}_h(Z_h) \bar{f}_k(Z_k) \bar{f}_\ell(Z_\ell) \bar{f}_m(Z_m)]
=
\PE_{\varrho}[ \bar{f}_h(Z_h) \bar{f}_k(Z_k) \bar{f}_\ell(Z_\ell) P^{m-\ell} \bar{f}_m(Z_\ell) ]
\eqsp.
\end{align*}
By \Cref{lem:coupling-indep-1}, we obtain
\begin{align*}
& \abs{ \PE_{\varrho}[ \bar{f}_h(Z_h) \bar{f}_k(Z_k) \bar{f}_\ell(Z_\ell) \bar{f}_m(Z_m)] }
\\
& \quad \le
\norm{ \bar{f}_\ell(Z_\ell) P^{m-\ell} \bar{f}_m(Z_\ell) }[\infty]
\Big\{ 
\abs{ \PE_{\varrho}[ \bar{f}_h(Z_h) \bar{f}_k(Z_k) ] }
+
\norm{ \bar{f}_h }[\infty]
\norm{ \bar{f}_k }[\infty]
\boundmixing{(\ell - k)}
\Big\}
\\
& \quad \le
4 
\norm{ {f}_\ell}[\infty]
\norm{ {f}_m }[\infty]
\norm{ \bar{f}_h }[\infty]
\norm{ \bar{f}_k }[\infty]
\boundmixing{(m - \ell)}
\Big\{ 
\boundmixing{(k - h)}
+
\boundmixing{(\ell - k)}
\Big\}
\eqsp,
\end{align*}
where we also used the Markov property and proceeded as in the derivation of \eqref{eq:bound-P-ell}.
\end{proof}

\subsection{Bounds on covariances}

In this part, we will make extensive use of the following norms, for a sequence of matrices $F_h : \msZ \rightarrow \rset^{d \times d}$ and vectors $g_h : \msZ^h \rightarrow \rset^{d}$
\begin{align*}
  \triplenorm{g}[2,\infty] 
  \eqdef
  \sup_{h \ge 0} \norm{ g_h }[2, \infty]
  \eqsp,
  \quad \text{ with } \quad
  \norm{g_h}[2,\infty] 
  & \eqdef
   \Big( \sum_{i=1}^d \sup_{z_{1:h} \in \Zset^h} \abs{ g_{h,i}(z_{1:h}) }^2 \Big)^{1/2}
  ~,~~
  \\
  \triplenorm{ F }[2,\infty]
  \eqdef \sup_{h \ge 0}
  \triplenorm{ F_h }[2,\infty]
  \eqsp,
  \quad
  \text{ with } \quad
  \triplenorm{ F_h }[2,\infty]
  & \eqdef
  \Big( \sum_{i=1}^d \norm{ F_{h,i,:} }[2,\infty]^2 \Big)^{1/2} 
  \eqsp.
\end{align*}

\begin{lemma}
\label{lem:bound-variance-sum-markov}
Assume that $P$ is uniformly geometrically ergodic with mixing time $\taumix$. Let $f_h : \Zset \rightarrow \rset^d$ be uniformly bounded for $h \ge 0$. %
Then, for any initial distribution $\varrho$, we get 
\begin{align}
\label{eq:bound-variance-sum-markov:non-stationary}
\PE_{\varrho} \Big[ \bnorm{ \sum_{h=0}^{\nlupdates-1} \{ f_h(\globRandState{h}) - \varrho P^h f_h \}}^2 \Big]
\le 15 \nlupdates \taumix \norm{ f }[2,\infty]^2 
\eqsp.
\end{align}
\end{lemma}
\begin{proof}
Expanding the square, we have
\begin{align}
  \nonumber
  & \PE_{\varrho} \Big[ \bnorm{ \sum_{h=0}^{\nlupdates-1} \{ f_h(\globRandState{h})  - \varrho P^h f_h\} }^2 \Big]
  \\
  \label{eq:expand-correlation-f-markov}
  &
  \!=\!
  \underbrace{\sum_{h=0}^{\nlupdates-1} \PE_{\varrho} [ \norm{ f_h(\globRandState{h}) - \varrho P^h f_h }^2 ] }_{\intermterm{A}{1}}
  + \underbrace{2 \!\!\!\!\!\sum_{0 \leq h < h' \leq H}\!\!\!\!\!
  \PE_{\varrho} [ \pscal{ f_h(\globRandState{h}) - \varrho P^h f_h }{ f_{h'}(\globRandState{h'}) - \varrho P^{h'} f_{h'} } ]}_{\intermterm{A}{2}}
  \eqsp.
\end{align}
We bound the first term by $\intermterm{A}{1} \leq 4 H \norm{f}[\infty]^2$. 
For the second term $\intermterm{A}{2}$, we proceed coordinate by coordinate, using the triangle inequality and the Markov property, which yields
\begin{align*}
  &
    \Big| \PE_{\varrho} \Big[ \bpscal{ f_h(\globRandState{h}) - \varrho P^h f_h }{ f_{h'}(\globRandState{h'}) - \varrho P^{h'} f_{h'}} \Big] \Big|
  \\
  & \quad \le
    \sum_{i=1}^d
    \Big| \PE_{\varrho} \Big[ e_i^\top \big( f_h(\globRandState{h}) - \varrho P^h f_h  \big) \cdot e_i^\top \big(  f_{h'}(\globRandState{h'}) - \varrho P^{h'} f_{h'} \big) \Big] \Big|
  \\
  & \quad  =
    \sum_{i=1}^d
    \Big| \PE_{\varrho} \Big[ \big( f_{h,i}(\globRandState{h}) - \varrho P^h f_{h,i} \big) \cdot \big(  \delta_{Z_h} P^{h'-h} e_i^\top f_{h'} -\varrho P^h P^{h'-h}  f_{h',i} \big)  \Big] \Big|
    \eqsp.
\end{align*}
By ergodicity of the Markov chain, we obtain
\begin{align*}
  \Big| \PE_{\varrho}\! \Big[\! \bpscal{ f_h(\globRandState{h}) \!-\! \varrho P^h f_h }{ f_{h'}(\globRandState{h'}) \!-\! \varrho P^{h'} f_{h'}} \!\Big] \Big|
  & \le
    \!\sum_{i=1}^d
    2 \norm{ f_{h,i} }[\infty]\! \cdot 2 \norm{ f_{h',i}}[\infty] \boundmixing{(h'-h)}
    \eqsp,
\end{align*}
where, $\norm{ f_{k,i} }[\infty] = \sup_{z \in \Zset} \abs{ f_{k,i}(z) }$ for $k \in \{h, h'\}$.
Using the definition of $\norm{ f }[2,\infty]$ then gives %
\begin{align*}
  \abs{\intermterm{A}{2}} \leq
  8 \norm{f}[2,\infty]^2 \sum_{0 \leq h < h'\leq H} 
  \boundmixing{(h'-h)}
  \leq 
  \frac{32 \norm{f}[2,\infty]^2 \taumix H}{3} 
  \eqsp,
\end{align*}
and the result follows by summing the bounds on $\intermterm{A}{1}$ and $\intermterm{A}{2}$.
\end{proof}

\begin{corollary}
\label{lem:bound-variance-sum-markov-stat}
Under the assumptions of \Cref{lem:bound-variance-sum-markov}, it holds that,
\begin{align}
\label{eq:bound-variance-sum-markov:stationary}
\PE_{\varrho} \Big[ \bnorm{ \sum_{h=0}^{\nlupdates-1} \{ f_h(\globRandState{h}) - \statdist{} P^h f_h \}}^2 \Big]
\le 
34 \nlupdates \taumix \norm{ f }[2,\infty]^2 
\eqsp.
\end{align}
\end{corollary}
\begin{proof}
We can decompose the error as
\begin{align*}
\PE_{\varrho} \Big[ \bnorm{ \sum_{h=0}^{\nlupdates-1} \{ f_h(\globRandState{h}) - \statdist{} P^h f_h \}}^2 \Big]
& \le
2 \PE_{\varrho} \Big[ \bnorm{ \sum_{h=0}^{\nlupdates-1} \{ f_h(\globRandState{h}) - \statdist{} P^h f_h \}}^2 \Big]
+
2 \PE_{\varrho} \Big[ \bnorm{ \sum_{h=0}^{\nlupdates-1} \{ \statdist{} P^h f_h - \varrho{} P^h f_h \}}^2 \Big]
\eqsp.
\end{align*}
The first term can be bounded using \Cref{lem:bound-variance-sum-markov}, while the second term can be bounded as
\begin{align*}
2 \PE_{\varrho} \Big[ \bnorm{ \sum_{h=0}^{\nlupdates-1} \{ \statdist{} P^h f_h - \varrho{} P^h f_h \}}^2 \Big]
& \le
2 \nlupdates  \sum_{h=0}^{\nlupdates-1} \PE_{\varrho} \Big[ \bnorm{ \{ \statdist{} P^h f_h - \varrho{} P^h f_h \}}^2 \Big]
\le 
\frac{8}{3} \nlupdates \taumix \norm{ f }[2,\infty]^2
\eqsp,
\end{align*}
where the second bound follows from the mixing property of the Markov chain.
\end{proof}

\begin{lemma}
\label{lem:bound-product-variance-interm-sum-markov}
Assume that $P$ is uniformly geometrically ergodic with unique invariant distribution $\statdist{}$, \ie\ $\statdist{} P = \statdist{}$ and mixing time $\taumix$. Let $F_h : \Zset \rightarrow \rset^{d\times d}$ and $g_h : \Zset^h \rightarrow \rset^d$, for $h \ge 0$, be uniformly bounded. %
Then, it holds that
\begin{align*}
\PE_{\varrho} \Big[ \bnorm{ \sum_{h=1}^{\nlupdates} \sum_{\ell=1}^{h-1} \big\{ F_{h}(\globRandState{h}) - \varrho P^h F_h \big\} g_{\ell}(\globRandState{1:\ell}) }^2 \Big]
& \le 
\frac{70}{3}  \nlupdates^3 \taumix 
\triplenorm{ F }[2,\infty]^2 
\sup_{h \in \intlist{1}{\nlupdates} } \PE[ \norm{g_{h}(\globRandState{1:h}) }^2 ]
  \eqsp,
  \\
\PE_{\varrho} \Big[ \bnorm{ \sum_{h=1}^{\nlupdates} \sum_{\ell=1}^{h-1} \big\{ F_{h}(\globRandState{h}) - \varrho P^h F_h \big\} g_{\ell}(\globRandState{1:\ell})  }^2 \Big]
& \le
\frac{70}{3} \nlupdates^3 \taumix \triplenorm{ F }[2,\infty]^2 \triplenorm{ g }[2,\infty]^2 
\eqsp.
\end{align*}
\end{lemma}
\begin{proof}
We define $\bar{F}_h(z) = F_h(z) - \varrho P^h F_h$.
Expanding the square, we have
\begin{align*}
& \PE_{\varrho} \Big[ \bnorm{ \sum_{h=1}^{\nlupdates}\sum_{\ell=1}^{h-1} \bar{F}_{h}(\globRandState{h}) g_{\ell}(\globRandState{1:\ell}) }^2 \Big]
\\
& =
\sum_{h=1}^{\nlupdates} \PE_{\varrho} \Big[ \bnorm{ \sum_{\ell=1}^{h-1}  \bar{F}_{h}(\globRandState{h}) g_{\ell}(\globRandState{1:\ell}) }^2 \Big]
+ 2 \!\!\!\sum_{h' > h = 1}^{\nlupdates} \sum_{\ell = 1}^{h-1} \sum_{\ell' = 1}^{h'-1} \PE_{\varrho}\Big[ \bpscal{ \bar{F}_{h}(\globRandState{h}) g_{\ell}(\globRandState{1:\ell}) }{ \bar{F}_{h'}(\globRandState{h'}) g_{\ell'}(\globRandState{1:\ell'}) } \Big]
\\
& =
\sum_{h=1}^{\nlupdates} \PE_{\varrho} \Big[ \bnorm{ \sum_{\ell=1}^{h-1}  \bar{F}_{h}(\globRandState{h}) g_{\ell}(\globRandState{1:\ell}) }^2 \Big]
\\
& \quad + 2 \!\!\!\!\sum_{h' > h = 1}^{\nlupdates} \!\sum_{\ell = 1}^{h-1} \sum_{\ell' = 1}^{h'-1} \underbrace{\PE_{\varrho}\Big[ 
\bpscal{ \bar{F}_{h}(\globRandState{h}) g_{\ell}(\globRandState{1:\ell}) }
{\! 
P^{h' - \max(h,\ell')} \bar{F}_{h'}(\globRandState{\max(h,\ell')\!})
g_{\ell'\!}(\globRandState{1:\ell'\!}) }
\Big]}_{\intermterm{A}{\ell,\ell'}[h,h']}
\eqsp.
\end{align*}
To bound the first sum, we remark that, for $h \in \intlist{1}{\nlupdates}$,
\begin{align*}
\PE_{\varrho} \Big[ \bnorm{ \sum_{\ell=1}^{h-1}  \bar{F}_{h}(\globRandState{h}) g_{\ell}(\globRandState{1:\ell}) }^2 \Big]
& \le 
h \sum_{\ell=1}^{h-1} \PE_{\varrho} \Big[ \bnorm{  \bar{F}_{h}(\globRandState{h}) g_{\ell}(\globRandState{1:\ell}) }^2 \Big]
\le 
4 h \sum_{\ell=1}^{h-1} 
\triplenorm{F}[2,\infty]^2 
\PE_{\varrho} \big[ \norm{ g_{\ell}(\globRandState{1:\ell}) }^2 \big]
\eqsp,
\end{align*}
where the second inequality follows by bounding $F_h$'s operator norm by $\triplenorm{F}[2,\infty]^2 $.
We thus obtain
\begin{align}
\label{eq:bound-centered-non-cented-prod-first-sum}
\sum_{h=1}^{\nlupdates} \PE_{\varrho} \Big[ \bnorm{ \sum_{\ell=1}^{h-1}  \bar{F}_{h}(\globRandState{h}) g_{\ell}(\globRandState{1:\ell}) }^2 \Big]
& \le 
2 {\nlupdates^3} \triplenorm{F}[2,\infty]^2 
\sup_{1 \le \ell \le \nlupdates}
\PE_{\varrho} \big[ \norm{ g_{\ell}(\globRandState{1:\ell}) }^2 \big]
\eqsp.
\end{align}
To bound the second term, we remark that
\begin{align*}
\abs{ \intermterm{A}{\ell, \ell'}[h,h'] }
& \le
\sum_{i=1}^d 
\Big| 
\PE_{\varrho}\Big[ 
e_i^\top \cdot ( \bar{F}_{h}(\globRandState{h}) g_{\ell}(\globRandState{1:\ell}) )
\times 
e_i^\top \cdot ( 
P^{h' - \max(h,\ell')} \bar{F}_{h'}(\globRandState{\max(h,\ell')\!})
g_{\ell'\!}(\globRandState{1:\ell'\!}) )
\Big]
\Big|
\\
& =
\sum_{i=1}^d 
\Big| 
  \PE_{\varrho}\Big[ 
  \sum_{j=1}^d \bar{F}_{h,i,j}(\globRandState{h}) g_{\ell,j}(\globRandState{1:\ell}) 
  \sum_{j=1}^d  
  P^{h' - \max(h,\ell')} \bar{F}_{h',i,j}(\globRandState{\max(h,\ell')\!})
g_{\ell',j\!}(\globRandState{1:\ell'\!}) )
\Big]
  \Big|
  \eqsp.
\end{align*}
By Jensen's then Hölder inequality, we have
\begin{align*}
\abs{ \intermterm{A}{\ell, \ell'}[h,h'] }
& \le
\sum_{i=1}^d 
  \PE^{1/2}_{\varrho}\Big[ \Big| 
  \sum_{j=1}^d \bar{F}_{h,i,j}(\globRandState{h}) g_{\ell,j}(\globRandState{1:\ell})
  \Big|^2 \Big]
  \PE^{1/2}_{\varrho}\Big[ \Big|   
  \sum_{j=1}^d  
  P^{h' - \max(h,\ell')} \bar{F}_{h',i,j}(\globRandState{\max(h,\ell')\!})
g_{\ell',j\!}(\globRandState{1:\ell'\!}) )
  \Big|^2
  \Big]
  \\
& \le
\sum_{i=1}^d 
  \PE^{1/2}_{\varrho}\Big[
  \norm{ \bar{F}_{h,i,:}(\globRandState{h}) }^2
  \norm{ g_{\ell,:}(\globRandState{1:\ell}) }^2
  \Big]
  \PE^{1/2}_{\varrho}\Big[
  \norm{ P^{h' - \max(h,\ell')} \bar{F}_{h',i,:}(\globRandState{\max(h,\ell')\!}) }^2
  \norm{ g_{\ell',:\!}(\globRandState{1:\ell'\!}) }^2
\Big]
  \\
& \le
  \sum_{i=1}^d
  \norm{ \bar{F}_{h,i,:} }[2,\infty]
  \norm{ P^{h' - \max(h,\ell')} \bar{F}_{h',i,:} }[2,\infty]
  \PE^{1/2}_{\varrho}\Big[
  \norm{ g_{\ell,:}(\globRandState{1:\ell}) }^2
  \Big]
  \PE^{1/2}_{\varrho}\Big[
  \norm{ g_{\ell',:}(\globRandState{1:\ell'\!}) }^2
\Big]
  \eqsp,
\end{align*}
where we also used the Cauchy-Schwarz inequality in the second inequality. %
By the mixing property of the Markov chain, and using $\norm{ \bar{F}_{h,i,:} }[2,\infty] \le 2 \norm{{F}_{h,i,:} }[2,\infty] $ we obtain %
\begin{align*}
\abs{ \intermterm{A}{\ell, \ell'}[h,h'] }
  & \le
    4 \sum_{i=1}^d \boundmixing{(h' - \max(h,\ell'))}
    \norm{ F_{h,i,:} }[2,\infty]
    \norm{ F_{h',i,:} }[2,\infty]
  \PE^{1/2}_{\varrho}\Big[
  \norm{ g_{\ell,:}(\globRandState{1:\ell}) }^2
  \Big]
  \PE^{1/2}_{\varrho}\Big[
  \norm{ g_{\ell',:}(\globRandState{1:\ell'\!}) }^2
    \Big]
    \\
  & \le
    4 \boundmixing{(h' - \max(h,\ell'))} \triplenorm{ F_{h} }[\infty]
    \triplenorm{ F_{h'} }[\infty]
  \PE^{1/2}_{\varrho}\Big[
  \norm{ g_{\ell,:}(\globRandState{1:\ell}) }^2
  \Big]
  \PE^{1/2}_{\varrho}\Big[
  \norm{ g_{\ell',:}(\globRandState{1:\ell'\!}) }^2
\Big]
  \eqsp.
\end{align*}
Finally, summing over all $\ell, \ell'$, we obtain
\begin{align*}
  \sum_{\ell=1}^h \! \sum_{\ell'=1}^{h'} \! \abs{ \intermterm{A}{\ell, \ell'}[h,h'] } 
& \!\le\!
 4 \sum_{\ell = 0}^{h-1} \! \sum_{\ell' = 0}^{h'-1} \boundmixing{(h'\! - \!\max(h,\ell'))}  \! \triplenorm{ F }[2,\infty]^2
\PE_{\varrho}^{1/2} \Big[ \norm{ g_{\ell}(\globRandState{1:\ell}) }^2 \Big]
\PE_{\varrho}^{1/2} \Big[ \norm{ g_{\ell'}(\globRandState{1:\ell'}) }^2 \Big] 
\eqsp.
\end{align*}
We then bound each $\PE^{1/2}[ \norm{g_{\ell+1}(\globRandState{1:\ell})}^2 ]$ by their supremum over $\ell \in \intlist{1}{\nlupdates}$, sum over all $h \neq h'$, and use the triangle inequality to obtain
\begin{align*}
  &   \Big| 2 \! \sum_{h' > h} \!\sum_{\ell=1}^h \!\sum_{\ell'=1}^{h'} \! \intermterm{A}{\ell, \ell'}[h,h'] \Big|
   \! \le
8 \! \sum_{h' > h}  \sum_{\ell = 0}^{h-1} \sum_{\ell' = 0}^{h'-1} \boundmixing{(h' - \max(h,\ell'))} \triplenorm{ F }[2,\infty]^2
    \sup_{1 \le m \le d}
    \PE_{\varrho} \Big[ \norm{ g_{m}(\globRandState{1:m}) }^2 \Big]
\eqsp.
\end{align*}
We then remark that 
\begin{align*}
  \sum_{\ell = 0}^{h-1} \sum_{\ell' = 0}^{h'-1} \boundmixing{(h' - \max(h,\ell'))}
  \!\le\!
\sum_{\ell = 0}^{h-1} \sum_{\ell' = 0}^{h' - 1}
\boundmixing{(h' - h)} \!+ \boundmixing{(h' - \ell')} 
& \!\le\! 
\frac{8 h' \taumix}{3}
\eqsp,
\end{align*}
and the bound on the first term follows by summing this inequality over $h' > h$, and combining the resulting inequality with \eqref{eq:bound-centered-non-cented-prod-first-sum}.
\end{proof}
\begin{corollary}
\label{cor:bound-product-variance-interm-sum-markov}
Under the assumptions of \Cref{lem:bound-product-variance-interm-sum-markov}, it holds that
\begin{align*}
\PE_{\varrho} \Big[ \bnorm{ \sum_{h=1}^{\nlupdates} \sum_{\ell=1}^{h-1} \big\{ F_{h}(\globRandState{h}) - \statdist{} F_h \big\} g_{\ell}(\globRandState{1:\ell}) }^2 \Big]
& \le 
48 \nlupdates^3 \taumix 
\triplenorm{ F }[2,\infty]^2 
\sup_{h \in \intlist{1}{\nlupdates} } \PE[ \norm{g_{h}(\globRandState{1:h}) }^2 ]
  \eqsp,
  \\
\PE_{\varrho} \Big[ \bnorm{ \sum_{h=1}^{\nlupdates} \sum_{\ell=1}^{h-1} \big\{ F_{h}(\globRandState{h}) - \statdist{} F_h \big\} g_{\ell}(\globRandState{1:\ell})  }^2 \Big]
& \le
48 \triplenorm{ F }[2,\infty]^2 \triplenorm{ g }[2,\infty]^2 \nlupdates^3 \taumix
\eqsp.
\end{align*}    
\end{corollary}
\begin{proof}
Recentering the $F_h(Z_h)$'s on their expectation, and using the fact that $\statdist{} = P \statdist{}$, we get
\begin{align}
\nonumber
& \PE_{\varrho} \Big[ \bnorm{ \sum_{h=1}^{\nlupdates}\sum_{\ell=1}^{h-1} \Big\{ {F}_{h}(\globRandState{h}) - \statdist{}  {F}_{h} \Big\} g_{\ell}(\globRandState{1:\ell}) }^2 \Big]
\\
\label{eq:decomp-error-prod-fg}
& \le
2 \PE_{\varrho} \Big[ \bnorm{ \sum_{h=1}^{\nlupdates}\sum_{\ell=1}^{h-1} \bar{F}_{h}(\globRandState{h}) g_{\ell}(\globRandState{1:\ell}) }^2 \Big]
+ 
2 \PE_{\varrho} \Big[ \bnorm{ \sum_{h=1}^{\nlupdates}\sum_{\ell=1}^{h-1} \Big\{ \varrho P^h {F}_{h} - \statdist{} P^h F_h \Big\}  g_{\ell}(\globRandState{1:\ell}) }^2 \Big]
\eqsp.
\end{align}
The first term is directly bounded by \Cref{lem:bound-product-variance-interm-sum-markov}.
The second term can be bounded using Jensen's inequality together with the mixing property
\begin{align*}
2 \PE_{\varrho} \Big[ \bnorm{ \sum_{h=1}^{\nlupdates}\sum_{\ell=1}^{h-1} \Big\{ \varrho P^h {F}_{h} \!-\! \statdist{} P^h F_h \Big\}  g_{\ell}(\globRandState{1:\ell}) }^2 \Big]
& \le
2 \nlupdates
\sum_{h = 1}^\nlupdates
h
\sum_{\ell = 1}^{h-1}
\triplenorm{ F }[2,\infty]^2 \boundmixing{h}
\PE[ \norm{ g_\ell(\globRandState{1:\ell}) }^2 ]
\\
& \le
\frac{4\nlupdates^3 \taumix}{3}
\triplenorm{ F }[2,\infty]^2 
\sup_{1 \le h \le \nlupdates} 
\PE[ \norm{ g_h(\globRandState{1:h}) }^2 ]
\eqsp,
\end{align*}
and the result follows.
\end{proof}

When $g$ is centered and only depends on one of the Markov iterates, we have the refined bound.
\begin{lemma}
\label{lem:bound-product-variance-interm-sum-markov-refined-g-last}
Assume that $P$ is uniformly geometrically ergodic with mixing time $\taumix$. Let $F_h : \Zset \rightarrow \rset^{d \times d}$ and $g_h : \Zset \rightarrow \rset^d$ be uniformly bounded for $h \ge 0$. %
Then, for any initial distribution $\varrho$, it holds that
\begin{align*}
\PE_{\varrho} \Big[ \bnorm{ \sum_{h=1}^{\nlupdates} \sum_{\ell=1}^{h-1} \{F_{h}(\globRandState{h}) - \varrho P^h F_h\} \{ g_{\ell}(\globRandState{\ell}) - \varrho P^{\ell} g_\ell\} }^2 \Big] 
\le
  76 \triplenorm{F}[2,\infty]^2 \triplenorm{g}[2,\infty]^2 \nlupdates^2 \taumix^2
  \eqsp.
\end{align*}
\end{lemma}
\begin{proof}
We set $\bar{F}_h= F_h - \varrho P^h F_h$ and $\bar{g}_\ell= g_\ell - \varrho P^\ell g_\ell$.
Expanding the square, we have
\begin{align*}
& \PE_{\varrho} \Big[ \bnorm{ \sum_{h=1}^{\nlupdates} \sum_{\ell=1}^{h-1} \bar{F}_{h}(\globRandState{h}) \bar{g}_{\ell}(\globRandState{\ell}) }^2 \Big]
\\
& \quad = 
\sum_{h=1}^{\nlupdates} \PE_{\varrho} \Big[ \bnorm{ \sum_{\ell=1}^{h-1} \bar{F}_{h}(\globRandState{h}) \bar{g}_{\ell}(\globRandState{\ell}) }^2 \Big]
+ 2 \sum_{h < h' = 1 }^{\nlupdates}
\sum_{\ell=1}^{h-1}  \sum_{\ell'=1}^{h'-1} 
\PE_{\varrho} [   \pscal{ \bar{F}_{h}(\globRandState{h}) \bar{g}_{\ell}(\globRandState{\ell}) }{ \bar{F}_{h'}(\globRandState{h'}) \bar{g}_{\ell'}(\globRandState{\ell'}) } ]
\eqsp.
\end{align*}
\textbf{Bound on the variance term.}
To bound the first term, we expand it as%
\begin{align*}
&\PE_{\varrho} \Big[ \bnorm{ \sum_{\ell=1}^{h-1} \bar{F}_{h}(\globRandState{h}) \bar{g}_{\ell}(\globRandState{\ell}) }^2 \Big] 
\\
& \quad 
= \sum_{\ell=1}^{h-1} \PE_{\varrho}[ \norm{\bar{F}_h(Z_h) \bar{g}_\ell(Z_\ell)}^2] + 2 \sum_{\ell < \ell' = 1}^{h-1} 
\PE_{\varrho}[\pscal{\bar{F}_h(Z_h) \bar{g}_\ell(Z_\ell)}{\bar{F}_h(Z_h) \bar{g}_{\ell'}(Z_{\ell'})}] 
\eqsp.
\end{align*}
It is easily seen %
that $\PE_{\varrho}[ \norm{\bar{F}_h(Z_h) \bar{g}_\ell(Z_\ell)}^2] \leq 4 \triplenorm{F}[2,\infty]^2 \triplenorm{g}[2,\infty]^2$.
On the other hand, we write
\begin{align*}
& \PE_{\varrho}[\pscal{\bar{F}_h(Z_h) \bar{g}_\ell(Z_\ell)}{\bar{F}_h(Z_h) \bar{g}_{\ell'}(Z_{\ell'})}] 
 =
\sum_{i = 1}^d
\sum_{j,j' = 1}^d
\bar{F}_{h,i,j}(Z_h) \bar{g}_{\ell,j}(Z_\ell) 
\bar{F}_{h,i,j'}(Z_h) \bar{g}_{\ell',j'}(Z_{\ell'})
\\
& =
\sum_{i = 1}^d
\sum_{j,j' = 1}^d
\overline{FF}_{h,i,j,j'}(Z_h)
\bar{g}_{\ell,j}(Z_\ell) \bar{g}_{\ell',j'}(Z_{\ell'})
+
\varrho P^h (\bar{F}_{h,i,j} \bar{F}_{h,i,j'})
\bar{g}_{\ell,j}(Z_\ell) \bar{g}_{\ell',j'}(Z_{\ell'})
\eqsp,
\end{align*}
where we denote $\overline{FF}_{h,i,j,j'}(Z_h) = \bar{F}_{h,i,j}(Z_h) \bar{F}_{h,i,j'}(Z_h) - \varrho P^h (\bar{F}_{h,i,j} \bar{F}_{h,i,j'})$.
By the triangle inequality, \Cref{lem:coupling-indep-2} and the mixing property, we obtain
\begin{align*}
& \abs{ \PE_{\varrho}[\pscal{\bar{F}_h(Z_h) \bar{g}_\ell(Z_\ell)}{\bar{F}_h(Z_h) \bar{g}_{\ell'}(Z_{\ell'})}] }
\\
& \le
\sum_{i = 1}^d
\sum_{j,j' = 1}^d
16 \norm{{F}_{h,i,j}}[\infty]
\norm{{F}_{h,i,j'}}[\infty]
\norm{{g}_{\ell,j}}[\infty]
\norm{{g}_{\ell',j'}}[\infty]
\boundmixing{(h-\ell')}
\boundmixing{(\ell'-\ell)}
\\
& \qquad\qquad \quad\qquad
+
16 \norm{{F}_{h,i,j}}[\infty]
\norm{{F}_{h,i,j'}}[\infty]
\norm{{g}_{\ell,j}}[\infty]
\norm{{g}_{\ell',j'}}[\infty]
\boundmixing{(\ell'-\ell)}
\eqsp.
\end{align*}
Using the definition of $\triplenorm{\cdot}[2,\infty]$, we obtain
\begin{align*}
\abs{ \PE_{\varrho}[\pscal{\bar{F}_h(Z_h) \bar{g}_\ell(Z_\ell)}{\bar{F}_h(Z_h) \bar{g}_{\ell'}(Z_{\ell'})}] }
& \le
32 \triplenorm{F}[2,\infty]^2\triplenorm{g}[2,\infty]^2
\boundmixing{(\ell'-\ell)}
\eqsp.
\end{align*}
Summing over $h, \ell$, and $\ell'$, this gives a bound on the first term
\begin{align}
\nonumber
\sum_{h=1}^{\nlupdates} \PE_{\varrho} \Big[ \bnorm{ \sum_{\ell=1}^{h-1} \bar{F}_{h}(\globRandState{h}) \bar{g}_{\ell}(\globRandState{\ell}) }^2 \Big] 
& 
\le
4 \nlupdates^2 \triplenorm{F}[2,\infty]^2 \triplenorm{g}[2,\infty]^2 
+ 2 \sum_{h=1}^{\nlupdates} \frac{4 \cdot 32}{3} h \triplenorm{F}[2,\infty]^2 \triplenorm{g}[2,\infty]^2 \taumix^2
\\
& \label{eq:bound-double-centered-second-part}
\le
\frac{140}{3} \nlupdates^2\taumix^2 \triplenorm{F}[2,\infty]^2 \triplenorm{g}[2,\infty]^2 
\eqsp.
\end{align}

\textbf{Bound on the covariance term.}
To study the covariance terms, we consider all different orderings of $h, h', \ell$, and $\ell'$.
Note that $h' > h$, $h' > \ell'$, and $h > \ell$ by construction, there are thus three possible orderings: $h' > h > \ell' \ge \ell$, $h' > h > \ell > \ell'$, and $h' > \ell' \ge h > \ell$.
We now treat each case separately. We set %
for $i,i',j,j' \in \{1,\dots,d\}$,    
\[
\intermterm{B}{h,\ell,h',\ell'}[i,i',j,j'] 
= \PE_{\statdist{}}[\bar{F}_{h,i,j}(\globRandState{h}) \bar{g}_{\ell,j}(\globRandState{\ell}) \bar{F}_{h',i',j'}(\globRandState{h'}) \bar{g}_{\ell',j'}(\globRandState{\ell'})]
\eqsp.
\]
Denoting $\intermterm{C}{h,\ell,h',\ell'}[i,i',j,j'] = 16 \norm{{F}_{h,i,j}}[\infty] \norm{{g}_{\ell,j}}[\infty] \norm{{F}_{h',i',j'}}[\infty] \norm{{g}_{\ell',j'}}[\infty]$. 

\textit{(Case $h' > h > \ell' > \ell$.)}
Applying \Cref{lem:coupling-indep-2}, we get, since $\ell' > \ell$,
\begin{align*}
& \abs{\intermterm{B}{h,\ell,h',\ell'}[i,i',j,j'] } 
\le
\intermterm{C}{h,\ell,h',\ell'}[i,i',j,j']
\Big( 
\boundmixing{(h'-h)} \boundmixing{(\ell'-\ell)}
+ 
 \boundmixing{(h'-h)} \boundmixing{(h-\ell')}
\Big)
\eqsp.
\end{align*}

\textit{(Case $h' > h > \ell \ge \ell'$.)} We simply need to switch the $\ell'$ and $\ell$.
\begin{align*}
\abs{\intermterm{B}{h,\ell,h',\ell'}[i,i',j,j'] } 
\le
\intermterm{C}{h,\ell,h',\ell'}[i,i',j,j']
\Big( 
\boundmixing{(h'-h)} \boundmixing{(\ell-\ell')}
+ 
 \boundmixing{(h'-h)} \boundmixing{(h-\ell)}
\Big)
\eqsp.
\end{align*}

\textit{(Case $h' > \ell' \ge h > \ell$.)}
Similarly to the first two cases, applying \Cref{lem:coupling-indep-2}
\begin{align*}
{\abs{\intermterm{B}{h,\ell,h',\ell'}[i,i',j,j'] }}  
\le
\intermterm{C}{h,\ell,h',\ell'}[i,i',j,j']
\Big( 
\boundmixing{(h'-\ell')} \boundmixing{(h-\ell)}
+ 
\boundmixing{(h'-\ell')} \boundmixing{(\ell'-h)}
\Big)
\eqsp.
\end{align*}
\textit{(Final bound.)}
Splitting the sum and combining the three above results, we obtain
\begin{align*}
  \sum_{\ell=1}^{h-1}  \sum_{\ell'=1}^{h'-1} 
  \Big| \intermterm{C}{h,\ell,h',\ell'}[i,i',j,j'] \Big|
  & \le
    \intermterm{C}{h,\ell,h',\ell'}[i,i',j,j']
    \sum_{h' > h > \ell' > \ell} \!\!\!\!
    \boundmixing{(h'\!-\!h)} \boundmixing{(\ell'\!-\!\ell)}
    + 
    \boundmixing{(h'\!-\!h)} \boundmixing{(h\!-\!\ell')}
  \\ &  \quad
    +
    \intermterm{C}{h,\ell,h',\ell'}[i,i',j,j'] 
    \sum_{h' > h > \ell \ge \ell'}\!\!\!\!
    \boundmixing{(h'\!-\!h)} \boundmixing{(\ell\!-\!\ell')}
    + 
    \boundmixing{(h'\!-\!h)} \boundmixing{(h\!-\!\ell)}
  \\ &  \quad
    +
    \intermterm{C}{h,\ell,h',\ell'}[i,i',j,j']
    \!\!\!\!\sum_{h' > \ell' \ge h > \ell} \!\!\!\!
    \boundmixing{(h'\!-\!\ell')} \boundmixing{(h\!-\!\ell)}
    + 
    \boundmixing{(h'\!-\!\ell')} \boundmixing{(\ell'\!-\!h)}
  \\
  & 
    \le
    \intermterm{C}{h,\ell,h',\ell'}[i,i',j,j']
 (4/3)^2 \nlupdates^2 \taumix^2
    \eqsp.
\end{align*}
Summing over all $i,i',j,j'$, and using the definitions of the norms, we obtain
\begin{align*}
2 \sum_{h < h' = 1 }^{\nlupdates}
\sum_{\ell=1}^{h-1}  \sum_{\ell'=1}^{h'-1} 
\PE_{\varrho} [   \pscal{ \bar{F}_{h}(\globRandState{h}) \bar{g}_{\ell}(\globRandState{\ell}) }{ \bar{F}_{h'}(\globRandState{h'}) \bar{g}_{\ell'}(\globRandState{\ell'}) } ]
\le
\frac{256}{9} \nlupdates^2 \taumix^2 \triplenorm{F}[2,\infty]^2 \triplenorm{g}[2,\infty]^2
\eqsp,
\end{align*}
and the result follows by combining this inequality with \eqref{eq:bound-double-centered-second-part}.
\end{proof}

\begin{corollary}
\label{coro:bound-product-variance-interm-sum-markov-refined-g-last}
Under the assumptions of \Cref{lem:bound-product-variance-interm-sum-markov-refined-g-last}, with $\statdist{}$ such that $\statdist{} = P \statdist{}$, we have
\begin{align*}
\PE_{\varrho} \Big[ \bnorm{ \sum_{h=1}^{\nlupdates} \sum_{\ell=1}^{h-1} \{F_{h}(\globRandState{h}) - \statdist{} P^h F_h\} \{ g_{\ell}(\globRandState{\ell}) - \varrho P^{\ell} g_\ell\} }^2 \Big] 
\le
  312 \nlupdates^2 \taumix^2 \triplenorm{F}[2,\infty]^2 \triplenorm{g}[2,\infty]^2 
  \eqsp.
\end{align*}
\end{corollary}
\begin{proof}
Using Young's inequality, we have
\begin{align*}
& \PE_{\varrho} \Big[ \bnorm{ \sum_{h=1}^{\nlupdates} \sum_{\ell=1}^{h-1} \{F_{h}(\globRandState{h}) - \statdist{} P^h F_h\} \{ g_{\ell}(\globRandState{\ell}) - \varrho P^{\ell} g_\ell\} }^2 \Big] 
\\
& \quad \le
2 \PE_{\varrho} \Big[ \bnorm{ \sum_{h=1}^{\nlupdates} \sum_{\ell=1}^{h-1} \{F_{h}(\globRandState{h}) - \varrho{} P^h F_h\} \{ g_{\ell}(\globRandState{\ell}) - \varrho P^{\ell} g_\ell\} }^2 \Big] 
\\
& \qquad + 
2 \PE_{\varrho} \Big[ \bnorm{ \sum_{h=1}^{\nlupdates} \sum_{\ell=1}^{h-1} \{  \varrho{} P^h F_h - \statdist{} P^h F_h\} \{ g_{\ell}(\globRandState{\ell}) - \varrho P^{\ell} g_\ell\} }^2 \Big] 
\eqsp.
\end{align*}
The second term can be bounded using \Cref{lem:bound-variance-sum-markov}, which gives
\begin{align*} 
2 \PE_{\varrho} \Big[ \bnorm{ \sum_{h=1}^{\nlupdates} \sum_{\ell=1}^{h-1} \{  \varrho{} P^h F_h - \statdist{} P^h F_h\} \{ g_{\ell}(\globRandState{\ell}) - \varrho P^{\ell} g_\ell\} }^2 \Big] 
& \le 
2 \nlupdates \sum_{h=1}^{\nlupdates} \PE_{\varrho} \Big[ \bnorm{\sum_{\ell=1}^{h-1} \{  \varrho{} P^h F_h - \statdist{} P^h F_h\} \{ g_{\ell}(\globRandState{\ell}) - \varrho P^{\ell} g_\ell\} }^2 \Big] 
\\
& \le 
2 \nlupdates \sum_{h=1}^{\nlupdates} 2 \cdot 15 \cdot 4 h \taumix \boundmixing{h}  \triplenorm{F}[2,\infty]^2 \triplenorm{g}[2,\infty]^2 
\\
& \le 
160 \nlupdates^2 \taumix^2  \triplenorm{F}[2,\infty]^2 \triplenorm{g}[2,\infty]^2
\eqsp,
\end{align*}
where we used $\triplenorm{\varrho{} P^h F_h - \statdist{} P^h F_h}[2,\infty] \le  2\boundmixing{h} \triplenorm{F}[2,\infty] $ in the second inequality.
The first term can be bounded using \Cref{lem:bound-product-variance-interm-sum-markov-refined-g-last}, and the result follows.
\end{proof}

%% file: src/sup-proof-single-sarsa.tex
\subsection{Bounds on policy improvement}
In subsequent proofs, we will require the following lemma, restated from \citet{zou2019finite}'s Lemma 3.
It allows to bound the difference between the two measures $\mu_{\theta_1}$ and $\mu_{\theta_2}$, parameterized by two parameters $\param_1, \param_2$, as a function of the distance between these two parameters.
\begin{lemma}[from \cite{zou2019finite}]
\label{lem:bound-distrib-diff-theta}
Assume \Cref{assum:bounded-A-b} and \Cref{assum:markov-chain}.
Let $\mu_\param$ the invariant joint distribution of state and action after following policy $\pi_\param$. Then, for $\theta_1, \theta_2 \in \Rr^d$ we have:
\begin{align}
    \norm{ \mu_{\theta_1} - \mu_{\theta_2} }[\TV]
    \leq 
    \invdistlip
    \norm{ \theta_1 - \theta_2 } 
    \eqsp,
\end{align}
where $\invdistlip \eqdef \implip \nactions 
    \left(1 + 4 \taumix \right)$.
\end{lemma}
\begin{proof}
Note that under our assumption \Cref{assum:markov-chain}, the assumptions of \citet{zou2019finite} are satisfied with $m = 4$ and $\rho = (1/4)^{1/\taumix}$.
Lemma 3 from \citet{zou2019finite} then gives
\begin{align*}
\invdistlip 
\le 
\implip \nactions \left( 1 + \lceil \log(1/m) \rceil + \frac{1}{1-\rho} \right) 
\le
\implip \nactions \left( 1 + \frac{1}{1 - \exp(-\log(4)/\taumix)} \right)
\eqsp,
\end{align*}
and the result follows from $\exp(-x) \le 1 - x/2$ and $1/\log(4) \le 2$.
\end{proof}
A crucial corollary of this lemma is that the matrix $\nAs{1}{\cdot}$ and the vector $\nbs{1}{\cdot}$ are Lipschitz.
\begin{corollary}
\label{coro:lip-A-b}
Assume \Cref{assum:bounded-A-b} and \Cref{assum:markov-chain}.
Then, for any $\param \in \rset^d$, the following property holds
\begin{align*}
\norm{ \nAs{1}{\globparam{}}\!-\! \nAs{1}{\paramlim} }
& \le
\invdistlip \boundA \norm{ \globparam{} - \paramlim }
\eqsp,
\quad
\norm{  \nbs{1}{\globparam{}}\!-\! \nbs{1}{\paramlim} }
\le
\invdistlip \boundb \norm{ \globparam{} - \paramlim }
\eqsp,
\end{align*}
where 
$\invdistlip = \implip \nactions 
    \left(1 + 4 \taumix \right)$ is defined in \Cref{lem:bound-distrib-diff-theta}.
\end{corollary}
\begin{proof}
From \eqref{eq:def-bar-Ab}, we have, with $\nAs{c}{s,a,s',a'} = \feature(s, a) \left(\gamma \feature(s', a')^\top - \feature(s, a)^\top\right)$, that
\begin{align*}
\nbarA{c}(\param) = \PE\!\!_{\substack{(s, a) \sim \mu_{\theta} \\ s' \sim \nkerMDP{c}(\cdot|s,a) \\ a' \sim \pi_\theta(\cdot|s')} } \!\!\left[ \nAs{c}{s,a,s',a'}\right]
\eqsp.
\end{align*}
This gives
\begin{align*}
\norm{ \nAs{1}{\globparam{}}\!-\! \nAs{1}{\paramlim} }
& =
\bnorm{ 
\PE\!\!_{\substack{(s, a) \sim \mu_{\globparam{}} \\ s' \sim \nkerMDP{c}(\cdot|s,a), \\ a' \sim \pi_{\globparam{}}(\cdot|s')}}\!\![ \nAs{c}{s,a,s',a'}]
- 
\PE\!\!_{\substack{(s, a) \sim \mu_{\paramlim} \\ s' \sim \nkerMDP{c}(\cdot|s,a) \\ a' \sim \pi_{\paramlim}(\cdot|s')}}\!\![ \nAs{c}{s,a,s',a'}]
}
\\
& \le
\norm{ \mu_{\globparam{}} - \mu_{\paramlim} }[\TV]
\sup_{s,a,s',a'} \norm{ \nAs{c}{s,a,s',a'} }
\eqsp,
\end{align*}
and the result follows from the definition of $\boundA$ and \Cref{lem:bound-distrib-diff-theta}.
The second inequality follows from similar derivations.
\end{proof}

\subsection{Expression of difference of matrix products}
In our derivations, we will use the follwing lemma, that allows to decompose the difference of two matrix products.
\begin{lemma}
\label{lem:product-diff}
For $c \in \intlist{1}{\nagent}$, $t \ge 0$, $\step_t \ge 0$, and $\nlupdates \ge 0$, we have
\begin{align}
    \loccontract{c}{t,1:\nlupdates} - \Id
    & =
    \step_t \sum_{h=1}^\nlupdates \nbarA{c}(\paramlim) \loccontract{1}{t,1:h-1} 
    \eqsp,
    \\
    \loccontract{c}{t,1:\nlupdates} - \loccontract{avg}{1:\nlupdates} 
    & =
    \step_t \sum_{h=1}^H 
    \loccontract{c}{t,h-1} 
    (\nbarA{c}(\paramlim) - \barA(\paramlim))
    \loccontract{avg}{h+1:\nlupdates} 
   \eqsp.
\end{align}
\end{lemma}
\begin{proof}
Remark that for all $h \ge 0$, we have
\begin{align*}
\loccontract{c}{t,1:h+1} - \Id
& = (\Id + \step_{t} \nbarA{c}(\paramlim)) \loccontract{1}{t,1:h} - \Id
= \left(\loccontract{c}{t,1:h} - \Id \right)
+ \step_{t} \nbarA{c}(\paramlim)\loccontract{1}{t,1:h} 
\eqsp,
\end{align*}
and the first identity follows.    
The second one follows from similar computations.
\end{proof}

%% file: src/sup-terms-decomp.tex
\subsection{Decomposition of the error}
\label{sec:app-error-decomposition}
\decompositionoflocalupdate*
\begin{proof}
We have, by adding and removing $\step_{t} \left( \nbarA{1}(\paramlim) \locparam{1}{t, h} + \nbarb{1}(\paramlim) \right)$ to the update of $\locparam{1}{t,h+1}$,
\begin{align*}
\locparam{1}{t, h+1} - \paramlim
& =
\locparam{1}{t,h} - \paramlim
+ \step_{t} 
\left( 
\nAs{1}{\locRandState{1}{t,h+1}} \locparam{1}{t, h} 
+ \nbs{1}{\locRandState{1}{t,h+1}} \right)
\\
& =
\locparam{1}{t,h} - \paramlim
+ \step_{t} 
\left( 
\nbarA{1}(\paramlim) \locparam{1}{t, h} 
+ \nbarb{1}(\paramlim) \right)
\\
& \quad 
+ \step_{t} 
\left( 
( \nAs{1}{\locRandState{1}{t,h+1}} - \nbarA{1}(\paramlim) ) \locparam{1}{t, h} 
+ \nbs{1}{\locRandState{1}{t,h+1}} - \nbarb{1}(\paramlim) \right)
\eqsp.
\end{align*}
Since {$\nbarb{1}(\paramlim) = - \nbarA{1}(\paramlim) \paramlim$} in the single-agent case, we have
\begin{align*}
\locparam{1}{t, h+1} - \paramlim
& =
\locparam{1}{t,h} - \paramlim
+ \step_{t} \nbarA{1}(\paramlim) \left(\locparam{1}{t, h} - \paramlim \right)
\\
& \quad 
+ \step_{t} 
\left( 
( \nAs{1}{\locRandState{1}{t,h+1}} - \nbarA{1}(\globparam{t}) ) \locparam{1}{t, h} 
+ \nbs{1}{\locRandState{1}{t,h+1}} - \nbarb{1}(\globparam{t}) \right)
\\
& \quad 
+ \step_{t} 
\left( 
( \nbarA{1}(\globparam{t}) - \nbarA{1}(\paramlim) ) \locparam{1}{t, h} 
+ \nbarb{1}(\globparam{t}) - \nbarb{1}(\paramlim) \right)
\eqsp.
\end{align*}
We then replace $\locerrordet{1}{t,h}$ and $\locnoisetheta{1}{t,h}{\locRandState{1}{t,h+1}}$ by their definitions
\begin{align*}
\locerrordet{1}{t,h}
& =
( \nbarA{1}(\globparam{t}) - \nbarA{1}(\paramlim) ) \locparam{1}{t, h} 
+ \nbarb{1}(\globparam{t}) - \nbarb{1}(\paramlim) 
\eqsp,
\\
\locnoisetheta{1}{t,h}{\locRandState{1}{t,h+1}}
& =
( \nAs{1}{\locRandState{1}{t,h+1}} - \nbarA{1}(\globparam{t}) ) \locparam{1}{t, h} 
+ \nbs{1}{\locRandState{1}{t,h+1}} - \nbarb{1}(\globparam{t})
\eqsp,
\end{align*}
and the result follows after unrolling the recursion.
\end{proof}

\subsection{Bounds on iterates' norm and variance}
First, we provide some bounds on the differences between the global and local iterates.
\begin{lemma}[Bound on the local iterates]
\label{lem:bound-local-step-vs-global-step}
Assume \Cref{assum:bounded-A-b}.
Let $c \in \intlist{1}{\nagent}$, $t, h \ge 0$, and $\globparam{t} \in \rset^d$, and $\locparam{c}{t,h} \in \rset^d$ be the parameters obtained after $h$ local TD updates started from $\globparam{t}$.
Then, whenever $\step_t \nlupdates \boundA \le 1$ it holds that
\begin{enumerate}[label=(\alph*)]
    \item \label{lem:bound-local-step-vs-global-step:diff}
    the distance between last policy improvement and current estimate is bounded
    \begin{align*}
    \norm{ \globparam{t} - \locparam{c}{t,h} }
    & \le
    3 \step_t \nlupdates ( \boundA \projradius + \boundb )
    \le
    3(\projradius + 1)
    \eqsp,
    \end{align*}
    \item \label{lem:bound-local-step-vs-global-step:unif} 
    the norm of current iterate is bounded as
    \begin{align*}
        \norm{ \locparam{c}{t,h} }
        & \le
        \locprojradius \eqdef
        4(\projradius + 1)
        \eqsp,
    \end{align*}
  \item \label{lem:bound-local-step-variance}
    if $\step_t \nlupdates \boundA \le 1/6$, the variance, conditionally on $\globfiltr{t}$, of the local updates is bounded as
    \begin{align*}
        \CPE{ \norm{ \locparam{c}{t,h} -  \CPE{ \locparam{c}{t,h} }{\globfiltr{t} }}^2}{\globfiltr{t}}
      & \le
        55 \step_t^2 \nlupdates \taumix \boundGrad^2
        \eqsp,
    \end{align*}
\end{enumerate}
\end{lemma}
\begin{proof}
\textit{Proof of \ref{lem:bound-local-step-vs-global-step}-\ref{lem:bound-local-step-vs-global-step:diff}.}
    Triangle inequality gives
    \begin{align*}
        \norm{ \globparam{t} - \locparam{c}{t,h+1} }
        & =
        \norm{ \globparam{t} - \locparam{c}{t,h} + \locparam{c}{t,h} - \locparam{c}{t,h+1} }
        \le
        \norm{ \globparam{t} - \locparam{c}{t,h} }
        + \norm{ \locparam{c}{t,h} - \locparam{c}{t,h+1} }
        \eqsp.
    \end{align*}
    Plugging in the update, we have
    \begin{align*}
        \norm{ \globparam{t} - \locparam{c}{t,h+1} }
        & \le
        \norm{ \globparam{t} - \locparam{c}{t,h} }
        + \step_t \norm{ \nAs{c}{\locRandState{c}{t,h+1}} \locparam{c}{t, h} 
+ \nbs{c}{\locRandState{c}{t,h+1}} }
        \\
        & \le
        \norm{ \globparam{t} - \locparam{c}{t,h} }
        + \step_t \norm{ \nAs{c}{\locRandState{c}{t,h+1}} } \norm{ \locparam{c}{t, h} }
        + \step_t \norm{ \nbs{c}{\locRandState{c}{t,h+1}} }
        \eqsp.
    \end{align*}
    Bounding $\norm{ \locparam{c}{t,h} } \le \norm{ \locparam{c}{t,h} - \globparam{t} } + \norm{ \globparam{t} }$ and using the bounds from \Cref{assum:bounded-A-b}, we obtain
    \begin{align*}
        \norm{ \globparam{t} - \locparam{c}{t,h+1} }
        & \le
        \norm{ \globparam{t} - \locparam{c}{t,h} }
        + \step_t \boundA \norm{ \locparam{c}{t, h} - \globparam{t} }
        + \step_t \boundA \norm{ \globparam{t} }
        + \step_t \boundb
        \\
        & \le
        (1 + \step_t \boundA) \norm{ \globparam{t} - \locparam{c}{t,h} }
        + \step_t \boundA \projradius
        + \step_t \boundb
        \eqsp,
    \end{align*}
    where the second inequality comes from $\norm{ \globparam{t} } \le \projradius$.
    Unrolling the recursion, we obtain
    \begin{align*}
        \norm{ \globparam{t} - \locparam{c}{t,h} }
        & \le
        \step_t
        \sum_{\ell=0}^{h-1}
        (1 + \step_t \boundA)^\ell ( \boundA \projradius + \boundb )
        \le
        3 \step_t \nlupdates ( \boundA \projradius + \boundb )
        \eqsp,
    \end{align*}
    where the last inequality comes from the fact that, for $\ell \le \nlupdates$ and $\step_t \boundA \le 1/\nlupdates$, it holds that $(1 + \step_t \boundA)^\ell \le (1 + 1/\nlupdates)^\nlupdates \le 3$.
    The bound follows from $\step_t \nlupdates \le 1/\boundA$ and $\boundb / \boundA \le 1$.

    \textit{Proof of \ref{lem:bound-local-step-vs-global-step}-\ref{lem:bound-local-step-vs-global-step:unif}.}
    The second bound follows from $\norm{ \locparam{c}{t,h} } \le \norm{ \locparam{c}{t,h} - \globparam{t} } + \norm{ \globparam{t} }$ and \ref{lem:bound-local-step-vs-global-step}-\ref{lem:bound-local-step-vs-global-step:diff}.

    \textit{Proof of \ref{lem:bound-local-step-vs-global-step}-\ref{lem:bound-local-step-variance}.}
    First, remark that $\CPE{ \norm{ \locparam{c}{t,0} - \CPE{ \locparam{c}{t,0}}{\globfiltr{t}} }^2 }{\globfiltr{t}} = 0$, thus the result holds for $h = 0$.
    Let $h \ge 0$, and assume that the result holds for all $\ell \le h$.
    Then, we have
    \begin{align*}
      & \CPE{ \norm{ \locparam{c}{t,h+1} - \CPE{ \locparam{c}{t,h+1}}{\globfiltr{t}} }^2 }{\globfiltr{t}}
      \\
      & =
       \bCPE{ \bnorm{ \step_t \sum_{\ell=0}^h \nAs{c}{\locRandState{c}{\ell+1}} \locparam{c}{t,\ell} + \nbs{c}{\locRandState{c}{\ell+1}} - \CPE{ \nAs{c}{\locRandState{c}{\ell+1}} \locparam{c}{t,\ell} + \nbs{c}{\locRandState{c}{\ell+1}} }{\globfiltr{t} } }^2 }{\globfiltr{t}}
      \\
      & \le
        2 \bCPE{ \bnorm{ \step_t \sum_{\ell=0}^h \nAs{c}{\locRandState{c}{\ell+1}} \locparam{c}{t,\ell} - \CPE{ \nAs{c}{\locRandState{c}{\ell+1}} \locparam{c}{t,\ell} }{\globfiltr{t} } }^2 }{\globfiltr{t}}
    \\
    & \quad 
    + 2 \bCPE{ \bnorm{ \step_t \sum_{\ell=0}^h \nbs{c}{\locRandState{c}{\ell+1}} - \CPE{ \nbs{c}{\locRandState{c}{\ell+1}} }{\globfiltr{t} } }^2 }{\globfiltr{t}}
        \eqsp.
    \end{align*}
    We decompose
    \begin{align*}
      & \nAs{c}{\locRandState{c}{\ell+1}} \locparam{c}{t,\ell} - \CPE{ \nAs{c}{\locRandState{c}{\ell+1}} \locparam{c}{t,\ell} }{\globfiltr{t} }
      \\ &
        = \nAs{c}{\locRandState{c}{\ell+1}} \Big(\locparam{c}{t,\ell} - \CPE{ \locparam{c}{t,\ell} }{\globfiltr{t}} \Big)
      \\ & \quad
           + \Big( \nAs{c}{\locRandState{c}{\ell+1}} - \CPE{  \nAs{c}{\locRandState{c}{\ell+1}} }{ \globfiltr{t} } \Big) \CPE{ \locparam{c}{t,\ell} }{\globfiltr{t} }
           + \bCPE{ \nAs{c}{\locRandState{c}{\ell+1}} \Big(\locparam{c}{t,\ell} - \CPE{ \locparam{c}{t,\ell} }{\globfiltr{t} } \Big) }{\globfiltr{t} }
           \eqsp,
    \end{align*}
    which gives, using $\norm{a+b+c}^2 \le 3 \norm{a}^2 + 3 \norm{b}^2 + 3 \norm{c}^2$,
    \begin{align*}
      & \bCPE{ \bnorm{ \step_t \sum_{\ell=0}^h \nAs{c}{\locRandState{c}{\ell+1}} \locparam{c}{t,\ell} - \CPE{ \nAs{c}{\locRandState{c}{\ell+1}} \locparam{c}{t,\ell} }{\globfiltr{t} } }^2 }{\globfiltr{t}}
      \\[-0.5em]
      & \le
        6 \bCPE{ \bnorm{ \step_t \sum_{\ell=0}^h \nAs{c}{\locRandState{c}{\ell+1}} \Big(\locparam{c}{t,\ell} - \CPE{ \locparam{c}{t,\ell} }{\globfiltr{t}} \Big) }^2 }{\globfiltr{t}}
      \\[-0.5em]
      & \quad
        +
        3 \bCPE{ \bnorm{ \step_t \sum_{\ell=0}^h \Big( \nAs{c}{\locRandState{c}{\ell+1}} - \CPE{  \nAs{c}{\locRandState{c}{\ell+1}} }{ \globfiltr{t} } \Big) \CPE{ \locparam{c}{t,\ell} }{\globfiltr{t} } }^2 }{\globfiltr{t}}
        \eqsp,
    \end{align*}
    where we also used Jensen's inequality to bound 
    \begin{align*}
        & \bnorm{ \step_t \sum_{\ell=0}^h \bCPE{ \nAs{c}{\locRandState{c}{\ell+1}} \Big(\locparam{c}{t,\ell} - \CPE{ \locparam{c}{t,\ell} }{\globfiltr{t} } \Big) }{\globfiltr{t} } }^2 
        \le 
        \bCPE{ \bnorm{ \step_t \sum_{\ell=0}^h { \nAs{c}{\locRandState{c}{\ell+1}} \Big(\locparam{c}{t,\ell} - \CPE{ \locparam{c}{t,\ell} }{\globfiltr{t} } \Big) } }^2 }{\globfiltr{r}}
        \eqsp.
    \end{align*}
    Since the $\nAs{c}{\cdot}$ are uniformly bounded, we can use the induction hypothesis, \Cref{lem:bound-variance-sum-markov}, and \Cref{lem:bound-product-variance-interm-sum-markov}, to obtain
      \begin{align*}
        \!\bCPE{ \bnorm{ \step_t \!\sum_{\ell=0}^h \!\nAs{c}{\locRandState{c}{\ell+1}} \locparam{c}{t,\ell} \!-\! \CPE{ \nAs{c}{\locRandState{c}{\ell+1}} \locparam{c}{t,\ell}\! }{\globfiltr{t} } }^2\! }{\globfiltr{t}}
        & \!\!\le\!
          330 \step_t^4 \nlupdates^3 \boundA^2 \taumix \boundGrad^2
         + 45 \step_t^2 \nlupdates \taumix \boundA^2 \locprojradius^2
          \eqsp.
      \end{align*}
      Similarly, \Cref{lem:bound-variance-sum-markov} gives
      \begin{align*}
        \bCPE{ \bnorm{ \step_t \sum_{\ell=0}^h \nbs{c}{\locRandState{c}{\ell+1}}  - \CPE{ \nbs{c}{\locRandState{c}{\ell+1}} }{\globfiltr{t} } }^2 }{\globfiltr{t}}
        & \le
          15 \step_t^2 \nlupdates \taumix \boundb^2  
          \eqsp,
      \end{align*}
      and the result follows from $\step_t \nlupdates \boundA \le 1/6$.
\end{proof}

\begin{proposition}
\label{prop:contract-multiple-step}
Assume \Cref{assum:bounded-A-b}.
Let $t \ge 0$, and $\step_t \le 1/\boundA$.
Then, for any vector $u \in \rset^d$, and $k \le h$, we have
\begin{align*}
\norm{ \loccontract{1}{t,k:h} u }^2 \le (1 - \step_{t} \lminAbar)^{k-h+1} \norm{ u }^2
\eqsp.
\end{align*}
\end{proposition}
\begin{proof}
The result follows directly from \Cref{assum:bounded-A-b}.
\end{proof}

\begin{lemma}
\label{lem:bound-norm-sum-expect-epsilon}
Assume \Cref{assum:bounded-A-b}-\ref{assum:lipschitz-improvement}.
Let $t \ge 0$ and assume that $\step_t \nlupdates \boundA \le 1$, then for $c \in \intlist{1}{\nagent}$, we have
\begin{align*}
\bnorm{ \PE\Big[ \sum_{h=1}^{H} \loccontract{c}{t,h+1:H} \locnoisetheta{c}{t,h}{\locRandState{c}{t,h}} \Big] }
& \le 
\frac{8 \boundGrad  }{3}
\big(\step_t \boundA \nlupdates \taumix + \startedfromstationary \taumix \big)
\eqsp,
\end{align*}
where $\startedfromstationary = 0$ if $\locRandState{c}{t,0}$ is sampled from the stationary distribution $\statdist{\globparam{t}}$ and $\startedfromstationary=1$ otherwise.
\end{lemma}
\begin{proof}
If $\startedfromstationary=1$, we use \Cref{lem:relation-markov-stationary} to construct a Markov chain $\locRandStateOpt{c}{t,h}$ such that $\locRandStateOpt{c}{t,h} \sim \statdist{\globparam{t}}$ starts from the stationary distribution of the policy $\policy_{\globparam{t}}$ and
\begin{align}
\label{eq:lem-proof-expect-sum-eps:bound-markov:first-ineq}
\bnorm{ \PE\Big[ \sum_{h=1}^{H} \loccontract{c}{t,h+1:H} \locnoisetheta{c}{t,h}{\locRandState{c}{t,h}} \Big] }
& \le 
\bnorm{ \PE\Big[ \sum_{h=1}^{H} \loccontract{c}{t,h+1:H} \locnoisetheta{c}{t,h}{\locRandStateOpt{c}{t,h}} \Big] }
+ \frac{8 (\boundA \locprojradius + \boundb) }{3} \startedfromstationary \taumix
\eqsp,
\end{align}
which also holds when $\startedfromstationary=0$.
Then, we write
\begin{align}
\nonumber
\PE\Big[ \sum_{h=1}^{H} \loccontract{c}{t,h+1:H} 
\locnoisetheta{c}{t,h}{\locRandStateOpt{c}{t,h}} \Big]
& =
\sum_{h=1}^{H} \loccontract{c}{t,h+1:H} \PE \Big[ \nAt{c}{\locRandStateOpt{c}{t,h}} \locparam{c}{t,h-1} + \nbt{c}{\locRandStateOpt{c}{t,h}} \Big]
\\
\label{eq:lem-proof-expect-sum-eps:bound-markov}
& =
- \step_t  \sum_{h=1}^{H} \sum_{\ell=1}^{h-1} \loccontract{c}{t,h+1:H} \PE\Big[ \nAt{c}{\locRandStateOpt{c}{t,h}} \big( \nAs{c}{\locRandStateOpt{c}{t,\ell}} \locparam{c}{t,\ell-1} + \nbs{c}{\locRandStateOpt{c}{t,\ell}} \big)
\Big]
\eqsp,
\end{align}
where we recall that $\nAt{c}{z} = \nAs{c}{z} - \nbarA{c}(\globparam{t})$ and we used the fact that $\PE[ \nAt{c}{\locRandStateOpt{c}{t,h}} \globparam{t} ] = 0$.
Now, remark that, by \Cref{assum:bounded-A-b} and Jensen's inequality,
\begin{align*}
\bnorm{ \PE\Big[ \nAt{c}{\locRandStateOpt{c}{t,h}} ( \nAs{c}{\locRandStateOpt{c}{t,\ell}} \locparam{c}{t,\ell-1} + \nbs{c}{\locRandStateOpt{c}{t,\ell}} \Big] }
& =
\bnorm{ \PE\Big[ \CPE{ \nAt{c}{\locRandStateOpt{c}{t,h}} }{ {\locRandStateOpt{c}{t,\ell} } } \big( \nAs{c}{\locRandStateOpt{c}{t,\ell}} \locparam{c}{t,\ell-1} + \nbs{c}{\locRandStateOpt{c}{t,\ell}}\big) \Big] }
\\
& \le 
\PE\Big[ \bnorm{ \CPE{  \nAt{c}{\locRandStateOpt{c}{t,h}}  }{ \locRandStateOpt{c}{t,\ell} } } \Big]  ( \boundA \locprojradius + \boundb)
\eqsp.
\end{align*}
Then, \Cref{assum:markov-chain} gives the bound $\bnorm{ \CPE{  \nAt{c}{\locRandStateOpt{c}{t,h}}  }{ \locRandStateOpt{c}{t,\ell} } }  \le 2 \boundA \boundmixing{(h-\ell)}$, and we obtain
\begin{align*}
\bnorm{ 
\PE\Big[ \sum_{h=1}^{H} \loccontract{c}{t,h+1:H} \locnoisetheta{c}{t,h}{\locRandStateOpt{c}{t,h}} \Big]
}
& \le
2 \step_t \sum_{h=1}^\nlupdates \sum_{\ell=1}^{h-1} \boundA (\boundA \locprojradius + \boundb) \boundmixing{(h - \ell)}
\eqsp.
\end{align*}
Bounding the inner sum by the sum of the series, plugging the result in \eqref{eq:lem-proof-expect-sum-eps:bound-markov} gives
\begin{align}
\label{eq:lem-proof-expect-sum-eps:bound-markov:last-ineq}
\bnorm{ 
\PE\Big[ \sum_{h=1}^{H} \loccontract{c}{t,h+1:H} \locnoisetheta{c}{t,h}{\locRandStateOpt{c}{t,h}} \Big]
}
& \le
\frac{8 \step_t \boundA (\boundA \locprojradius + \boundb) \nlupdates \taumix }{3}
\eqsp,
\end{align}
and the result follows from plugging \eqref{eq:lem-proof-expect-sum-eps:bound-markov:last-ineq} in \eqref{eq:lem-proof-expect-sum-eps:bound-markov:first-ineq}.
\end{proof}

\begin{lemma}
\label{lem:bound-expect-norm-sq-sum-epsilon-N}
Assume \Cref{assum:bounded-A-b}-\ref{assum:lipschitz-improvement}.
Let $t \ge 0$, and assume the step size satisfies $\step_t \nlupdates \boundA \le 1$, then
\begin{align*}
& \PE\Big[ \bnorm{ \frac{1}{\nagent} \sum_{c=1}^\nagent \sum_{h=1}^{H} \loccontract{c}{t,h+1:H} \locnoisetheta{c}{t,h}{\locRandState{c}{t,h}} }^2 \Big] 
\le
\frac{136 \nlupdates \taumix \boundGrad^2}{\nagent}
+
5504 \step_t^2 \boundA^2 \boundGrad^2 
  \nlupdates^2 \taumix^2 
\eqsp,
\end{align*}
where $\startedfromstationary=1$ if $\locRandState{c}{t,0} \sim \statdist{\globparam{t}}$ and $\startedfromstationary=0$ otherwise.
\end{lemma}
\begin{proof}
We start from
\begin{align}
\nonumber
\PE\Big[ \bnorm{ \frac{1}{\nagent} \sum_{c=1}^\nagent \sum_{h=1}^{H}\loccontract{c}{t,h+1:H} \locnoisetheta{c}{t,h-1}{\locRandState{c}{t,h}} }^2
\Big]
& \le
2 \PE\Big[ \bnorm{ \frac{1}{\nagent} \sum_{c=1}^\nagent \sum_{h=1}^{H} \loccontract{c}{t,h+1:H} \nAt{c}{\locRandState{c}{t,h}} \locparam{c}{t,h-1} }^2
\\
\label{eq:prelim-proof-decomp-a-b}
& \quad +
2 \PE\Big[ \bnorm{ \frac{1}{\nagent} \sum_{c=1}^\nagent \sum_{h=1}^{H} \loccontract{c}{t,h+1:H} \nbt{c}{\locRandState{c}{t,h}} }^2
\eqsp.
\end{align}
First, the second term of \eqref{eq:prelim-proof-decomp-a-b} can be bounded by \Cref{lem:bound-variance-sum-markov-stat}, which gives
\begin{align}
\label{eq:prelim-proof-decomp-b-bound}
2 \PE\Big[ \bnorm{ \frac{1}{\nagent} \sum_{c=1}^\nagent \sum_{h=1}^{H} \loccontract{c}{t,h+1:H} \nbt{c}{\locRandStateOpt{c}{t,h}} }^2
& \le 
\frac{34 \nlupdates \taumix \boundb^2}{\nagent} 
\eqsp.
\end{align}
Then, to bound the first term of \eqref{eq:prelim-proof-decomp-a-b}, we recall that $\locparam{c}{t,h} = \globparam{t} - \step_t \sum_{\ell=1}^{h} \nAs{c}{\locRandStateOpt{c}{t,\ell}} \locparam{c}{t,\ell-1} + \nbs{c}{\locRandStateOpt{c}{t,\ell}}$, which gives the following decomposition, using $\locparam{c}{t,h} = \locparam{c}{t,h} - \CPE{ \locparam{c}{t,h} }{\globfiltr{t}} + \CPE{\locparam{c}{t,h}}{\globfiltr{t}}$, %
\begin{align}
\nonumber
& 2 \PE\Big[ \bnorm{ \frac{1}{\nagent} \sum_{c=1}^\nagent \sum_{h=1}^{H} \loccontract{c}{t,h+1:H} \nAt{c}{\locRandState{c}{t,h}} \locparam{c}{t,h-1} }^2
\Big]
\\
\label{eq:prelim-proof-decomp-a-bound}
& \le
4 \PE\Big[ \bnorm{ \frac{1}{\nagent} \sum_{c=1}^\nagent \sum_{h=1}^{H}  \loccontract{c}{t,h+1:H} \nAt{c}{\locRandState{c}{t,h}} \CPE{\locparam{c}{t,h-1}}{\globfiltr{t}}}^2
\\
\nonumber
& \quad
  + 4 \step_t^2 \PE\Big[ \bnorm{ \frac{1}{\nagent} \sum_{c=1}^\nagent \sum_{h=1}^{H}  \sum_{\ell=1}^{h-1} \loccontract{c}{t,h+1:H} \nAt{c}{\locRandState{c}{t,h}} \Big( \nAs{c}{\locRandState{c}{t,\ell}} \locparam{c}{t,{\ell-1}}
  \!-\! \CPE{ \nAs{c}{\locRandState{c}{t,\ell}} \locparam{c}{t,{\ell-1}} }{\globfiltr{t}} \Big) }^2
\eqsp.
\end{align}
Since the $\nAt{c}{\locRandState{c}{t,h}}$ are centered and independent from one agent to another, we can use \Cref{lem:bound-variance-sum-markov-stat} to bound the first term as
\begin{align}
4 \PE\Big[ \bnorm{ \frac{1}{\nagent} \sum_{c=1}^\nagent \sum_{h=1}^{H}  \loccontract{c}{t,h+1:H} \nAt{c}{\locRandState{c}{t,h}} \PE[ \locparam{c}{t,h} ] }^2
\label{eq:prelim-proof-decomp-a-bound-p1}
& \le 
\frac{136 \nlupdates \taumix \boundA^2 \locprojradius^2}{\nagent}%
\eqsp,
\end{align}
To bound the second term, we use the decomposition
\begin{align*}
\nAs{c}{\locRandState{c}{t,\ell}} \locparam{c}{t,{\ell-1}}
  - \CPE{ \nAs{c}{\locRandState{c}{t,\ell}} \locparam{c}{t,{\ell-1}} }{\globfiltr{t}}
  & =
    \nAs{c}{\locRandState{c}{t,\ell}} \big( \locparam{c}{t,{\ell-1}} - \CPE{ \locparam{c}{t,{\ell-1}} }{\globfiltr{t}} \big)
  \\
  & \quad + \big( \nAs{c}{\locRandState{c}{t,\ell}} - \PE[ \nAs{c}{\locRandState{c}{t,\ell}} ] \big) \CPE{ \locparam{c}{t,{\ell-1}} }{\globfiltr{t}}
  \\
  & \quad + \CPE{ \nAs{c}{\locRandState{c}{t,\ell}} \big( \locparam{c}{t,{\ell-1}} - \CPE{ \locparam{c}{t,{\ell-1}} }{\globfiltr{t}} \big) }{\globfiltr{t}}
    \eqsp.
\end{align*}
Using $\norm{ a + b + c }^2 \le 3 \norm{a}^2 + 3 \norm{b}^2 + 3 \norm{c}^2$ and Jensen's inequality, we have
\begin{align*}
  & 4 \step_t^2 \PE\Big[ \bnorm{ \frac{1}{\nagent} \sum_{c=1}^\nagent \sum_{h=1}^{H}  \sum_{\ell=1}^{h-1} \loccontract{c}{t,h+1:H} \nAt{c}{\locRandState{c}{t,h}} \Big( \nAs{c}{\locRandState{c}{t,\ell}} \locparam{c}{t,{\ell-1}}
  \!-\! \CPE{ \nAs{c}{\locRandState{c}{t,\ell}} \locparam{c}{t,{\ell-1}} }{\globfiltr{t}} \Big) }^2
  \\
  & \le
    6 \cdot 4 \step_t^2 \PE\Big[ \bnorm{ \frac{1}{\nagent} \sum_{c=1}^\nagent  \sum_{h=1}^{H}  \sum_{\ell=1}^{h-1}\loccontract{c}{t,h+1:H} \nAt{c}{\locRandState{c}{t,h}} \nAs{c}{\locRandState{c}{t,\ell}} \Big( \locparam{c}{t,{\ell-1}} - \CPE{ \locparam{c}{t,{\ell-1}} }{\globfiltr{t}} \Big) }^2 \Big]
    \\
  & \quad +
    3 \cdot 4 \step_t^2 \PE\Big[ \bnorm{  \frac{1}{\nagent} \sum_{c=1}^\nagent  \sum_{h=1}^{H}  \sum_{\ell=1}^{h-1} \loccontract{c}{t,h+1:H} \nAt{c}{\locRandState{c}{t,h}} \nAo{c}{\locRandState{c}{t,\ell}} \CPE{ \locparam{c}{t,{\ell-1}} }{\globfiltr{t}} }^2 \Big]
    \eqsp,
\end{align*}
where we defined $\nAo{c}{\locRandState{c}{t,\ell}} = \nAs{c}{\locRandState{c}{t,\ell}} - \CPE{\nAs{c}{\locRandState{c}{t,\ell}}}{\globfiltr{t}}$.
The second term can be bounded using \Cref{coro:bound-product-variance-interm-sum-markov-refined-g-last}, which gives
\begin{align*}
3 \cdot 4 \step_t^2 \PE\Big[ \bnorm{  \frac{1}{\nagent} \sum_{c=1}^\nagent  \sum_{h=1}^{H}  \sum_{\ell=1}^{h-1} \loccontract{c}{t,h+1:H} \nAt{c}{\locRandState{c}{t,h}} \nAt{c}{\locRandState{c}{t,\ell}} \CPE{ \locparam{c}{t,{\ell-1}} }{\globfiltr{t}} }^2 \Big]
  & \le
    3744 \step_t^2 \nlupdates^2 \taumix^2 \boundA^4 \locprojradius^2
    \eqsp.
\end{align*}
To bound the first one, we remark that, by \Cref{lem:bound-local-step-vs-global-step}-\ref{lem:bound-local-step-variance}, we have
\begin{align*}
\sup_{1 \le h \le \nlupdates}
\PE\Big[ \bnorm{ \nAs{c}{\locRandState{c}{t,h}} \Big( \locparam{c}{t,{h-1}} - \CPE{ \locparam{c}{t,{h-1}} }{\globfiltr{t}} \Big) }^2 \Big]
\le 
55 \step_t^2 \boundA^2 \boundGrad^2 \nlupdates \taumix%
\eqsp.
\end{align*}
Thus, by \Cref{cor:bound-product-variance-interm-sum-markov}, we obtain
\begin{align}
    6 \cdot 4 \step_t^2 \PE\Big[ \bnorm{ \frac{1}{\nagent} \sum_{c=1}^\nagent \sum_{h=1}^{H}  \sum_{\ell=1}^{h} \loccontract{c}{t,h+1:H} \nAt{c}{\locRandState{c}{t,h}} \nAs{c}{\locRandState{c}{t,\ell}} \Big( \locparam{c}{t,{\ell-1}} - \PE[ \locparam{c}{t,{\ell-1}} ] \Big) }^2 \Big]
\label{eq:prelim-proof-decomp-a-bound-p2}
  & \le 
  6 \cdot 4 \cdot 48 \cdot 55 
  \step_t^4 \boundA^4 \boundGrad^2 
  \nlupdates^4 \taumix^2 
    \eqsp.
\end{align}
Plugging \eqref{eq:prelim-proof-decomp-a-bound-p1} and \eqref{eq:prelim-proof-decomp-a-bound-p2} in \eqref{eq:prelim-proof-decomp-a-bound}, together with \eqref{eq:prelim-proof-decomp-b-bound}, allows to upper bound \eqref{eq:prelim-proof-decomp-a-b} as
\begin{align*}
\PE\Big[ \bnorm{ \frac{1}{\nagent} \sum_{c=1}^\nagent \sum_{h=1}^{H}\loccontract{c}{t,h+1:H} \locnoisetheta{c}{t,h-1}{\locRandStateOpt{c}{t,h}} }^2
\Big]
&  \le
\frac{136 \nlupdates \taumix \boundGrad^2}{\nagent} 
+
63360 \step_t^4 \boundA^4 \boundGrad^2 
  \nlupdates^4 \taumix^2 
+ 3744 \step_t^2 \nlupdates^2 \taumix^2 \boundA^4 \locprojradius^2
\\
&  \le
\frac{136 \nlupdates \taumix \boundGrad^2}{\nagent}
+
5504 \step_t^2 \boundA^2 \boundGrad^2 
  \nlupdates^2 \taumix^2 
    \eqsp,
\end{align*}
where we also used $\step_t \nlupdates \boundA \le 1/6$, $\boundA^2 \locprojradius^2 \le \boundGrad^2$.
\end{proof}

%% file: src/sup-proof-descent-single-agent.tex
\begin{algorithm}[t]
\caption{\SARSA: Single-Agent State-Action-Reward-State-Action}
\label{algo:single-sarsa}
\begin{algorithmic}[1]
\STATE \textbf{Input:} step sizes $\step_t > 0$, initial parameters $\param_0$, projection set $\projset$, number of local steps $H > 0$, number of communications $\nepisode > 0$, initial distribution $\varrho$ over states
\STATE Initialize first state $s_{-1,\nlupdates} \sim \varrho$ and initial policy $\policy_{\globparam{0}} = \polimprove(\qfunc_{\globparam{0}}$
\FOR{step $t=0$ to $\nepisode-1$}
\STATE Initialize $\locparam{1}{t,0} = \globparam{t}$, take first action $a^{(1)}_{t,0} \sim \policy_{\globparam{t}}(\cdot|s^{(1)}_{t-1,\nlupdates})$
\FOR{step $h=0$ to $\nlupdates-1$}
\STATE Take action $a^{(1)}_{t, h+1} \sim \policy_{\globparam{t}}(\cdot|s^{(1)}_{t, h})$, observe reward $\nreward{1}(s^{(1)}_{t,h}, a^{(1)}_{t,h})$, next state $s^{(1)}_{t, h+1}$
\STATE Compute $\delta^{(1)}_{t,h} = \nreward{1}(s^{(1)}_{t, h+1}, a^{(1)}_{t,h+1}) + \discount \feature(s^{(1)}_{t, h+1}, a^{(1)}_{t, h+1})^\top \locparam{1}{t, h} - \feature(s^{(1)}_{t,h}, a^{(1)}_{t,h})^\top \locparam{1}{t, h}$
\STATE Update $\locparam{1}{t, h+1} =  \locparam{1}{t,h} + \step_{t} \delta^{(1)}_{t,h} \feature(s^{(1)}_{t,h}, a^{(1)}_{t,h}) $
\ENDFOR
\STATE Update global parameter $\globparam{t+1} = \Pi\left( \locparam{1}{t,\nlupdates} \right)$
\STATE Policy improvement $\policy_{\param_{t+1}} = \polimprove(\qfunc_{\param_{t+1}})$
\ENDFOR
\STATE \textbf{Return: } $\globparam{\nepisode}$
\end{algorithmic}
\end{algorithm}

\singleagentonestepimprovement*
\begin{proof}
We use a step size $\eta_{t,h} = \eta_t$, that verifies $\eta_t a \leq 1/2$, and remains constant during $\nlupdates$ steps and only depends on $t$.
Expanding the expected square of \Cref{claim:decomposition-error}'s error decomposition, we have
\begin{align*}
& \Ee \left[ \norm{ \locparam{1}{t,\nlupdates} - \paramlim }^2 \right]
\\
& =
\norm{ \loccontract{c}{t,1:\nlupdates} ( \globparam{t} - \paramlim ) }^2
+ 
2\Ee \Big[ 
\bpscal
{\loccontract{c}{t,1:\nlupdates} ( \globparam{t} - \paramlim )}
{\sum_{h=1}^{H} \step_{t} \loccontract{c}{t,h+1:H}  \left( \locnoisetheta{c}{t,h-1}{\locRandState{c}{t,h}}
 +  \locerrordet{c}{t,h-1} \right)
 } \Big]
 \\ 
& \quad + 
\Ee \left[ \bnorm{ \sum_{h=1}^{H} \step_{t} \loccontract{c}{t,h+1:H}  \left( \locnoisetheta{c}{t,h-1}{\locRandState{c}{t,h}}
 +  \locerrordet{c}{t,h-1} \right) }^2 \right]
\\
& \le
\underbrace{\left[ \norm{ \loccontract{c}{t,1:\nlupdates}  \left( \locparam{1}{t} - \paramlim \right) }^2 \right]}_{\dterm{1}}
 + 
\underbrace{2 \step_t 
\bpscal
{\loccontract{c}{t,1:\nlupdates} ( \globparam{t} - \paramlim )}
{\Ee \Big[  \sum_{h=1}^{H} \loccontract{c}{t,h+1:H}   \locnoisetheta{c}{t,h-1}{\locRandState{c}{t,h}}
  \Big] } }_{\dterm{2}}
 \\ 
& \quad + 
\underbrace{2 \step_t 
\bpscal
{\loccontract{c}{t,1:\nlupdates} ( \globparam{t} - \paramlim )}
{\Ee \Big[ \sum_{h=1}^{H} \loccontract{c}{t,h+1:H} \locerrordet{c}{t,h-1}   \Big] } }_{\dterm{3}}
 \\ 
& \quad + 
\underbrace{2 \step_t^2 \Ee \left[ \bnorm{ \sum_{h=1}^{H}  \loccontract{c}{t,h+1:H}\locnoisetheta{c}{t,h-1}{\locRandState{c}{t,h}}}^2 \right]}_{\dterm{4}}
+ \underbrace{2 \step_t^2 \Ee \left[ \bnorm{ \sum_{h=1}^{H} \step_{t} \loccontract{c}{t,h+1:H} \locerrordet{c}{t,h-1} }^2 \right]}_{\dterm{5}}
\eqsp.
\end{align*}
In this decomposition, $\dterm{1}$ is an optimization error, representing progress towards the solution $\paramlim$; $\dterm{2}$ is a sampling bias error due to the Markovian noise; $\dterm{3}$ and $\dterm{5}$ are error terms due to sub-optimality of the current policy; and $\dterm{4}$ is a variance term.
Next, we bound each term of this decomposition.

\textbf{Bound on $\boldsymbol{\dterm{1}}$.}
The first term is the contraction term, and can be bounded using \Cref{prop:contract-multiple-step} %
\begin{align}
\label{eq:singleagent-onestep-bound-T1}
\dterm{1} 
& \le
(1 - \step_t \lminAbar)^{\nlupdates} \norm{ \globparam{t} - \paramlim }^2
\eqsp.
\end{align}

\textbf{Bound on $\boldsymbol{\dterm{2}}$.}
We bound $\dterm{2}$ using Cauchy-Schwarz inequality and Young's inequality
\begin{align*}
\nonumber
\dterm{2}
 & \le
 2 \step_{t} \norm{\loccontract{1}{1:H}} \norm{ \globparam{t} - \paramlim }
 \bnorm{\Ee\Big[  \sum_{h=1}^\nlupdates \loccontract{1}{h+1:H}   \locnoise{1}{\locRandStateOpt{1}{t,h}}
 \Big] } 
 \\
 & \le
 \frac{\step_{t} \nlupdates \lminAbar}{8} \norm{ \globparam{t} - \paramlim }^2 
 + \frac{8 \step_t}{\nlupdates \lminAbar}
 \bnorm{\Ee\Big[  \sum_{h=1}^\nlupdates \loccontract{1}{h+1:H}   \locnoise{1}{\locRandStateOpt{1}{t,h}}
 \Big] }^2 
 \eqsp,
\end{align*}
where we also used the fact that $ \norm{\loccontract{1}{1:H}} \le 1$.
Then, by \Cref{lem:bound-norm-sum-expect-epsilon}, we have  
\begin{align}
\nonumber
\dterm{2}
 & \le
 \frac{\step_{t} \nlupdates \lminAbar}{8} \norm{ \globparam{t} - \paramlim }^2 
 + \frac{8 \step_t}{\nlupdates \lminAbar}
\Big( 
\frac{8 \boundGrad  }{3}
\big(\step_t \boundA \nlupdates \taumix + \startedfromstationary \taumix \big) 
  \Big)^2
  \\
 & \le
 \frac{\step_{t} \nlupdates \lminAbar}{8} \norm{ \globparam{t} - \paramlim }^2 
 +
\frac{58 \step_t^3 \nlupdates \taumix^2 \boundA^2 \boundGrad^2}{ \lminAbar}
+
\frac{58 \startedfromstationary \step_t \boundGrad^2 
 \taumix^2 }{ \nlupdates \lminAbar}
  \label{eq:singleagent-onestep-bound-T2}
 \eqsp.
\end{align}

\textbf{Bound on $\boldsymbol{\dterm{3}}$.}
The term $\dterm{3}$ can be bounded using Cauchy-Schwarz inequality and \Cref{coro:lip-A-b}, 
\begin{align}
\dterm{3}
 \le
2 \step_t 
\norm{ \loccontract{c}{t,1:\nlupdates} } \norm{ \globparam{t} - \paramlim }
\bnorm{\Ee \Big[ \sum_{h=1}^{H} \loccontract{c}{t,h+1:H} \locerrordet{c}{t,h-1}   \Big] }
\label{eq:singleagent-onestep-bound-T3}
& \le
2 \step_t \nlupdates  
\invdistlip \boundGrad 
\norm{ \globparam{t} \!-\! \paramlim }^2
 \eqsp ,
\end{align}
where we also used $\norm{ \PE \locerrordet{1}{t,h} } \le
\norm{ ( \nbarA{1}(\globparam{t}) - \nbarA{1}(\paramlim) ) } \locprojradius
+ \norm{ \nbarb{1}(\globparam{t}) - \nbarb{1}(\paramlim)  } \le \invdistlip \boundGrad \norm{ \globparam{t} \!-\! \paramlim }$.

\textbf{Bound on $\boldsymbol{\dterm{4}}$.}
The term $\dterm{4}$ can be bounded using \Cref{lem:bound-expect-norm-sq-sum-epsilon-N} with $\nagent=1$, which gives
\begin{align}
\dterm{4} 
\label{eq:singleagent-onestep-bound-T4}
 & \le
 136 \step_t^2 \nlupdates \taumix \boundGrad^2
+
5504 \step_t^4 \boundA^2 \boundGrad^2 
  \nlupdates^2 \taumix^2
\eqsp.
\end{align}

\textbf{Bound on $\boldsymbol{\dterm{5}}$.}
Finally, the last bound is also a consequence of \Cref{coro:lip-A-b}, which gives
\begin{align}
    \dterm{5}  \leq 2 \step_t^2 \nlupdates
    \sum_{h=1}^H \Ee \Big[ \bnorm{ \loccontract{1}{h+1:H} \locerrordet{1}{t,h-1} }^2 \Big] 
    \label{eq:singleagent-onestep-bound-T5}
    & \le 
     8 \step_t^2 \nlupdates^2 \boundGrad^2 \invdistlip^2 \norm{ \globparam{t} \!-\! \paramlim }^2
     \eqsp.
\end{align}

\textbf{Full error bound.}
Plugging \eqref{eq:singleagent-onestep-bound-T1}, \eqref{eq:singleagent-onestep-bound-T2}, \eqref{eq:singleagent-onestep-bound-T3}, \eqref{eq:singleagent-onestep-bound-T4} and \eqref{eq:singleagent-onestep-bound-T5} in the above decomposition, we have
\begin{align*}
\Ee[ \norm{ \locparam{1}{t,\nlupdates} - \paramlim }^2 ]
& \le
(1 - \step \lminAbar)^{\nlupdates} \norm{ \globparam{t} - \paramlim }^2
 + \frac{\step_{t} \nlupdates \lminAbar}{8} \norm{ \globparam{t} - \paramlim }^2 
 +
\frac{58 \step_t^3 \nlupdates \taumix^2 \boundA^2 \boundGrad^2}{ \lminAbar}
+
\frac{58 \startedfromstationary \step_t  \boundGrad^2 
 \taumix^2 }{ \nlupdates \lminAbar}
\\
& 
+ 2 \step_t \nlupdates \invdistlip \boundGrad
\norm{ \globparam{t} \!-\! \paramlim }^2
+  136 \step_t^2 \nlupdates \taumix \boundGrad^2
+
5504 \step_t^4 \boundA^2 \boundGrad^2 
  \nlupdates^2 \taumix^2
+ 8 \step_t^2 \nlupdates^2 \boundGrad^2 \invdistlip^2 \norm{ \globparam{t} \!-\! \paramlim }^2
\eqsp.
\end{align*}
After reorganizing the terms, we obtain
\begin{align*}
\Ee[ \norm{ \locparam{1}{t,\nlupdates} - \paramlim }^2 ]
& \le
\Big(1 - \step \lminAbar\nlupdates/2 + {\step_{t} \nlupdates \lminAbar}/{8} + 2 \step_t \nlupdates 
 \invdistlip \boundGrad
+8 \step_t^2 \nlupdates^2 \boundGrad^2 \invdistlip^2
\Big) \norm{ \globparam{t} - \paramlim }^2
\\
& \quad 
+ 136 \step_t^2 \nlupdates \taumix \boundGrad^2
+
\frac{58 \startedfromstationary \step_t \boundGrad^2 
 \taumix^2 }{ \nlupdates \lminAbar}
 +
\frac{58 \step_t^3 \nlupdates \taumix^2 \boundA^2 \boundGrad^2}{ \lminAbar}
 +
 5504 \step_t^4 \boundA^2 \boundGrad^2 
  \nlupdates^2 \taumix^2
\eqsp.
\end{align*}
Finally, remark that $16 \boundGrad \invdistlip \le \lminAbar$, which implies that 
\begin{align*}
1 - \step \lminAbar\nlupdates/2 + \step_t \lminAbar \nlupdates/8 
+ 2 \step_t \nlupdates \invdistlip \boundGrad
+ 8 \step_t^2 \nlupdates^2 \boundGrad^2 \invdistlip^2
& \le
1
- \step \lminAbar \nlupdates/4
\eqsp,
\end{align*}
and the result of the lemma follows by bounding $\tfrac{58 \step_t^3 \nlupdates \taumix^2 \boundGrad^2 \boundA^2 }{\lminAbar} 
+ 5504 \step_t^4 \boundA^2 \boundGrad^2 \nlupdates^2 \taumix^2 \le \tfrac{976 \step_t^3 \boundGrad^2 \boundA^2 \nlupdates \taumix^2}{\lminAbar} $. %
\end{proof}

\singleagentconvergencerate*
\begin{proof}
Since projections on convex sets are contractions, and $\paramlim$ is within the set on which we project, we have
\begin{align*}
\norm{ \globparam{t+1} - \paramlim }^2
\le
\norm{ \avgparam{t+1} - \paramlim }^2
\eqsp.
\end{align*}
Applying \Cref{lem:one-step-improvement} and unrolling the recursion gives the result.
\end{proof}
We now prove the corollary for \SARSA's sample complexity.
\singleagentcomplexity*
\begin{proof}
Let $\epsilon > 0$.
From \Cref{thm:convergence-rate-single-agent}, we have
\begin{align*}
& \Ee[ \norm{ \globparam{\nepisode} - \paramlim }^2 ]
\le
(1 - \tfrac{\step \lminAbar\nlupdates}{4} )^\nepisode \norm{ \globparam{0} - \paramlim }^2
 + \frac{\cstthm[c_1]{544} \step \taumix \boundGrad^2}{\lminAbar} %
+ \startedfromstationary \tfrac{ \cstthm[c_1]{232}  \taumix^2 \boundGrad^2 }{\nlupdates^2 \lminAbar^2} 
+ \tfrac{\cstthm[c_1]{3904} \step^2 \taumix^2 \boundGrad^2 \boundA^2 }{\lminAbar^2} 
\eqsp,
\end{align*}
To obtain an overall mean squared error smaller than $\epsilon^2$, each term has to be smaller than $\epsilon^2$, which gives the conditions of $\step$ and $\nlupdates$. The bounds on $\nepisode$ and $\nepisode \nlupdates$ follow from bounding the exponentially decreasing term.

\end{proof}

%% file: src/sup-proof-fed-sarsa.tex
\subsection{Convergence of \FedSARSA --- Proofs of Proposition~\ref{prop:existence-theta-star-fed} and Proposition~\ref{prop:propdistanceloctoglobal}}
\label{sec:app-convergence-fedsarsa}

\propconvergencefedsarsatothetastar*
\begin{proof}
    The proof follows ideas similar to \citet{de2000existence}'s Theorem 5.1. 
    For $c \in \intlist{1}{\nagent}$, we define the function $s^{(c)}$ as
    \begin{align*}
    s^{(c)}(\theta) 
    =  
    \nbarA{c}(\theta) \theta + \nbarb{c}(\theta)
    \eqsp,
    \end{align*}
    as well as the map $F^{(c)}_\eta : \theta \mapsto \theta + \eta s^{(c)}(\theta)$, with $\eta>0$.
    We show that $\frac{1}{N} \sum_{c=1}^N  F_\eta^{(c)}$ has a fixed point.

    Remark that $\frac{1}{N} \sum_{c=1}^N  F_\eta^{(c)}$ is a continuous function, and that the set $\cC :=\{ \theta \mid \|\theta\| \leq \frac{(1+\xi) \bar R}{1-\xi}\}$ is closed under $\frac{1}{N} \sum_{c=1}^N  F_\eta^{(c)}$ for a well-chosen $\xi \in (0, 1)$ (the same as for the individual $F_\eta^{(c)}$, whose explicit expression is provided in~\citet{de2000existence}). Indeed, we have that
    \begin{align*}
        \bnorm{ \frac{1}{N} \sum_{c=1}^N F_\eta^{(c)}(\theta) }
        & \leq \frac{1}{N} \sum_{c=1}^N \| F_\eta^{(c)}(\theta)\| 
        \leq
        \frac{1}{N} \sum_{c=1}^N (\xi \|\theta\| + (1+\xi)\bar R)
        =
        \xi \|\theta\| + (1+\xi)\bar R
        \eqsp .
    \end{align*}
    By Brouwer's fixed-point theorem, this gives the existence of some $\paramlim \in \rset^d$ such that:
    \begin{align*}
        \frac{1}{N} \sum_{c=1}^N F_\eta^{(c)} \paramlim = \paramlim  
    \Longleftrightarrow~~ & 
    \frac{1}{N} \sum_{c=1}^N s^{(c)}(\paramlim) = 0 
    \\
      \Longleftrightarrow~~ &  
      \frac{1}{N} \sum_{c=1}^N \nbarA{c}(\paramlim) \paramlim + \frac{1}{N} \sum_{c=1}^N \nbarb{c}(\paramlim) = 0 
      \eqsp ,
    \end{align*}
    which proves the existence of a fixed point.
    To prove unicity, let $\theta_*^1, \theta_*^2 \in \mathcal{W}$ be two fixed points of \FedSARSA.
    Then 
    $$
    \theta_*^1 - \theta_*^2 
    =
    \theta_*^1 - \theta_*^2 - \frac{\step}{N} \sum_{c=1}^N \nbarA{c}(\theta_*^1) \theta_*^1 - \nbarA{c}(\theta_*^2) \theta_*^2 + \nbarb{c}(\theta_*^1) - \nbarb{c}(\theta_*^2)
    \eqsp.
    $$
    This yields, denoting $\barA = 1/N \sum_c \nbarA{c}$,
    $$
    \theta_*^1 - \theta_*^2 
    =
    (I - \step \barA(\theta_*^1)) (\theta_*^1 - \theta_*^2) - \frac{\step}{N} \sum_{c=1}^N (\nbarA{c}(\theta_*^1) - \nbarA{c}(\theta_*^2)) \theta_*^2 + \nbarb{c}(\theta_*^1) - \nbarb{c}(\theta_*^2)
    \eqsp.
    $$
    Taking the norm, using triangle inequality, using \Cref{coro:lip-A-b} to bound $\|\nbarA{c}(\theta_*^1)-\nbarA{c}(\theta_*^2)\| \le \invdistlip\boundA $, $\|\nbarb{c}(\theta_*^1)-\nbarb{c}(\theta_*^2)\| \le \invdistlip \boundb $, and using the bound $\|\theta_*^2\|\le  \projradius$, we obtain
    $$
    \|\theta_*^1 - \theta_*^2\| 
    \le
    (1 - \step a + 2\step \boundGrad \invdistlip) \|\theta_*^1 - \theta_*^2\| \le (1 - \step a /2) \|\theta_*^1 - \theta_*^2\|$$ where the first inequality comes from Lipschitzness of the policy improvement and the second one comes from \Cref{assum:lipschitz-improvement}. This proves that $\|\theta_*^1 - \theta_*^2\| = 0$, guaranteeing the uniqueness of \FedSARSA's limit point.

  \end{proof}

\propdistanceloctoglobal*

\begin{proof} 

From \Cref{prop:heterogeneity-constants}, we obtain
\begin{align*}
    \norm{ \nbarA{c}(\paramlim) (\locparamlim{c} - \paramlim) } \le \hgtyAtheta \eqsp.
\end{align*}
Using \Cref{assum:contraction}, we then obtain
\begin{align*}
\norm{ \locparamlim{c} - \paramlim }
\le \frac{\hgtyAtheta}{a} \eqsp . 
\end{align*}%
Now we need to bound $\| \locparamlim{c} - \locparamstar{c} \|$. They are respectively defined by:
 \begin{gather*}
      \nbarA{c}(\paramlim) \locparamlim{c} + \nbarb{c}(\paramlim) = 0 ,
      \\
      \nbarA{c}(\locparamstar{c}) \locparamstar{c}  + \nbarb{c}(\locparamstar{c} ) = 0 \eqsp .
 \end{gather*}
Subtracting the equalities we obtain
\begin{align*}
    \left( \nbarA{c}(\locparamstar{c}) - \nbarA{c}(\paramlim)  \right) \locparamstar{c}+ \nbarA{c}(\paramlim)  \left(\locparamstar{c} - \locparamlim{c} \right) =  \nbarb{c}(\paramlim) - \nbarb{c}(\locparamstar{c} ) \eqsp .
\end{align*}
Using \Cref{assum:contraction} a second time yields
\begin{align*}
    \norm{\locparamstar{c} - \locparamlim{c} } & \leq \frac{1}{a} \left(  \left\| \left( \nbarA{c}(\locparamstar{c}) - \nbarA{c}(\paramlim)  \right) \locparamstar{c} \right\| + \left\| \nbarb{c}(\paramlim) - \nbarb{c}(\locparamstar{c} ) \right\| \right) 
    \eqsp.
\end{align*}
Then, \Cref{coro:lip-A-b}, gives
\begin{align*}
    \norm{\locparamstar{c} - \locparamlim{c} } & \leq \frac{1}{a} \left( \invdistlip \boundA \| \locparamstar{c}\|  + \invdistlip \boundb  \right) \| \locparamstar{c} - \paramlim \| \, .
\end{align*}
Using the triangle inequality, we obtain:
\begin{align*}
    \|  \locparamstar{c} - \theta_\star \| & \leq \frac{\hgtyAtheta}{a} +\frac{1}{a} \left( \invdistlip \boundA \| \locparamstar{c}\|  + \invdistlip \boundb  \right) \| \locparamstar{c} - \paramlim \| \, .
\end{align*}
Since $\| \locparamstar{c} \| \leq \projradius$ and using \Cref{assum:lipschitz-improvement}, we have
\begin{align*}
    \projradius \invdistlip \boundA + \invdistlip \boundb
    \leq a/80,
\end{align*}
then the result follows from the definition of $\hgtyAtheta$.
\end{proof}

\subsection{Bound on heterogeneity drift --- Proofs of Proposition~\ref{prop:heterogeneity-constants}}
\label{sec:app-heterogeneity-drift-bound}

\boundonheterogeneityconstants*
\begin{proof}
This proofs is inspired by the proof of \citet{zhang2024finite}'s Theorem 1.

\textit{(Heterogeneity of the $\nbarA{c}$'s.)}
By \citet{mitrophanov2005sensitivity}'s Theorem 3.1, we have
\begin{align}
\label{eq:mitrophanov-inequality}
\norm{ \locstatestatdist{c}{\paramlim} - \locstatestatdist{c}{\paramlim} }[\TV]
\le 
4 (1 + \taumix) \sup_{\varrho \sim \mathcal{P}(\S)} \norm{ \varrho \nkerMDP{c}_{\paramlim} - \varrho \nkerMDP{c'}_{\paramlim} }[TV]
\eqsp,
\end{align}
Thus, we have, for any $c \in \intlist{1}{\nagent}$
\begin{align*}
\frac{1}{\nagent}
\sum_{c'=1}^\nagent 
\norm{ \nbarA{c}(\paramlim) - \nbarA{c'}(\paramlim) }^2
\le
\frac{\boundA^2}{\nagent}
\sum_{c'=1}^\nagent  \norm{ \locstatestatdist{c}{\paramlim} - \locstatestatdist{c'}{\paramlim} }[\TV]^2
\le
16 \boundA^2 (1 + \taumix)^2 \kerhgty^2
\eqsp.
\end{align*}

\textit{(Heterogeneity of the $\locparamlim{c}$'s.)}
By definition of $\paramlim$ and $\locparamlim{c}$, we have $\barA(\paramlim) \paramlim = \barb(\paramlim)$ and $\nbarA{c}(\paramlim) \locparamlim{c} = \nbarb{c}(\paramlim)$.
This gives the identity
\begin{align*}
\barA(\paramlim) \paramlim
- \nbarA{c}(\paramlim) \locparamlim{c} 
= \barb(\paramlim) - \nbarb{c}(\paramlim)
\eqsp.
\end{align*}
Adding and subtracting $\nbarA{c}(\paramlim) \paramlim$ on the left side, we obtain
\begin{align*}
(\barA(\paramlim) - \nbarA{c}(\paramlim)) \paramlim
+ \nbarA{c}(\paramlim) (\paramlim - \locparamlim{c})
= \barb(\paramlim) - \nbarb{c}(\paramlim)
\eqsp.
\end{align*}
Reorganizing the terms, we obtain
\begin{align*}
\nbarA{c}(\paramlim) (\paramlim - \locparamlim{c})
=  
(\nbarA{c}(\paramlim) - \barA(\paramlim)) \paramlim
+ \barb(\paramlim) - \nbarb{c}(\paramlim)
\eqsp,
\end{align*}
which gives, by taking the norm, using the triangle inequality and \Cref{assum:contraction}, that
\begin{align*}
\norm{ \nbarA{c}(\paramlim) (\paramlim - \locparamlim{c}) }
\le
\norm{ \nbarA{c}(\paramlim) - \barA(\paramlim) } \norm{  \paramlim }
+ \norm{ \barb(\paramlim) - \nbarb{c}(\paramlim) }
\eqsp.
\end{align*}
Averaging over all $c \in \intlist{1}{\nagent}$ and using \eqref{eq:mitrophanov-inequality}, we obtain
\begin{align*}
\frac{1}{\nagent} \sum_{c=1}^\nagent 
\norm{ \paramlim - \locparamlim{c} }^2
\le
32 \boundA^2 (1 + \taumix)^2 \kerhgty^2 \norm{  \paramlim }^2
+ 32 \boundb^2 (1 + \taumix)^2 \rewardhgty^2  
\eqsp,
\end{align*}
and the result follows.
\end{proof}

\begin{lemma}
\label{lem:bound-drift-crude}
Let $\step_t$ such that $\step_t \nlupdates \boundA \le 1$. 
Then, it holds that
\begin{align*}
\norm{ \biasterm }
& \le
2 \step_t \boundA \locprojradius
\eqsp.
\end{align*}
\end{lemma}
\begin{proof}
We have
\begin{align*}
\norm{ \biasterm }
& =
\bnorm{ 
\frac{1}{\nagent} \sum_{c=1}^\nagent \Big( \Id - \loccontract{c}{t,1:\nlupdates}  \Big) \left( \locparamlim{c} - \paramlim \right) }
\\
& \le
2 \norm{ \Id - \loccontract{c}{t,1:\nlupdates} } \locprojradius
\eqsp,
\end{align*}
where we used $\norm{ \paramlim } \le \locprojradius$ and $\norm{ \locparamlim{c} } \le \locprojradius$. 
The result comes from $\norm{ \Id - \loccontract{c}{t,1:\nlupdates} } \le \step_t \nlupdates \boundA$.
\end{proof}

\begin{lemma}
\label{lem:bound-drift-one-round}
Let $\step_t$ such that $\step_t \nlupdates \boundA \le 1$. 
Then, it holds that
\begin{align*}
\norm{ \biasterm }^2
& \le
\frac{\step_t^4 \nlupdates^2 (\nlupdates-1)^2}{4} \hgtyA^2 \hgtyAtheta^2
\eqsp.
\end{align*}
\end{lemma}
\begin{proof}
Using \Cref{lem:product-diff}, we have
\begin{align*}
\nonumber
& \frac{1}{\nagent} \sum_{c=1}^\nagent \left( \Id - \loccontract{c}{t,1:H} \right)  \left( \locparamlim{c} - \paramlim \right)
=
- \frac{\step_t}{\nagent} \sum_{c=1}^\nagent \sum_{h=1}^H\loccontract{c}{t,h+1:\nlupdates}  \nbarA{c} \left( \locparamlim{c} - \paramlim \right)
\eqsp.
\end{align*}
Then, we remark that
\begin{align*}
\sum_{c=1}^\nagent \nbarA{c}(\paramlim) \locparamlim{c}
- \sum_{c=1}^\nagent \nbarA{c}(\paramlim) \paramlim
=
\sum_{c=1}^\nagent \nbarb{c}(\paramlim)
- \sum_{c=1}^\nagent \nbarb{c}(\paramlim)
= 0
\eqsp.
\end{align*}
Thus, using the notation $\loccontract{\text{avg}}{h+1:\nlupdates} = \frac{1}{\nagent} \sum_{c=1}^\nagent \loccontract{c}{t,h+1:\nlupdates}$, we have
\begin{align*}
\biasterm
& = \frac{\step_t}{\nagent} \sum_{c=1}^\nagent \sum_{h=1}^H \left(\loccontractbar{c}{t,h+1:\nlupdates} - \loccontractbar{\text{avg}}{t,h+1:\nlupdates} \right) \nbarA{c}(\paramlim) \left( \locparamlim{c} - \paramlim \right)
\\
& = \frac{\step_t^2}{\nagent} \sum_{c=1}^\nagent \sum_{h=1}^H \sum_{\ell=1}^{h-1} \loccontractbar{c}{t,1:\ell-1}(\nbarA{c}(\paramlim) - \barA(\paramlim)) \loccontractbar{\text{avg}}{t,\ell+h+1:\nlupdates} \nbarA{c}(\paramlim) \left( \locparamlim{c} - \paramlim \right)
\eqsp.
\end{align*}
Consequently, we have, by the triangle inequality,
\begin{align*}
\norm{ \biasterm }
& \le
\frac{\step_t^2 }{\nagent} \sum_{c=1}^\nagent \sum_{h=1}^H \sum_{\ell=1}^{h-1} \norm{ \loccontractbar{c}{t,1:\ell-1}(\nbarA{c}(\paramlim) - \barA(\paramlim)) \loccontractbar{\text{avg}}{t,\ell+h+1:\nlupdates} \nbarA{c}(\paramlim) \left( \locparamlim{c} - \paramlim \right) }
\\
& \le
\frac{\step_t^2}{\nagent} \sum_{c=1}^\nagent \sum_{h=1}^H \sum_{\ell=1}^{h-1} \norm{ \nbarA{c}(\paramlim) - \barA(\paramlim) }^2 \norm{ \nbarA{c}(\paramlim) \left( \locparamlim{c} - \paramlim \right) }
\\
& =
\frac{\step_t^2 \nlupdates(\nlupdates-1)}{2\nagent} \sum_{c=1}^\nagent \norm{ \nbarA{c}(\paramlim) - \barA(\paramlim) }^2 \norm{ \nbarA{c}(\paramlim) \left( \locparamlim{c} - \paramlim \right) }
\eqsp.
\end{align*}
Using Cauchy-Schwarz inequality, we obtain
\begin{align*}
\norm{ \biasterm }^2
& \le
\frac{\step_t^4 \nlupdates^2 (\nlupdates-1)^2 }{4} \Big( \frac{1}{\nagent} \sum_{c=1}^\nagent  \norm{ \nbarA{c}(\paramlim) - \barA(\paramlim) }^2 \Big) \Big( \frac{1}{\nagent}  \sum_{c=1}^\nagent \norm{ \nbarA{c}(\paramlim) \left( \locparamlim{c} - \paramlim \right) }^2 \Big)
\eqsp,
\end{align*}
which is the result.
\end{proof}

\subsection{Convergence rate --- Proofs of Lemma~\ref{lem:federated-one-step-improvement} and Theorem~\ref{thm:convergence-rate-federated}}
\label{sec:app-convergence-rate-federated}

\federatedonestepimprovement*
\begin{proof}
Starting from \Cref{claim:decomposition-error-federated}, and expanding the square, we have
\begin{align*}
& \PE\Big[ \norm{ \avgparam{t+1} - \paramlim }^2 \Big]
=
\bnorm{ \frac{1}{\nagent} \sum_{c=1}^\nagent \loccontract{c}{t,1:\nlupdates}  \left( \globparam{t} - \paramlim \right) }^2
\\
& \qquad  
+
2 \bpscal{ 
\frac{1}{\nagent} \sum_{c=1}^\nagent \loccontract{c}{t,1:\nlupdates}  \left( \globparam{t} - \paramlim \right)  
}{
\biasterm
+ \frac{\step_t}{\nagent} \sum_{c=1}^\nagent \sum_{h=1}^{H}\loccontract{c}{t,h+1:H} \bCPE{ \locnoisetheta{c}{t,h-1}{\locRandState{c}{t,h}} + \locerrordet{c}{t,h-1} }{\globfiltr{t}}
}
\\
& \qquad +
\bCPE{ 
\bnorm{
\biasterm
+ \frac{\step_t}{\nagent} \sum_{c=1}^\nagent \sum_{h=1}^{H}\loccontract{c}{t,h+1:H} \Big( \locnoisetheta{c}{t,h-1}{\locRandState{c}{t,h}} + \locerrordet{c}{t,h-1} \Big)
}^2
}{\globfiltr{t}}
\eqsp.
\end{align*}
This gives the decomposition
\begin{align*}
& \PE\Big[ \norm{ \avgparam{t+1} - \paramlim }^2 \Big]
\\
& =\!
\underbrace{
\bnorm{ \frac{1}{\nagent} \sum_{c=1}^\nagent \loccontract{c}{t,1:\nlupdates} \! \left( \globparam{t} \!-\! \paramlim \right) \!}^2
}_{\dfterm{1}}
\!+
\underbrace{2 \bpscal{ 
\frac{1}{\nagent}\! \sum_{c=1}^\nagent \loccontract{c}{t,1:\nlupdates} \! \left( \globparam{t} \!-\! \paramlim \right) \!\! 
}{
\frac{\step_t}{\nagent}\! \sum_{c=1}^\nagent \sum_{h=1}^{H}\loccontract{c}{t,h+1:H} \bCPE{ \locnoisetheta{c}{t,h-1}{\locRandState{c}{t,h}} \!}{\globfiltr{t}}
}}_{\dfterm{2}}
\\
& \quad 
+
\underbrace{2 \bpscal{ 
\frac{1}{\nagent} \sum_{c=1}^\nagent \loccontract{c}{t,1:\nlupdates}  \left( \globparam{t} - \paramlim \right)  
}{
\frac{\step_t}{\nagent} \sum_{c=1}^\nagent \sum_{h=1}^{H}\loccontract{c}{t,h+1:H} \bCPE{ \locerrordet{c}{t,h-1} }{\globfiltr{t}}
}}_{\dfterm{3}}
\\
& \quad +
\underbrace{3 \bCPE{ 
\bnorm{ \frac{\step_t}{\nagent} \sum_{c=1}^\nagent \sum_{h=1}^{H}\loccontract{c}{t,h+1:H} \locnoisetheta{c}{t,h-1}{\locRandState{c}{t,h}} }^2
}{\globfiltr{t}}}_{\dfterm{4}}
+
\underbrace{3 \bCPE{ 
\bnorm{ \frac{\step_t}{\nagent} \sum_{c=1}^\nagent \sum_{h=1}^{H}\loccontract{c}{t,h+1:H} \locerrordet{c}{t,h-1} }^2
}{\globfiltr{t}}}_{\dfterm{5}}
\\
& \quad
+ 
\underbrace{3 \norm{ \biasterm }^2
+
2 \bpscal{ 
\frac{1}{\nagent} \sum_{c=1}^\nagent \loccontract{c}{t,1:\nlupdates}  \left( \globparam{t} - \paramlim \right)  
}{
\biasterm
}}_{\dfterm{6}}
\eqsp.
\end{align*}
The terms $\dfterm{1}$ to $\dfterm{5}$ are analogous to the terms $\dterm{1}$ to $\dterm{5}$ in the single-agent setting, although with a variance term $\dfterm{4}$ whose leading term will scale in $1/\nagent$.
The term $\dfterm{6}$ is due to heterogeneity, and accounts for the differences in local updates from one agent to another.

\textbf{Bound on $\dfterm{1}$.}
We have, using Jensen's inequality and \Cref{prop:contract-multiple-step},
\begin{align}
\label{eq:bound-fed-decomp:U1}
\dfterm{1} 
& \le
\frac{1}{\nagent} \sum_{c=1}^\nagent \norm{ \loccontract{c}{t,1:\nlupdates}  \left( \globparam{t} - \paramlim \right) }^2 
\le 
(1 - \step \lminAbar)^\nlupdates \norm{ \globparam{t} - \paramlim }^2
\eqsp.
\end{align}
\textbf{Bound on $\dfterm{2}$.}
First, we use \Cref{lem:relation-markov-stationary} to construct a Markov chain $\{\locRandStateOpt{c}{t,h} \}_{h \ge 0}$ initialized in the stationary distribution alike in the single-agent case.
We then have
\begin{align*}
\dfterm{2}
& \le
\frac{\step_t \lminAbar \nlupdates}{8} \bnorm{ \frac{1}{\nagent} \sum_{c=1}^\nagent \loccontract{c}{t,1:\nlupdates}  \left( \globparam{t} - \paramlim \right) }^2
+ \frac{8}{\step_t \lminAbar \nlupdates}
\bnorm{ \frac{\step_t}{\nagent} \sum_{c=1}^\nagent \sum_{h=1}^{H}\loccontract{c}{t,h+1:H} \PE\Big[ \locnoisetheta{c}{t,h-1}{\locRandState{c}{t,h}} \Big] }^2 
\\
& \le
\frac{\step_t \lminAbar \nlupdates}{8} \norm{ \globparam{t} - \paramlim }^2
+ \frac{8}{\step_t \lminAbar \nlupdates}
\bnorm{ \frac{\step_t}{\nagent} \sum_{c=1}^\nagent \sum_{h=1}^{H}\loccontract{c}{t,h+1:H} \PE\Big[ \locnoisetheta{c}{t,h-1}{\locRandState{c}{t,h}} \Big] }^2 
\eqsp.
\end{align*}
Using \Cref{lem:bound-norm-sum-expect-epsilon}, we obtain
\begin{align}
\dfterm{2}
\label{eq:bound-fed-decomp:U2}
& \le
 \frac{\step_{t} \nlupdates \lminAbar}{8} \norm{ \globparam{t} - \paramlim }^2 
 +
\frac{58 \step_t^3 \nlupdates \taumix^2 \boundA^2 \boundGrad^2}{ \lminAbar}
+
\frac{58 \startedfromstationary \step_t \boundGrad^2 
 \taumix^2 }{ \nlupdates \lminAbar}
\eqsp.
\end{align}

\textbf{Bound on $\dfterm{3}$.}
First, we use the triangle inequality to split the mean. Then, similarly to the bound on $\dterm{3}$, we use Cauchy-Schwarz inequality and \Cref{coro:lip-A-b} to obtain
\begin{align}
\dfterm{3}
\label{eq:bound-fed-decomp:U3}
& \le
2 \step_t \nlupdates \invdistlip \boundGrad
\norm{ \globparam{t} \!-\! \paramlim }^2
\eqsp.
\end{align}

\textbf{Bound on $\dfterm{4}$.}
Using \Cref{lem:bound-expect-norm-sq-sum-epsilon-N}, we have the bound
\begin{align}
\label{eq:bound-fed-decomp:U4}
\dfterm{4}
& \le 
\frac{136 \step_t^2 \nlupdates \taumix \boundGrad^2}{\nagent}
+
5504 \step_t^4 \boundA^2 \boundGrad^2 
  \nlupdates^2 \taumix^2
\end{align}

\textbf{Bound on $\dfterm{5}$.}
Alike $\dterm{5}$, the bound on $\dfterm{5}$ is a consequence of \Cref{coro:lip-A-b},
\begin{align}
\label{eq:bound-fed-decomp:U5}
\dfterm{5}  
\leq \frac{3 \step_t^2 \nlupdates}{\nagent}
\sum_{c=1}^\nagent \sum_{h=1}^H \Ee \Big[ \bnorm{ \loccontract{1}{h+1:H}\locerrordet{c}{t,h-1} }^2 \Big] 
& \le 
 12 \step_t^2 \nlupdates^2 \boundGrad^2 \invdistlip^2 \norm{ \globparam{t} \!-\! \paramlim }^2
 \eqsp.
\end{align}

\textbf{Bound on $\dfterm{6}$.}
This term is due to heterogeneity.
First, we split the second term using Young's inequality, then use \Cref{lem:bound-drift-one-round}, which gives the bound
\begin{align*}
\dfterm{6}
& \le
\frac{\step_t \lminAbar}{8} \bnorm{ \frac{1}{\nagent} \sum_{c=1}^\nagent \loccontract{c}{t,1:\nlupdates}  \left( \globparam{t} - \paramlim \right) }^2
+ \Big(3 + \frac{8}{\step_t \lminAbar}\Big) \norm{ \biasterm }^2
\\
& \le
\frac{\step_t \lminAbar \nlupdates}{8} \bnorm{ \frac{1}{\nagent} \sum_{c=1}^\nagent \loccontract{c}{t,1:\nlupdates}  \left( \globparam{t} - \paramlim \right) }^2
+ \Big(3 + \frac{8}{\step_t \lminAbar \nlupdates}\Big) 
\step_t^4 \nlupdates^2 (\nlupdates-1)^2 \hgtyA^2 \hgtyAtheta^2
\eqsp,
\end{align*}
which gives, using $3 \le 1/(\step_t \lminAbar \nlupdates)$, 
\begin{align}
\label{eq:bound-fed-decomp:U6}
\dfterm{6}
& \le
\frac{\step_t \lminAbar \nlupdates}{8} \norm{ \globparam{t} - \paramlim }^2
+ \frac{9 \step_t^3 \nlupdates^3 (\nlupdates-1)^2}{\lminAbar} \hgtyA^2 \hgtyAtheta^2
\eqsp.
\end{align}

\textbf{Final bound.}
Plugging \eqref{eq:bound-fed-decomp:U1}, \eqref{eq:bound-fed-decomp:U2}, \eqref{eq:bound-fed-decomp:U3}, \eqref{eq:bound-fed-decomp:U4}, \eqref{eq:bound-fed-decomp:U5}, and \eqref{eq:bound-fed-decomp:U6} the error decomposition above, we obtain
\begin{align*}
\Ee[ \norm{ \avgparam{t+1} - \paramlim }^2 ]
& \le
(1 - \step \lminAbar)^{\nlupdates} \norm{ \globparam{t} - \paramlim }^2
+ \frac{\step_{t} \nlupdates \lminAbar}{8} \norm{ \globparam{t} - \paramlim }^2 
 +
\frac{58 \step_t^3 \nlupdates \taumix^2 \boundA^2 \boundGrad^2}{ \lminAbar}
+
\frac{58 \startedfromstationary \step_t \boundGrad^2 
 \taumix^2 }{ \nlupdates \lminAbar}
\\
& 
+ 2 \step_t \nlupdates \invdistlip \boundGrad
\norm{ \globparam{t} \!-\! \paramlim }^2
+
\frac{136 \step_t^2 \nlupdates \taumix \boundGrad^2}{\nagent}
+
5504 \step_t^4 \boundA^2 \boundGrad^2 
  \nlupdates^2 \taumix^2
\\ & \quad 
+ 12 \step_t^2 \nlupdates^2 \boundGrad^2 \invdistlip^2 \norm{ \globparam{t} \!-\! \paramlim }^2
+ \frac{9 \step_t^3 \nlupdates^3 (\nlupdates-1)^2}{\lminAbar} \hgtyA^2 \hgtyAtheta^2
\eqsp.
\end{align*}
Simplifying, we obtain
\begin{align*}
\PE\Big[ \norm{ \avgparam{t+1} - \paramlim }^2 \Big]
& \le
(1 - \step \lminAbar/8)^\nlupdates \norm{ \globparam{t} - \paramlim }^2
+ \frac{9 \step_t^3 \nlupdates (\nlupdates-1)^2}{\lminAbar} \hgtyA^2 \hgtyAtheta^2
+ \frac{136 \step_t^2 \nlupdates \taumix \boundGrad^2}{\nagent}
\\
& \quad
+ 
\frac{58 \step_t^3 \nlupdates \taumix^2 \boundA^2 \boundGrad^2}{ \lminAbar}
+
\frac{58 \startedfromstationary \step_t \boundGrad^2 
 \taumix^2 }{ \nlupdates \lminAbar}
 + 
5504 \step_t^4 \boundA^2 \boundGrad^2 
  \nlupdates^2 \taumix^2
\eqsp,
\end{align*}
which gives the result of the lemma follows by bounding $\tfrac{58 \step_t^3 \nlupdates \taumix^2 \boundGrad^2 \boundA^2 }{\lminAbar} 
+ 5504 \step_t^4 \boundA^2 \boundGrad^2 \nlupdates^2 \taumix^2 \le \tfrac{976 \step_t^3 \boundGrad^2 \boundA^2 \nlupdates \taumix^2}{\lminAbar} $.
\end{proof}

\federatedconvergencerate*
\begin{proof}
The proof follows the same arguments as the proof of \Cref{thm:convergence-rate-single-agent}.
\end{proof}

\corollaryfedsarsasamplecommcomplexity*
\begin{proof}
Let $\epsilon > 0$.
From \Cref{thm:convergence-rate-federated}, we have
\begin{align*}
\Ee[ \norm{ \globparam{\nepisode} - \paramlim }^2 ]
& \lessthm
(1 - \tfrac{\step \lminAbar\nlupdates}{8} ) \norm{ \globparam{t} - \paramlim }^2+ \tfrac{\cstthm[c_2]{72} \step^2 (\nlupdates-1)^2}{\lminAbar^2} \hgtyA^2 \hgtyAtheta^2
+ \tfrac{\cstthm[c_2]{1088} \step \taumix \boundGrad^2}{\nagent \lminAbar} %
+ \tfrac{\cstthm[c_2]{464}   \startedfromstationary \boundGrad^2 \taumix^2 }{\nlupdates^2 \lminAbar^2} 
+ \tfrac{\cstthm[c_2]{7808} \step^2 \boundGrad^2 \boundA^2 \taumix^2}{\lminAbar^2} 
\eqsp.
\end{align*}
To obtain an overall mean squared error smaller than $\epsilon^2$, each term has to be smaller than $\epsilon^2$, which gives the conditions on $\step$ and $\nlupdates$. 
For the exponential term to be small, we require $\nepisode \approx \tfrac{1}{\step \nlupdates \lminAbar} \log(\tfrac{\norm{ \param_{0} - \paramlim }^2}{\epsilon^2}) $ and $\nepisode \nlupdates \approx \tfrac{1}{\step \lminAbar} \log(\tfrac{\norm{ \param_{0} - \paramlim }^2}{\epsilon^2})$, which give the bounds on $\nepisode$ and $\nepisode \nlupdates$.
\end{proof}